\documentclass{article}

\usepackage[acronym]{glossaries}
\makeglossaries
\setacronymstyle{long-sc-short}
\newacronym{lsd}{lsd}{Lagrangian self-distillation}
\newacronym{esd}{esd}{Eulerian self-distillation}
\newacronym{pde}{pde}{partial differential equation}
\newacronym{ode}{ode}{ordinary differential equation}
\newacronym{fid}{fid}{Fréchet inception distance}
\newacronym{rl}{rl}{reinforcement learning}
\newacronym{imf}{imf}{Improved MeanFlow}

\usepackage{amsmath}
\usepackage{amssymb}
\usepackage{amsthm}
\usepackage{dsfont}
\usepackage{mathtools}
\allowdisplaybreaks

\usepackage{graphicx}
\usepackage{wrapfig}

\usepackage{booktabs}
\usepackage{multirow}
\usepackage{caption}

\usepackage[table]{xcolor}
\usepackage{parskip} %
\usepackage{soul} %

\usepackage{caption}
\usepackage{subcaption}
\usepackage[square,numbers]{natbib}
\usepackage{hyperref}
\hypersetup{
  colorlinks=true,
  linkcolor=teal,
  citecolor=gray,
  urlcolor=blue
}
\usepackage{cleveref}

\definecolor{MidnightBlue}{rgb}{0.1, 0.1, 0.44}
\definecolor{mahogany}{rgb}{0.75, 0.25, 0.0}

\usepackage{listings}
\usepackage{fancyvrb}
\fvset{fontsize=\normalsize}

\lstdefinestyle{mystyle}{
    commentstyle=\color{OliveGreen},
    keywordstyle=\color{BurntOrange},
    numberstyle=\tiny\color{black!60},
    stringstyle=\color{MidnightBlue},
    basicstyle=\ttfamily,
    breakatwhitespace=false,
    breaklines=true,
    captionpos=b,
    keepspaces=true,
    numbers=left,
    numbersep=5pt,
    showspaces=false,
    showstringspaces=false,
    showtabs=false,
    tabsize=2
}
\usepackage{tikz}
\usetikzlibrary{shapes,decorations,arrows,calc,arrows.meta,fit,positioning}
\tikzset{
    -Latex,auto,node distance =1 cm and 1 cm,semithick,
    state/.style ={circle, draw, minimum width = 0.7 cm},
    detstate/.style ={rectangle, draw, minimum width = 0.7 cm, minimum height = 0.7 cm},
    point/.style = {circle, draw, inner sep=0.04cm,fill,node contents={}},
    bidirected/.style={Latex-Latex,dashed},
    el/.style = {inner sep=2pt, align=left, sloped}
}
\usepackage[algoruled, algo2e]{algorithm2e}
\usepackage{algorithm}
\usepackage{algorithmic}

\usepackage{tcolorbox}
\usepackage{stmaryrd}
\usepackage{enumitem}

\usepackage{amsmath}
\usepackage{accents}
\usepackage{graphicx} %

\usepackage{tabularx}
\usepackage{booktabs}
\usepackage{array}
\usepackage{ragged2e}
\usepackage{colortbl}

\tcbuselibrary{breakable}

\usepackage{pifont}

\definecolor{darkgreen}{RGB}{0, 128, 0}
\definecolor{darkred}{RGB}{178, 34, 34}  %
\newcommand{\cmark}{\textcolor{green!60!black}{\ding{51}}}
\newcommand{\xmark}{\textcolor{red!80!black}{\ding{55}}}

\newcommand{\tr}{\operatorname{tr}}

\newtheorem*{assumption*}{Assumption}
\newtheorem{corollary}{Corollary}
\newtheorem*{corollary*}{Corollary}
\newtheorem{definition}{Definition}
\newtheorem*{definition*}{Definition}
\newtheorem{lemma}{Lemma}
\newtheorem*{lemma*}{Lemma}
\newtheorem{proposition}{Proposition}
\newtheorem*{proposition*}{Proposition}
\newtheorem{theorem}{Theorem}
\newtheorem*{theorem*}{Theorem}

\renewcommand{\mid}{~\vert~}

\newcommand{\mbx}{\mathbf{x}}
\newcommand{\mby}{\mathbf{y}}

\newcommand{\cL}{\mathcal{L}}
\newcommand{\cN}{\mathcal{N}}
\newcommand{\cO}{\mathcal{O}}

\newcommand{\E}{\mathop{\mathbb{E}}} %

\newcommand{\bbR}{\mathbb{R}}
\newcommand{\R}{\mathbb{R}}

\newcommand{\Var}{\mathbb{V}\textrm{ar}}
\newcommand{\Cov}{\mathbb{C}\textrm{ov}}

\usepackage[utf8]{inputenc} %
\usepackage[T1]{fontenc}    %
\usepackage{hyperref}       %
\usepackage{url}            %
\usepackage{booktabs}       %
\usepackage{amsfonts}       %
\usepackage{nicefrac}       %
\usepackage{microtype}      %
\usepackage{xcolor}         %
\usepackage{amsmath}

 \usepackage[main, final]{neurips_2026}
\usepackage[utf8]{inputenc} %
\usepackage[T1]{fontenc}    %
\usepackage{hyperref}       %
\usepackage{url}            %
\usepackage{booktabs}       %
\usepackage{amsfonts}       %
\usepackage{nicefrac}       %
\usepackage{microtype}      %
\usepackage{xcolor}         %

\newcommand{\slim}{{\color{magenta}\text{slim}}} %
\usepackage{pifont}

\newtheorem{remark}{Remark}

\newcommand{\SG}{\mathrm{\textsc{sg}}}
\newcommand{\sgb}[1]{\SG\{#1\}}

\definecolor{dgreen}{rgb}{0.0, 0.52, 0.34}

\definecolor{dteal}{rgb}{0.0, 0.45, 0.50}

\usepackage{hyperref}
\usepackage{url}

\title{ How I learned to stop worrying and love StopGrads: Stationarity, Convergence, and a \\ case study on Flow Map Learning }

\author{%
    Max W. Shen$^{*}$\\
    Genentech \\
    \And 
    Mark Goldstein$^{*}$\\
    Flatiron Institute\\
    \And
    Zichu Wang\\
    MIT\\
    \AND
    Aahlad Puli\\
    NYU\\
    \And
    Rajesh Ranganath\\
    NYU
}

\begin{document}

\maketitle

\begin{abstract}
Stopgrads are widely used in training machine learning models, but stopgrads can alter the gradient, stationary points and convergence guarantees of the original objective, which can make stopgrad training theoretically ungrounded. We introduce a \textit{stopgrad regression principle}, which identifies a general template for stopgrad objectives with a closed-form characterization of stationary points and their uniqueness, unifying stopgrad objectives for flow maps, reinforcement learning, and diffusion samplers.
We provide theoretical grounding for optimizing stopgrad flow map objectives by showing their unique stationary point is the true flow map, and showing positive convergence results for Eulerian and Lagrangian objectives, including MeanFlow and improved MeanFlow. 
Remarkably, we show that under functional semi-gradient flow, the learned flow map has a closed-form expression composing the initial flow map and the true flow map.
We additionally use our stopgrad regression principle to propose modified stopgrad placements for flow map objectives which reduce training memory by $2\times$.
\end{abstract}

\section{Introduction}

Stop gradients (\textit{stopgrad}) are a basic primitive in automatic differentiation packages like PyTorch and JAX.
Stopgrads leave the forward pass unchanged, but in the backward pass, they return a zero gradient and block further gradient propagation. 
Stopgrads are widely used in machine learning, including in self-referential objectives where a model is supervised in terms of itself, as in
self-distillation, \gls{rl} \citep{tsitsiklis1997analysis}, consistency models and flow maps \citep{song2023consistency, geng2025meanflow}, representation learning \citep{ponce2026dual}, as well as variational inference \citep{roeder2017stickingthelanding, vaitl2024fast}.
For example, in \gls{rl}, popular algorithms such as temporal-difference learning and actor-critic methods rely on stopgrads in practice~\citep{mnih2015human,lillicrap2016continuous,silver2017mastering,sutton2018reinforcement}.

However, stopgrad objectives update models with \textit{semi-gradients}
(differentiating only some occurrences of the model), which are not the gradient of the original objective  \citep{sutton2018reinforcement}.
Semi-gradient update fields may be non-conservative, meaning they may not be the gradient of \textit{any} objective, and updates could follow a loop indefinitely without making progress. 
Whereas standard gradient descent is akin to marching through mountains and valleys, optimizing stopgrad objectives can be like wandering in Escher landscapes (Fig. \ref{fig:fields}).
These concerns motivate understanding stopgrad optimization theory, which seeks to understand the \textit{stationary points} and \textit{convergence} of stopgrad objectives.

\begin{figure}[t]
\centering
\begin{minipage}[b]{.39\textwidth}
  \centering
  \includegraphics[width=\linewidth]{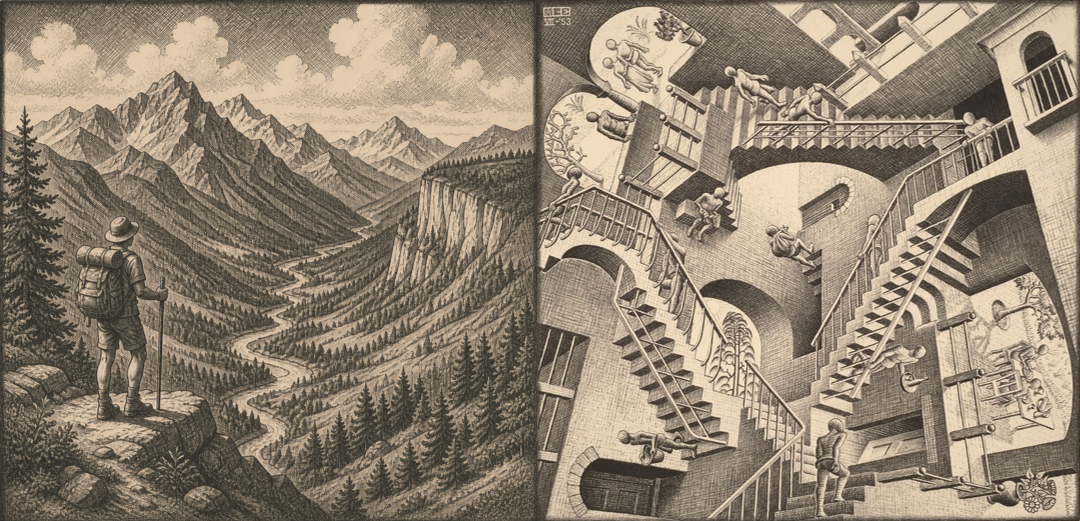}
\end{minipage}%
\hfill
\begin{minipage}[b]{.59\textwidth}
  \centering
  \includegraphics[width=\linewidth]{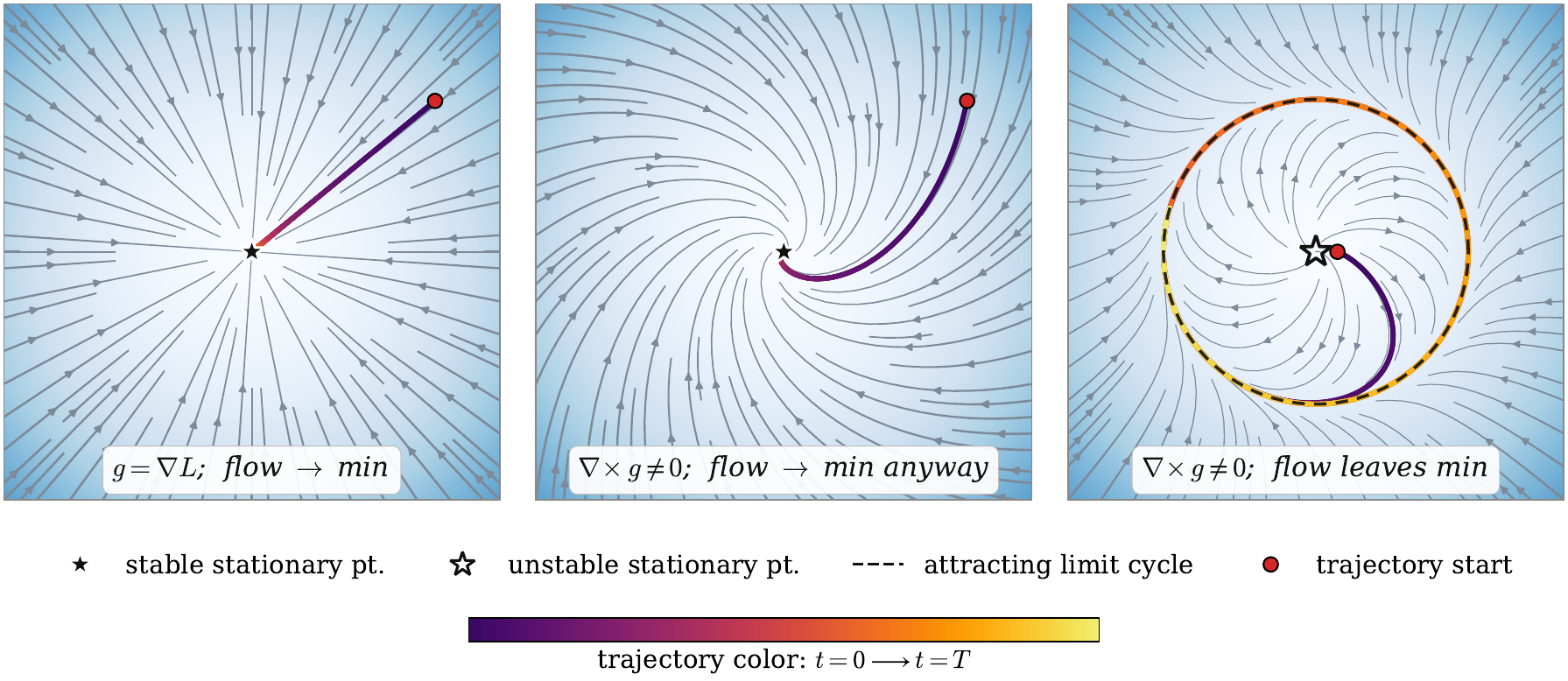}
\end{minipage}
\vspace{-5pt}
\caption{
\small{
    \textbf{Gradient and non-gradient update fields.}
    \textbf{Left:} Gradient fields are familiar valleys and mountains, while in stopgrad update fields, learning can make no progress and may not converge (\textit{Relativity, M.C. Escher}). 
    \textbf{Right:} 
      \textbf{(a)} A well-behaved gradient field.
      \textbf{(b, c)} Non-gradient fields can result from stopgrads.  While \textbf{(b)} still has a stable stationary point at the minimum, \textbf{(c)} has an \textit{unstable} stationary point; small perturbations move trajectories away from the minimum and toward the limit cycle. \textbf{In this work, we study a class of loss functions and stopgrad placements that yield 
     update fields that converge to desired stationary points.}
}}
\vspace{-10pt}
  \label{fig:fields}
\end{figure}

Recently, stopgrads have featured prominently in methods for training \textit{flow maps}, which learn to integrate probability-flow ODEs in few steps \citep{song2023consistency, song2023improved, lu2024simplifying, kim2023consistency, salimans2022progressive, boffi2024flowmapmatching,sabour2025align}.
MeanFlow \citep{geng2025meanflow} introduced a stopgrad objective but did not analyze its stationary points or convergence.
\citet{boffi2024flowmapmatching,boffi2025build} introduced the 
\gls{esd} and \gls{lsd} objectives
and proposed specific stopgrad placements;
they show the non-stopgrad loss has a unique minimum, but did not characterize the stationary points or convergence of the stopgrad variants.

Despite the ubiquity and empirical utility of stopgrads, 
no \textit{general} theoretical grounding exists for stopgrad objectives.
This paper formalizes stopgrads via the calculus of variations (\S\ref{sec:stopgrad_formalism}) and studies the \textit{stationary points} and \textit{convergence} of stopgrad objectives. Concretely, the paper makes the following contributions:

$\star$ \textbf{We introduce a \textit{stopgrad regression principle}} (\S\ref{sec:principle}), a general template for stopgrad objectives with a closed-form characterization of its stationary points as fixed points of the target operator.
The template describes objectives where the learned model appears once outside the stopgrad, and regresses against targets that are affine in labels and stopgrad terms of the learned model.
We also give a sufficient condition on the target operator under which a unique attainable minimizer of the ordinary (non-stopgrad) loss is the unique stationary point of the stopgrad loss, and we prove uniqueness directly for each flow map objective.
The stopgrad regression principle is general, encompassing stopgrad objectives for flow maps, \gls{rl}, and diffusion samplers \citep{havens2025adjointsampling}.
The principle's template is easy to check, so it assists in the design of new stopgrad objectives.

$\star$ 
\textbf{We apply the stopgrad regression principle to MeanFlow with stopgrad, and show that the true flow map is its unique stationary point}, as desired.
In contrast, removing the stopgrad removes the true flow map as a stationary point \citep{boffi2024flowmapmatching}, showing that the stopgrad in MeanFlow is necessary for principled optimization (\S\ref{sec:meanflow}).

$\star$
\textbf{We show that the true flow map is the unique stationary point of improved MeanFlow} (\S\ref{sec:imf}).

$\star$ \textbf{We derive two new stopgrad objectives by introducing stopgrads in placements prescribed by the regression principle template} in the ESD and LSD flow map objectives (\S\ref{sec:lsd}-\ref{sec:esd}), while preserving stationarity at the true flow map.
We call these \textit{slim} \gls{esd} and \gls{lsd}, as our stopgrads reduce training memory by $2\times$.
Concurrent work has proposed equivalent objectives \citep{kim2026stabilizing, potaptchik2026discreteflowmaps}, whereas we derive them in a principled manner from stopgrad optimization perspectives.
Slim \gls{esd} and \gls{lsd} match or improve on \gls{fid} on \textsc{cifar-10} and \textsc{ImageNet}.

$\star$ \textbf{We show positive convergence results for MeanFlow, \gls{imf}, slim \gls{esd}, slim \gls{lsd}} (\S\ref{sec:convergence}). We verify that their semi-gradients are not gradients of any objective, which raises questions on convergence.
For slim \gls{esd}, \gls{lsd}, and \gls{imf}, we study convergence of distillation training following perfect flow-matching pretraining, whereas we study MeanFlow training from scratch.
Remarkably, we show that under functional semi-gradient flow, the \textit{learned flow map is a closed-form composition of the initial flow map and the true flow map}, with ``hand-off time'' varying with optimization time (Fig. \ref{fig:ftau_closed_form}).
Under standard assumptions, convergence occurs exponentially fast.

\begin{figure}[!htbp]
  \centering
    \includegraphics[width=0.8\textwidth]{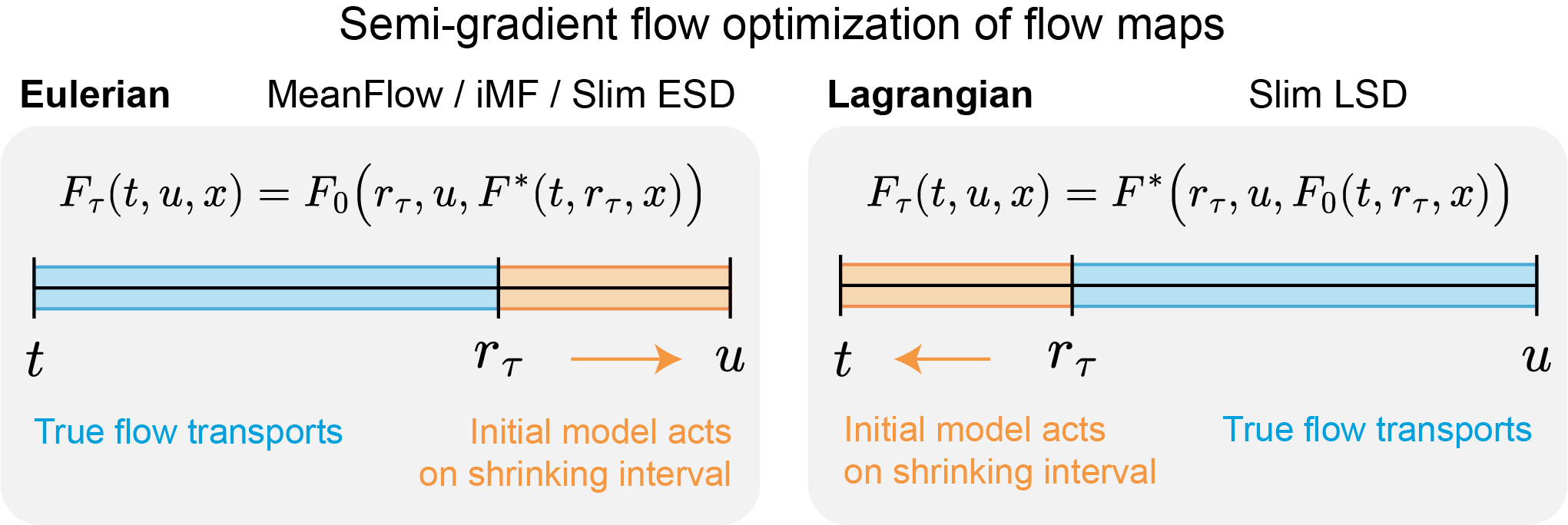}
  \caption{ \small{ \textbf{Convergence via closed-form composition (\S\ref{sec:convergence}).} Under functional semi-gradient flow, the learned flow map $F_\tau$ at optimization time $\tau$ is a composition of the initial $F_0$ and the true flow map $F^*$}. As optimization time $\tau$ passes, the ``hand-off'' time $r_\tau$ interpolates towards $u$ for Eulerian and $t$ for Lagrangian, giving convergence to the true flow map.
  \vspace{-5pt}
  }
  \label{fig:ftau_closed_form}
\end{figure}

\section{Stopgrads in ML: from PyTorch to Proofs}
\label{sec:stopgrad_formalism}
We formalize a mathematical definition of the stopgrad operator that is consistent with what is computed in practice.
We start with stopgrads computed in programming libraries, move to directional derivatives, and then generalize to functional derivatives, which form the basis of our results.

\paragraph{Stopgrads in practice}
Autodifferentiation frameworks such as PyTorch 
\citep{paszke2019pytorch} and JAX \citep{jax2018github} evaluate expressions in a forward pass, then compute parameter gradients in a backward pass via the chain rule.
The \emph{stopgrad} (or \texttt{detach}) operation acts as the identity in the forward pass but returns a zero gradient in the backward pass.
Let $\SG$ denote stopgrad as a mathematical operator.
The rule is:
\begin{align}
 \SG\{f_\theta\}:= f_\theta, \quad  \nabla_\theta^{\text{torch}} \SG\{f_\theta\} := 0.
\end{align}
As an example, consider:
\begin{align}
 \mathcal{L}_{\text{ex}}(\theta) := \mathbb{E}[\| f_\theta(\mbx)  - \sgb{\nabla_x f_\theta(\mbx)}\mby \|^2]
\end{align}
Due to the stopgrad, machine learning libraries would compute the parameter update direction:
\begin{align}
    \label{eq:example_functional}
   \nabla^{\text{torch}}_\theta  \mathcal{L}_{\text{ex}}(\theta)
    &=2\mathbb{E}[ (f_\theta(\mbx)  - \sgb{\nabla_x f_\theta(\mbx)}\mby) \cdot  \nabla_\theta
    f_\theta(\mbx) ].
\end{align}
This update direction is not the gradient of $\cL_{\text{ex}}$: the full gradient would also differentiate the Jacobian, computing $\nabla_\theta \nabla_x f_\theta$. We give this update direction a formal name below. 

\paragraph{Stopgrad for directional derivatives}
We now formalize the PyTorch stopgrad
as a precise operation on directional derivatives.
The directional derivative of $\cL$ at $\theta$ in direction $h$ is:
\begin{align}
    D_h \mathcal{L}(\theta)
    = \lim_{\epsilon \to 0}\frac{\mathcal{L}(\theta + \epsilon h) - \mathcal{L}(\theta)}{\epsilon}
    = \frac{d}{d\epsilon}\, \mathcal{L}(\theta + \epsilon h)\Big|_{\epsilon=0}.
\end{align}
We start by decomposing $\cL$ into a function
$\cO$ of two copies of $\theta$, with $\Theta_1$ tagging non-stopgrad occurrences and $\Theta_2$ tagging stopgrad occurrences.
For the example in \eqref{eq:example_functional},
\begin{align}
 \mathcal{O}_{\text{ex}}[\Theta_1, \Theta_2] = \mathbb{E}[\| f_{\Theta_1}(\mbx)  - \nabla_x f_{\Theta_2}(\mbx)\mby\|^2].
\end{align}
We then write $\cL(\theta) = \cO[\theta, \sgb{\theta}]$ as shorthand for the two-step procedure: tag each occurrence of $\theta$ as non-stopgrad ($\Theta_1$) or stopgrad ($\Theta_2$), then evaluate at $\Theta_1 = \Theta_2 = \theta$. The stopgrad rule then corresponds to differentiating only with respect to $\Theta_1$, i.e., perturbing only non-stopgrad parameters:
\begin{align*}
 D_h \mathcal{L}_{\text{ex}}(\theta)
   &=
   \frac{d}{d\epsilon}\, \mathbb{E}[\|f_{\theta + \epsilon h}(\mbx) - \nabla_x f_\theta(\mbx)\mby \|^2]\big|_{\epsilon = 0}
   =
   2\mathbb{E}\big[
    \big(f_{\theta}(\mbx) - \nabla_x f_\theta(\mbx)\mby  \big)
    \big[\nabla_\theta f_\theta(\mbx)\big]^\top
    h
    \big],
\end{align*}
which has maximum magnitude for $h$ in the direction $\nabla_\theta^{\text{torch}}\mathcal{L}$, matching the PyTorch computation.

\paragraph{Functional derivatives}
\label{appsec:flowmap_first_variation_definition}
To focus on the properties of objectives with stopgrads,
we generalize from finite-dimensional parameters to function space.
Working in function space makes the analysis coordinate-free: the first variation depends only on $f$ and a perturbation $h$, not on any parameterization, and characterizes which solutions can be found in the idealized limit of infinite model capacity.

We take $H := L^2(\mu, \mathbb{R}^d)$ as the pairing space where first variations are represented and semi-gradient residuals are compared. We take a linear subspace $\mathcal{M} \subset C^1 \cap H$ as the function or model class. We take  $\mathcal{D} \subset \mathcal{M}$ as the admissible perturbation class.
\begin{tcolorbox}[boxrule=0pt, frame empty]
\begin{definition}[First variation]
\label{def:first-variation}
The \emph{first variation} of a functional $\cL : \mathcal{M} \to \bbR$ at $f \in \mathcal{M}$ in direction $h \in \mathcal{D}$ is
\begin{align}
  \delta \cL[f; h]
  := \frac{d}{d\epsilon}\, \cL[f + \epsilon h]\big|_{\epsilon=0}.
\end{align}
\end{definition}
\end{tcolorbox}
For the functionals studied in this paper, direct computation 
(\S\ref{appsec:proofs_mf}) shows that
$\delta \mathcal{L}$
can be written as an inner product against an element $g[f] \in H$:
\begin{align}
  \delta \cL[f;\, h]
  \;=\;
  \E_\mu\!\big[\,g[f](\mbx)\cdot h(\mbx)\,\big]
  \;=\;
  \langle g[f],\, h\rangle_H,
  \qquad
  \langle \phi, \psi \rangle_H := \E_\mu[\phi(\mbx) \cdot \psi(\mbx)].
\end{align}
$g[f]$ is the \emph{functional gradient} of $\cL$ at $f$. Whenever $-g[f]$ is itself an admissible direction, $-g[f] \in \mathcal{D}$, it is the steepest-descent direction in $\langle \cdot, \cdot \rangle_\mu$.
We write $\langle \cdot, \cdot \rangle_\mu$ for $\langle \cdot, \cdot \rangle_H$ when the measure needs emphasis.

\paragraph{Stopgrads for functional derivatives}
\label{sec:sg_def}
We now generalize the two-argument construction from directional derivatives to functionals. Let us decompose $\mathcal{L}$ into a functional $\cO$ of two functions $f_1$ and $f_2$.
We then write $\cL[f] = \cO[f, \sgb{f}]$ as shorthand for labeling all non-stopgrad occurrences as $f_1$ and stopgrad occurrences as $f_2$, evaluating at $f_1=f_2=f$ such that the value of $\mathcal{L}[f]$ is $\cO[f,f]$, and computing the first variation by perturbing only the first argument:
\begin{align}
\label{eq:functional_stopgrad_rule}
  \delta \cL[f; h]
    = \frac{d}{d\epsilon}\, \cO[f + \epsilon h,\, f]\big|_{\epsilon=0}.
\end{align}
When $\cO$ depends non-trivially on its second argument, the resulting functional gradient 
$g$ such that $\delta \mathcal{L} = \langle g[f], h \rangle$ generally differs from the usual functional gradient; this is because the stopgrad rule perturbs only the first slot of $\cO$.

We call this object $g$ a \textit{semi-gradient}, borrowing terminology from \citet{sutton2018reinforcement}, where the term refers to the special case arising in \textsc{td} learning.

\begin{tcolorbox}[boxrule=0pt, frame empty]
\begin{definition}[Semi-gradient]
\label{def:semigradient}
Let $\cL[f] = \cO[f, \sgb{f}]$. The \emph{semi-gradient} of $\cL$ at $f$ is the element $g[f] \in H$ representing the first variation \eqref{eq:functional_stopgrad_rule}: $\delta \mathcal{L} = \langle g[f], h \rangle$.
\end{definition}
\end{tcolorbox}

The semi-gradient flow $\partial_\tau f_\tau = -g[f_\tau]$ is the function space analogue of the autograd update direction; updates stop precisely where $g = 0$.

\begin{tcolorbox}[boxrule=0pt, frame empty]
\begin{definition}[Stationary point]
\label{def:stationary}
$f^*$ is a stationary point of a functional $\cL$ if $\delta \cL[f^*; h] = 0$ for all $h \in \mathcal{D}$.
\end{definition}
\end{tcolorbox}
As shown by the fundamental lemma of variations (Lemma~\ref{lem:fundamental}), $\delta \cL=0$ for all $h$ implies that $g=0$ on $\operatorname{supp}\mu$.
Our theoretical results thus characterize stationary points by computing $\delta \mathcal{L}$ using this stopgrad rule, and solving for $g=0$.

\section{ Stopgrad Regression Principle }
\label{sec:principle}

Here, we present a unifying stopgrad regression principle.
This principle identifies a general template for stopgrad objectives where the learned function appears in only one term outside of stopgrads, though it can appear within stopgrads in multiple terms.
When the stopgrad-ed operator takes a certain affine form, our stopgrad regression principle characterizes the stationary points as the fixed points of the target operator.
Whether that fixed point is unique is a separate question: we give a sufficient condition in Corollary~\ref{cor:uniqueQ}, and establish uniqueness for each flow map objective in \S\ref{sec:flowmap_stationary}.

\begin{tcolorbox}[boxrule=0pt, frame empty, breakable]
    \begin{theorem}[Stopgrad Regression Principle]
    \label{thm:principle}
    Let \(\mbx\) denote the inputs, with distribution \(\mu\),
    and let \(\mby\) denote the labels.
    Assume \(\mu\) has full support on an open set \(S \subseteq \bbR^n\) (that is, \(\mu(U) > 0\) for every nonempty open \(U \subseteq S\), and that \(\operatorname{supp}\mu \subseteq \bar S\), the closure of \(S\)).
    Let \(T[f,\mby]\) be an operator that is
    \emph{affine in \(\mby\)}, meaning
    \begin{align}
      \label{eq:affine-target}
      T[f,\mby](\mbx)
      =
      P[f](\mbx) + Q[f](\mbx)\,\mby + R(\mbx)\,\mby
    \end{align}
    for operators \(P,Q\) that depend on \(f\) but not on \(\mby\), and \(R\) independent of both \(f\) and \(\mby\).
    Consider the functional
    \begin{align}
      \label{eq:principle_loss}
      \cL[f]
      :=
      \frac12\,\E_{\mbx,\mby}\Big[
        \big\|
          f(\mbx) - \sgb{T[f,\mby](\mbx)}
        \big\|^2
      \Big].
    \end{align}
    Let \(f^* \in \mathcal{M}\) with \(\cL[f^*] < \infty\) and \(\mbx \mapsto T[f^*,\,\E[\mby \mid \mbx]](\mbx)\) continuous on \(\bar S\).
    Then \(f^*\) is a stationary point of \(\cL\)
    (Definition~\ref{def:stationary}, with \(\mathcal{D} \supseteq C^1_c(S)\), the \(C^1\) functions with compact support in \(S\))
    if and only if
    \begin{align}
      \label{eq:principle_condition}
      f^*(\mbx) = T[f^*,\, \E[\mby \mid \mbx]](\mbx)
      \quad \text{for all } \mbx \in \operatorname{supp}\mu.
    \end{align}
    \end{theorem}
\end{tcolorbox}
The proof is in \S\ref{appsec:proofs_mf}.
The key idea is that since \(f\) appears exactly once outside \(\SG\),
the first variation is represented in $H$ as
\(\delta \cL[f;h] = \langle g[f], h\rangle_\mu\),
where the semi-gradient \(g[f] \in H\) is given pointwise by $g[f](\mbx) = f(\mbx) - T[f,\, \E[\mby \mid \mbx]](\mbx)$.
Since $g[f]$ is continuous, the fundamental lemma of variations (Lemma~\ref{lem:fundamental}) then shows that stationarity is equivalent to \(g[f] = 0\) on $\operatorname{supp}\mu$.
In Table \ref{tab:form1_six_examples} and \S\ref{appsec:unifying_stopgrad_objectives}, we show that the stopgrad regression principle unifies objectives for flow maps, \gls{rl}, and diffusion samplers.

In general, the stationary set for \eqref{eq:principle_condition} may have several solutions or none.
Here, we provide a simple condition, called \textit{additive label noise}, that characterizes uniqueness by Corollary \ref{cor:uniqueQ}.

\begin{tcolorbox}[boxrule=0pt, frame empty]
    \begin{definition}[Additive label noise]
    \label{def:additive}
    The target operator $T$ in \eqref{eq:affine-target} has \emph{additive label noise} if $Q[f] \equiv Q$ does not depend on $f$.
    \end{definition}
\end{tcolorbox}

The loss can be decomposed as $\cL[f] = \tfrac12\|g[f]\|_\mu^2 + V[f]$, where $V[f]$ is a conditional-variance ``floor'' (Proposition~\ref{prop:floor}, \S\ref{appsec:proofs_mf}).
In general, $V[f]$ can depend on $f$ through $Q[f]$, so comparing loss values across candidates says nothing about which are stationary.
However, under additive noise, $V[f]$ is independent of $f$, allowing uniqueness to be read off the loss (Cor. \ref{cor:uniqueQ}, proof in \S\ref{appsec:proofs_mf}.).

\begin{tcolorbox}[boxrule=0pt, frame empty]
\begin{corollary}[Uniqueness under additive label noise]
\label{cor:uniqueQ}
Assume Theorem~\ref{thm:principle} applies to every
$f\in\mathcal M$ and $T$ has additive label noise.
Let $V:=\tfrac12\E\|(Q+R)\xi\|^2$, where
$\xi:=\mby-\E[\mby\mid\mbx]$.
Then $\cL[f]\ge V$, with equality if and only if $f$ is stationary.
Consequently, if exactly one $f^*\in\mathcal M$ attains $V$,
it is the unique stationary point.
\end{corollary}
\end{tcolorbox}

The proof is in \S\ref{appsec:proofs_mf}.
Additive label noise holds for \gls{imf}, slim \gls{lsd}/\gls{esd},
and adjoint RAM, but not for MeanFlow or TD(0)/Q-learning
(Table~\ref{tab:form1_six_examples}); for these, we prove uniqueness with separate arguments.

\paragraph{Connection to fixed-point iteration}
In reinforcement learning, the value function is the fixed point of the Bellman operator \citep{bellman1954theory, sutton2018reinforcement, bertsekas2025neuro} in the \textsc{td(}{\small 0}\textsc{)} row of Table \ref{tab:form1_six_examples}. How does this relate to stop-gradient regression? Taking the semi-gradient of the stop-gradient regression and applying a unit step yields
\[
f \;\leftarrow\; f - (f - T [f,\mathbb{E}[\mby|\mbx]]) \;=\; T[f,\mathbb{E}[\mby|\mbx]]
\]
which is exactly fixed-point iteration of $T$. The update vanishes precisely when $f=T[f]$, so the zeros of the semi-gradient coincide with 
the fixed-point of $T$ by construction. In the \gls{rl} case, 
this fixed point is the desired value function, and it is unique because the Bellman operator is a $\gamma$-contraction for $\gamma < 1$ \citep{shapley1953stochastic,blackwell1965discounted,denardo1967contraction} (\S\ref{appsec:unifying_stopgrad_objectives}).
Our perspective highlights that this equivalence is not specific to Bellman iteration: semi-gradient updates of stop-gradient regression implement fixed-point iteration of the underlying operator, including for problems not typically framed in fixed-point terms,
such as \gls{pde}-based flow-map loss functionals.

\begin{table*}[t]
    \centering
    \scriptsize
    \setlength{\tabcolsep}{3pt}
    \renewcommand{\arraystretch}{0.0}
    
    \arrayrulecolor{black}
    \begin{tabularx}{\textwidth}{
    >{\RaggedRight\arraybackslash}p{1.6cm}
    >{\RaggedRight\arraybackslash}p{0.9cm}
    >{\RaggedRight\arraybackslash}p{1.1cm}
    >{\RaggedRight\arraybackslash}X
    >{\RaggedRight\arraybackslash}p{3.0cm}
    >{\RaggedLeft\arraybackslash}p{1.5cm}
    >{\Centering\arraybackslash}p{1.1cm}
    }
    \toprule
    Objective & \(\mbx\) & \(\mby\) & \(P[f](\mbx)\) & \(Q[f](\mbx)\mby\) & \(R(\mbx)\mby\) & Additive noise \\
    \midrule
    
    \arrayrulecolor{black!15}
    MeanFlow
    &
    \((t,u,x)\)
    &
    \(\dot x_t\)
    &
    \((u-t)\,\partial_t f(\mbx)\)
    &
    \((u-t)\,\partial_x f(\mbx)\dot \mbx_t\) 
    &
    \(\dot x_t\)
    &
    \xmark
    \\
    \midrule

    Improved MeanFlow
    &
    \((t,u,x)\)
    &
    \(\dot x_t\)
    &
    \(
    (u-t)\big[
      \partial_t f(\mbx)
      +(\partial_x f)(\mbx)\,f(t,t,x)
    \big]
    \)
    &
    \(0\)
    &
    \(\dot x_t\)
    &
    \cmark
    \\
    \midrule
    
    Slim \gls{lsd} (ours)
    &
    \((t,u,x)\)
    &
    \(0\)
    &
    \(
    -(u-t)\,\partial_u f(\mbx) +\,f\!\big(u,u,\,
    x+(u-t)f(\mbx)\big)
    \)
    &
    \(0\)
    &
    \(0\)
    &
    \cmark
    \\
    \midrule
    
    Slim \gls{esd} (ours)
    &
    \((t,u,x)\)
    &
    \(0\)
    &
    \((u-t)\,\partial_t f(\mbx) +\,(\partial_x F)(\mbx)\,f(t,t,x)
    \)
    &
    \(0\)
    &
    \(0\)
    &
    \cmark
    \\
    \midrule

    \textsc{td(0)}    
    &
    \(s\)
    &
    \((r,\delta_{S_{t+1}})\)
    &
    \(0\)
    &
    \(\gamma \int V(s')\,\delta_{S_{t+1}}(ds')\)
    &
    \(r\)
    &
    \xmark
    \\
    \midrule
    
    Q-learning
    &
    \((s,a)\)
    &
    \((r,\delta_{S_{t+1}})\)
    &
    \(0\)
    &
    \(\gamma \int \max_{a'} q(s',a')\,\delta_{S_{t+1}}(ds')\)
    &
    \(r\)
    &
    \xmark
    \\
    \midrule
    
    Adjoint RAM
    &
    \((x,t)\)
    &
    \(\nabla g(X_1)\)
    &
    \(0\)
    &
    \(0\)
    &
    \(-\sigma(t) \nabla g(X_1)\)
    &
    \cmark
    \\
    
    \arrayrulecolor{black}
    \bottomrule
    \end{tabularx}
    \caption{
        \small{
        Stopgrad objectives unified by our stopgrad regression principle (\S\ref{appsec:unifying_stopgrad_objectives}).
        We propose slim \gls{esd} and slim \gls{lsd} by modifying existing flow map objectives to be in line with our stopgrad regression principle.
        The last column records whether the label noise is additive (Definition~\ref{def:additive}), i.e.\ whether the coefficient $Q[f]$ of the label is independent of $f$.
        }
    }
    \label{tab:form1_six_examples}
\end{table*}

\section{ Stationary Points of Stopgrad Flow Map Objectives }
\label{sec:flowmap_stationary}

Flow maps learn to sample probability flow \glspl{ode} in few steps, and are trained with objectives involving stopgrad operators.
We now analyze the stationary points of these stopgrad flow map objectives. We begin by reviewing background.

\paragraph{Stochastic Interpolants}
Given  simple  $\mbx_0 \sim q_0$ and data $\mbx_1 \sim q_1$, 
stochastic interpolants \citep{lipman2022flow,albergo2023stochastic} define a density path by constructing  $\mbx_t = \alpha_t\,\mbx_0 + \sigma_t\,\mbx_1$, for $t \in [0,1]$.
Under this construction, there exists a velocity field $v(t,x) := \E[\dot \mbx_t \mid \mbx_t = x]$ with $\dot \mbx_t = \dot \alpha_t \mbx_0 + \dot \sigma_t \mbx_1$, such that the \gls{ode} $d\hat{\mbx}_t = v(t,\hat{\mbx}_t)\,dt$ has the same marginal densities as $\mbx_t$.
The \gls{ode} connection provides a way to sample the marginals of $\mbx_t$ while only having access to a sample of one endpoint and the velocity field $v$.
In practice, $v$ is learned by regressing a model $v_\theta$ against  $\dot \mbx_t$ via the so-called flow-matching loss $\E[\|v_\theta(t,\mbx_t) - \dot \mbx_t\|^2]$, and new samples are generated by discretizing the \gls{ode}.

\paragraph{Flow maps and the transport equation}
\label{sec:flow_maps}

Instead of learning $v_\theta$ and then integrating it in an \gls{ode}, one can sample by approximating the \emph{flow map} $F$: %
\begin{align}
\label{eq:flow_map_defn}
    F(t,u,\hat{\mbx}_t) := \hat{\mbx}_t + \int_t^u v(s, \hat{\mbx}_s) ds = \hat{\mbx}_t + \int_t^u v(s, F(t,s,\hat{\mbx}_t)) ds.
\end{align}
The flow map 
$F$
describes where a particle at position $x$ at time $t$ reaches by time $u$ when integrated along the probability flow \gls{ode} trajectories. Learning $F$ directly enables few-step sampling \citep{song2023consistency,kim2023consistency,boffi2024flowmapmatching,boffi2025build,song2023improved,geng2025meanflow}.
The flow map is characterized by a \emph{transport equation}, which is
a \gls{pde}
acquired by using the Leibniz rule
to take the time derivative of \cref{eq:flow_map_defn}:
\begin{align}
  \label{eq:transport_intro}
  \partial_t F(t,u,x) + (\partial_x F)(t,u,x)\, v(t,x) = 0,
  \qquad F(u,u,x) = x.
\end{align}
The boundary condition 
says that a particle at position $x$ at time $u$ is already at its destination; it is satisfied automatically when parameterizing flow maps as
\begin{align}
  \label{eq:f_param}
  F(t,u,x) = x + (u-t)\,f(t,u,x).
\end{align}
Rewriting \eqref{eq:transport_intro} in terms of this $f$ parameterization gives:
\begin{align}
  \label{eq:ftilde_transport}
  f = v + (u-t)\big(\partial_t f + (\partial_x f)\, v\big).
\end{align}
Flow-map training procedures are built from this identity and related rewritings. In practice, these methods often use stopgrads to avoid differentiating through model derivatives.
In this work, we seek to understand what these stopgrads preserve and when their induced dynamics remain well behaved.
Throughout this section, we assume the standing conditions of
Remark~\ref{rmk:regularity}, including the velocity regularity
used in prior work \citep{lipman2024flowmatchingguidecode}.

\paragraph{Stationary points of flow map training with stopgrads.}
We apply Theorem~\ref{thm:principle} to MeanFlow, improved MeanFlow,
and our slim \gls{lsd}/\gls{esd} objectives, whose stopgrad placements also avoid second-order derivatives.
Let $\Omega := \{0<t\le u<1\}\times\bbR^d$.
The training distribution $\mu$ of $(t,u,\mbx_t)$ has full support
on $\Omega^\circ$ and may place mass on the diagonal.

\subsection{MeanFlow}
\label{sec:meanflow}

MeanFlow \citep{geng2025meanflow} enforces the transport equation
while stopping gradients through Jacobian--vector products.
In our forward-time convention (\S\ref{appsec:t_vs_u}), its objective is:
\begin{align}
    \label{eq:meanflow_func}
    \cL_{\mathrm{MF}}[f]
    =  \frac 1 2 \E\Big[\big\|
    f(t,u,\mbx_t) - \dot \mbx_t
    - (u-t)\,\sgb{\partial_t f(t,u,\mbx_t)
      + \partial_x f(t,u,\mbx_t)\,\dot \mbx_t}
    \big\|^2\Big].
\end{align}
\begin{tcolorbox}[boxrule=0pt, frame empty]
\begin{corollary}[MeanFlow]
    \label{cor:meanflow}
    $f \in \mathcal{M}$ is a stationary point of $\cL_{\mathrm{MF}}$
    if and only if
    \begin{align*}
        f = v + (u-t)\big(\partial_t f + \partial_x f\, v\big) \quad \text{on } \Omega .
    \end{align*}
    Its unique solution is $f^*$, defined by $x + (u-t)f^*(t,u,x) = F^*(t,u,x)$.
    Thus, the true flow map is the unique stationary point.
\end{corollary}
\end{tcolorbox}
The characterization follows from Theorem~\ref{thm:principle};
uniqueness follows from the transport equation
(Lemma~\ref{lem:char}, \S\ref{app:corollaries}).
Since $Q[f]=(u-t)\partial_x f$ depends on $f$, the additive-noise
Corollary~\ref{cor:uniqueQ} does not apply.
The population stationarity equation was also derived independently
by \citet{kim2026stabilizing}.

In \S\ref{appsec:proofs_mf_nosg}, we further prove that the true flow map is not a stationary point of the \emph{simplified} MeanFlow objective obtained by removing the stopgrad, i.e.\ the objective in which the sample velocity $\dot \mbx_t$ replaces the marginal velocity $v$ inside the Jacobian term \citep{boffi2024flowmapmatching}; the corresponding covariance correction also appears in \citet{kim2026stabilizing}.
Thus, the stopgrad in MeanFlow is what makes that simplified objective principled.

\subsection{Improved MeanFlow}
\label{sec:imf}

Improved MeanFlow \citep{geng2025improved} replaces the $\dot \mbx_t$ term in the JVP with the model's instantaneous velocity $f(t, t, \mbx_t)$:
\begin{equation}
    \label{eq:imf_func}
    \cL_{\mathrm{iMF}}[f]
    = \frac 1 2 \E\Big[\big\|
    f(t,u,\mbx_t) - \dot \mbx_t
    - (u-t)\,\sgb{\partial_t f(t,u,\mbx_t)
      + \partial_x f(t,u,\mbx_t)\, f(t, t, \mbx_t)}
    \big\|^2\Big].
\end{equation}
\begin{tcolorbox}[boxrule=0pt, frame empty]
\begin{corollary}[Improved MeanFlow]
    \label{cor:imf}
    $f \in \mathcal{M}$ is a stationary point of $\cL_{\mathrm{iMF}}$
    if and only if
    \begin{align*}
        f = v + (u-t)\big(\partial_t f + (\partial_x f)\, f(t,t,\cdot)\big) \quad \text{on } \Omega .
    \end{align*}
    When $u=t$, this condition implies $f(t,t,\cdot) = v(t,\cdot)$. Applying this to $u\neq t$ recovers the transport equation \eqref{eq:transport_intro}, so the true flow map is the unique stationary point.
\end{corollary}
\end{tcolorbox}
The proof is in \S\ref{app:corollaries}.
Since $Q\equiv0$, Corollary~\ref{cor:uniqueQ} also identifies
the true flow map as the unique stationary point.
To our knowledge, this is the first uniqueness result for
stationary points of improved MeanFlow.

\subsection{Lagrangian Self-Distillation}
\label{sec:lsd}

The \gls{lsd} family of losses is based on the \emph{Lagrangian flow map identity} \citep{boffi2024flowmapmatching, boffi2025build}, a rewrite of \eqref{eq:transport_intro} that avoids spatial Jacobians.
For the true flow map $F^*$,
\begin{align}
  \label{eq:lagrangian_identity}
  \partial_u F^*(t,u,x) = v\big(u,\, F^*(t,u,x)\big).
\end{align}
Because the true velocity $v$ is unknown, \gls{lsd} replaces it with the model's prediction $f(u,u, \cdot)$, which equals $v(u,\cdot)$ at the true flow map.
The base \gls{lsd} functional without stopgrads penalizes violation of \eqref{eq:lagrangian_identity}.
The velocity condition is enforced by adding the flow matching term $\cL_{\mathrm{FM}}[f] := \tfrac12\,\E\|f(t,t,\mbx_t) - \dot \mbx_t\|^2$.

\gls{lsd} variants differ in where stopgrads are placed.
The original variant \citep{boffi2025build}---also used in Terminal Velocity Matching \citep{zhou2025terminal}---places the stopgrad only around the evaluation at the transported point \eqref{eq:lsd-orig}.
This does not fit Theorem~\ref{thm:principle}, so it is not obvious what its stationary points are.
\begin{align}
    \label{eq:lsd-orig}
    \cL_{\mathrm{LSD\text{-}orig}}[f]
    = \frac 1 2 \E\Big[\big\|
    f(t,u,\mbx_t) + (u-t)\partial_u f(t,u,\mbx_t)
    - \sgb{f\big(u,u, F(t,u,\mbx_t)\big)}
    \big\|^2\Big].
\end{align}

\paragraph{{\color{magenta}Slim} stopgrad for \gls{lsd}.}
Placing the derivative term inside the stopgrad brings the objective
within Theorem~\ref{thm:principle} and avoids second-order derivatives:
\begingroup
\thinmuskip=2mu
\medmuskip=2mu
\thickmuskip=3mu
\begin{align}
    \label{eq:lsd_slim_func}
    \cL_{\mathrm{LSD\text{-}\slim}}[f]
    = \frac 1 2 \E\Big[\big\|
    f(t,u,\mbx_t)
    + (u-t)\sgb{ \partial_u f(t,u,\mbx_t)}
    - \sgb{f\big(u,u, F(t,u,\mbx_t) \big) }
    \big\|^2\Big].
\end{align}
\endgroup
\begin{tcolorbox}[boxrule=0pt, frame empty]
\begin{corollary}[slim \gls{lsd}]
    \label{cor:lsd_slim}
    $f \in \mathcal{M}$ is a stationary point of $\cL_{\mathrm{LSD\text{-}\slim}}$
    if and only if
    \begin{align*}
        f + (u-t)\,\partial_u f = f\big(u,u,\, F(t,u,x)\big) \quad \text{on } \Omega,
        \qquad F := x + (u-t)f,
    \end{align*}
    equivalently $\partial_u F(t,u,x) = f\big(u,u,F(t,u,x)\big)$, the Lagrangian identity \eqref{eq:lagrangian_identity} with the model velocity $f(u,u,\cdot)$ in place of $v(u,\cdot)$.
\end{corollary}
\end{tcolorbox}
The corollary concerns the slim \gls{lsd} term alone.
For $\cL_{\mathrm{LSD\text{-}\slim}}+\cL_{\mathrm{FM}}$,
stationarity additionally requires $f(u,u,\cdot)=v(u,\cdot)$,
yielding the flow \gls{ode} and its unique solution $F^*$.
Lemma~\ref{lem:split} rules out cancellation between the two
first variations; proofs are in \S\ref{app:corollaries}.

\subsection{Eulerian Self-Distillation}
\label{sec:esd}

\gls{esd} is based on the transport equation
\eqref{eq:transport_intro}.
Like \gls{lsd}, \gls{esd} replaces the velocity with the model's prediction, and includes a flow matching term.
The original stopgrad placement from \citep{boffi2025build} is:
\begin{align}
    \label{eq:esd_orig_func}
    \cL_{\mathrm{ESD\text{-}orig}}[f]
    = \frac 1 2 \E\Big[\big\|
    \partial_t F(t,u,\mbx_t)
    + \sgb{(\partial_x F)(t,u,\mbx_t)\,f(t,t,\mbx_t)}
    \big\|^2\Big].
\end{align}
This does not fit the stopgrad regression principle, since 
$(u-t)\,\partial_t f -f$ appears outside the $\SG$.

\paragraph{{\color{magenta}Slim} stopgrad for \gls{esd}.}
As for slim \gls{lsd}, stopping gradients through the derivative
terms brings the objective within Theorem~\ref{thm:principle}
and avoids second-order derivatives:
\begingroup
\thinmuskip=2mu
\medmuskip=2mu
\thickmuskip=3mu
\begin{align}
    \label{eq:esd_slim_func}
    \cL_{\mathrm{ESD\text{-}\slim}}[f]
    = \frac 1 2 \E\Big[\big\|
    f(t,u,\mbx_t)
    - \sgb{(u-t)\,\partial_t f(t,u,\mbx_t)
      + (\partial_x F)(t,u,\mbx_t)\,f(t,t,\mbx_t)}
    \big\|^2\Big].
\end{align}
\endgroup
\begin{tcolorbox}[boxrule=0pt, frame empty]
\begin{corollary}[slim \gls{esd}]
    \label{cor:esd_slim}
    $f \in \mathcal{M}$ is a stationary point of $\cL_{\mathrm{ESD\text{-}\slim}}$
    if and only if
    \begin{align*}
        f = (u-t)\,\partial_t f + (\partial_x F)\, f(t,t,\cdot) \quad \text{on } \Omega,
        \qquad F := x + (u-t)f,
    \end{align*}
    equivalently $\partial_t F + (\partial_x F)\,f(t,t,x) = 0$: the transport equation \eqref{eq:transport_intro} with the model velocity $f(t,t,\cdot)$ in place of $v(t,\cdot)$.
\end{corollary}
\end{tcolorbox}
For $\cL_{\mathrm{ESD\text{-}\slim}}+\cL_{\mathrm{FM}}$,
stationarity additionally requires $f(t,t,\cdot)=v(t,\cdot)$,
yielding the transport equation with boundary condition
$F(u,u,x)=x$ and unique solution $F^*$.
The proof is in \S\ref{app:corollaries}.

We summarize all flow map stopgrad objectives discussed in this work in Table \ref{tab:flowmap_objective_comparison}.

\section{Convergence of Functional Semi-Gradient Flow for Flow Map Objectives}
\label{sec:convergence}

It is well known that standard gradient descent converges to stationary points under standard smoothness and step-size conditions
\citep{hestenes1952methods}.
Without stopgrads, the functional gradient $g$ likewise defines a descent direction (when $-g[f] \in \mathcal{D}$), and flowing along it brings optimization to a stationary point in function space.
However, for stopgrad objectives which have semi-gradient update fields, it is not obvious whether following the semi-gradient field arrives at a stationary point or not. Thus, our preceding characterization of stationary points for stopgrad flow map objectives must be paired with an analysis of convergence.

Here we show two things. First, semi-gradients are generally not gradients of any scalar objective; we confirm this concretely for MeanFlow, \gls{imf}, slim-\gls{lsd}, and slim-\gls{esd}, so standard optimization theory does not guarantee their convergence. Second, despite this, functional semi-gradient flow does converge for MeanFlow, \gls{imf}, slim-\gls{lsd}, and slim-\gls{esd}.

\paragraph{When is a semi-gradient the gradient of an objective?}

Could the semi-gradient $g[f]=f-T[f]$, where
$T[f]:=T[f,\E[\mby\mid\mbx]]$, be the gradient of another
scalar functional?
A necessary condition is symmetry of $DT[f]$ in
$\langle\cdot,\cdot\rangle_\mu$, by the classical symmetry
criterion for potential operators
\citep[Theorem~5.1]{vainberg1964variational}.
Theorem~\ref{thm:g-conservative}
(\S\ref{appsec:update_is_a_gradient}) states this condition
precisely. It fails for all four flow map objectives:

\begin{tcolorbox}[boxrule=0pt, frame empty]
\begin{corollary}
    \label{cor:non-conservative}
    The semi-gradients of $\cL_{\mathrm{MF}}$, $\cL_{\mathrm{iMF}}$, $\cL_{\mathrm{LSD\text{-}\slim}}$, and $\cL_{\mathrm{ESD\text{-}\slim}}$ are not gradients of any twice continuously differentiable scalar functional $J : \mathcal{M} \to \R$ in $\langle \cdot, \cdot \rangle_\mu$ (in the sense of Theorem~\ref{thm:g-conservative}).
\end{corollary}
\end{tcolorbox}

\paragraph{Convergence.}
We now show convergence results for functional semi-gradient flow for the stopgrad flow-map objectives: MeanFlow, improved MeanFlow, slim \gls{esd}, and slim \gls{lsd}.
We study MeanFlow directly, while for \gls{imf}, slim \gls{esd} and \gls{lsd}, we consider distillation training following perfect flow-matching pretraining.
This choice directly treats the matter of \textit{flow map convergence} for a given velocity field, and aligns with the separation of flow matching pretraining and flow map post-training setup often used in practice.
We succinctly state our setup and results here, and relegate full derivations and proofs to \S\ref{appsec:convergence}.

We have shown that Theorem \ref{thm:principle} applies to MeanFlow, \gls{imf}, slim \gls{esd}, and slim \gls{lsd}, so that under the functional inner product $\langle \cdot, \cdot \rangle_\mu$, the semi-gradient flow of each stopgrad term is $\partial_\tau f_\tau = -g[f_\tau]$, over optimization time $\tau$. 
For clarity of exposition, we choose a setting where $g: \mathcal{M} \to \mathcal{M}$, such that the semi-gradient flow on $f \in \mathcal{M}$ stays in $\mathcal{M}$.

We write the semi-gradient flow for the four objectives (eqs.~\ref{eq:meanflow_func}, \ref{eq:imf_func}, \ref{eq:lsd_slim_func}, \ref{eq:esd_slim_func}), assuming perfect flow-matching pretraining for slim \gls{esd}, slim \gls{lsd}, and \gls{imf} but not for MeanFlow (\S\ref{appsec:form_of_functional_gradient_flows}). 
The Eulerian objectives (MeanFlow, \gls{imf}, slim \gls{esd}) share an Eulerian semi-gradient flow, separate from the Lagrangian (slim \gls{lsd}) semi-gradient flow:
\begin{align}
    \text{Eulerian:}\quad \partial_\tau F_\tau(t,u,x) &= (u-t) \Big( \partial_t F_\tau(t,u,x)+(\partial_x F_\tau)(t,u,x)\,v(t,x) \Big) \label{eq:eulerian-gradient-flow} \\
    \text{Lagrangian:}\quad \partial_\tau F_\tau(t,u,x) &= -(u-t) \Big( \partial_u F_\tau(t,u,x)-v \bigl(u,F_\tau(t,u,x)\bigr) \Big) \label{eq:lagrangian-gradient-flow}
\end{align}
We now give our main result: These dynamics admit closed-form solutions composing $F_0$ and $F^*$ through an optimization-dependent ``hand-off'' time (Fig.~\ref{fig:ftau_closed_form}), establishing convergence:
\begin{tcolorbox}[boxrule=0pt, frame empty, breakable]
\begin{proposition}[Convergence]
    \label{prop:rectifying-templates-unified}
    Let \(\mathcal A := \{(t,u)\in[0,1]^2 : t \leq u \}\).
    Suppose the usual regularity conditions for flow maps (Remark~\ref{rmk:regularity}), such that the velocity field generates a unique $C^1$ flow map $F^*: [0,1]^2 \times \mathbb{R}^d \to \mathbb{R}^d$.
    Let \(F_0:\mathcal A\times\mathbb R^d\to\mathbb R^d\) be \(C^1\) and satisfy
    \(F_0(t,t,x)=x\).

    Denote ``hand-off'' times $s^\tau_{t{\shortrightarrow}u} := u-e^{-\tau}(u-t)$ and $s^\tau_{u{\shortrightarrow}t} := t+e^{-\tau}(u-t)$.
    Since $e^{-\tau} \in (0, 1]$, both $s^\tau_{t{\shortrightarrow}u}$ and $s^\tau_{u{\shortrightarrow}t}$ are convex combinations of $t$ and $u$.

    \begin{enumerate}[leftmargin=*]
        \item[\textbf{(a)}] \textbf{Eulerian.}
        Suppose $\partial_\tau F_\tau$ obeys equation \ref{eq:eulerian-gradient-flow}. Then,
        \begin{equation}
            F_\tau(t,u,x) = F_0
            \bigl(
                s^\tau_{t{\shortrightarrow}u},
                u,
                F^*(t,s^\tau_{t{\shortrightarrow}u},x)
            \bigr).
            \label{eq:f-explicit-eulerian-unified}
        \end{equation}
        In particular, for every fixed \((t,u,x)\), we have $F_\tau(t,u,x)\to F^*(t,u,x)$ as $\tau\to\infty$.

        \item[\textbf{(b)}] \textbf{Lagrangian.}
        Suppose $\partial_\tau F_\tau$ obeys equation \ref{eq:lagrangian-gradient-flow}. Then,
        \begin{equation}
            F_\tau(t,u,x)
            =
            F^*\bigl(s^\tau_{u{\shortrightarrow}t}, u, F_0(t,s^\tau_{u{\shortrightarrow}t},x)
            \bigr)
            .
            \label{eq:f-explicit-lagrangian-unified}
        \end{equation}
        In particular, for every fixed \((t,u,x)\), we have $F_\tau(t,u,x)\to F^*(t,u,x)$ as $\tau\to\infty$.
    \end{enumerate}
\end{proposition}
\end{tcolorbox}
We provide the proof in \S\ref{appsec:convergence_proof}. 
Under additional mild conditions, convergence occurs exponentially fast: $\| F_\tau(t,u,x) - F^*(t,u,x)\| \leq ce^{-\tau}(u-t)$ for a positive constant $c$ (Proposition \ref{prop:exp-rates-unified}, \S\ref{appsec:exponential_convergence}).

Thus, functional semi-gradient flow converges to the true flow map
for MeanFlow without flow-matching pretraining, and for
\gls{imf}, slim \gls{esd}, and slim \gls{lsd} under the stated
diagonal initialization.

\section{Experiments}

\paragraph{Numerical semi-gradient flow experiments.}
The result that $F_\tau$ has a closed-form expression in terms of $F_0$ and $F^*$ is remarkable. We perform numerical experiments to check this by simulating functional semi-gradient flow on $F_\tau$ using equations \ref{eq:eulerian-gradient-flow}, \ref{eq:lagrangian-gradient-flow}, and comparing this simulated $F_\tau$ to its theoretical value. We model $F$ on a discretized grid for $(t,u,x)$ in one to five spatial dimensions. See \S\ref{appsec:convergence-numerical-experiments} for further details.

In Figure \ref{fig:convergence-numerical-main} left, we compare $\|  F_\tau(t,u,x) \|_2$ from simulation to the theoretical value from composing $F^*$ and $F_0$. 
We observe a Pearson correlation of 1.000000 over 220M points, indicating good numerical agreement with the theoretical relationship.
In Figure \ref{fig:convergence-numerical-main} right, we plot the evolution of the pointwise error from the true flow map.
We see excellent agreement between the dashed (theoretical) and solid (simulated) lines, further supporting the validity of the theoretical relationship.

\begin{figure}[!htbp]
     \centering
     \begin{subfigure}[b]{0.24\textwidth}
         \centering
         \includegraphics[width=\textwidth]{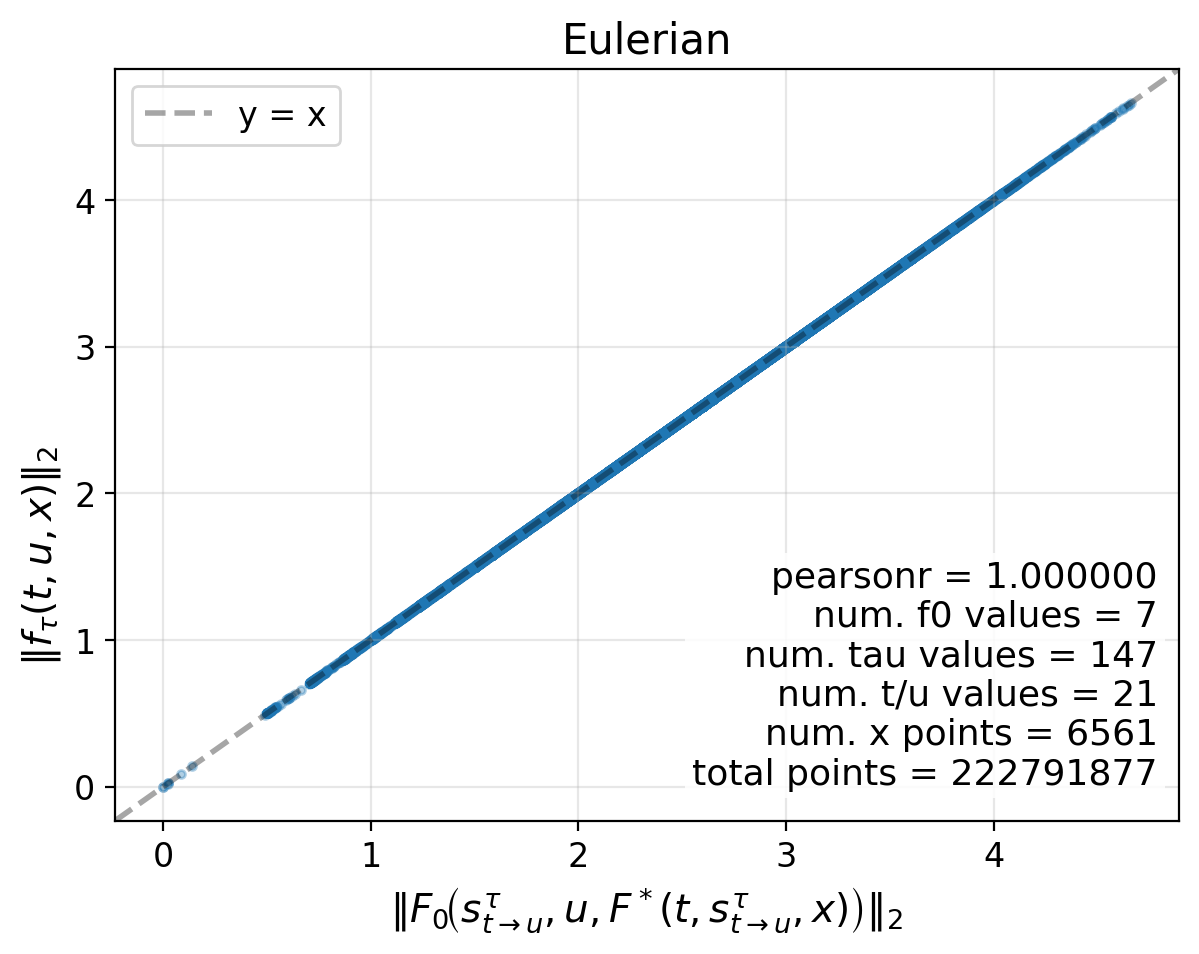}
     \end{subfigure}
     \hfill
     \begin{subfigure}[b]{0.24\textwidth}
         \centering
         \includegraphics[width=\textwidth]{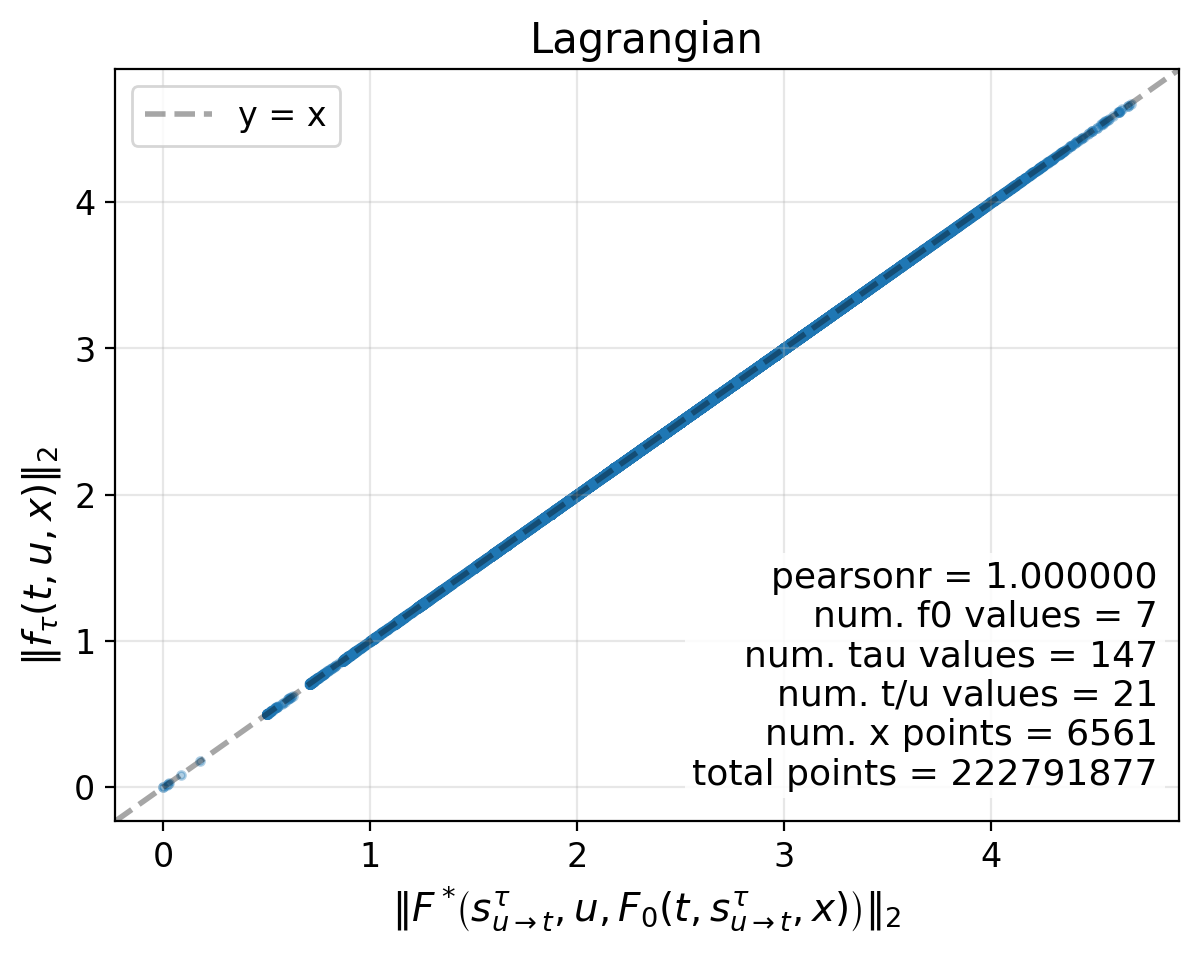}
     \end{subfigure}
     \hfill
     \begin{subfigure}[b]{0.24\textwidth}
         \centering
         \includegraphics[width=\textwidth]{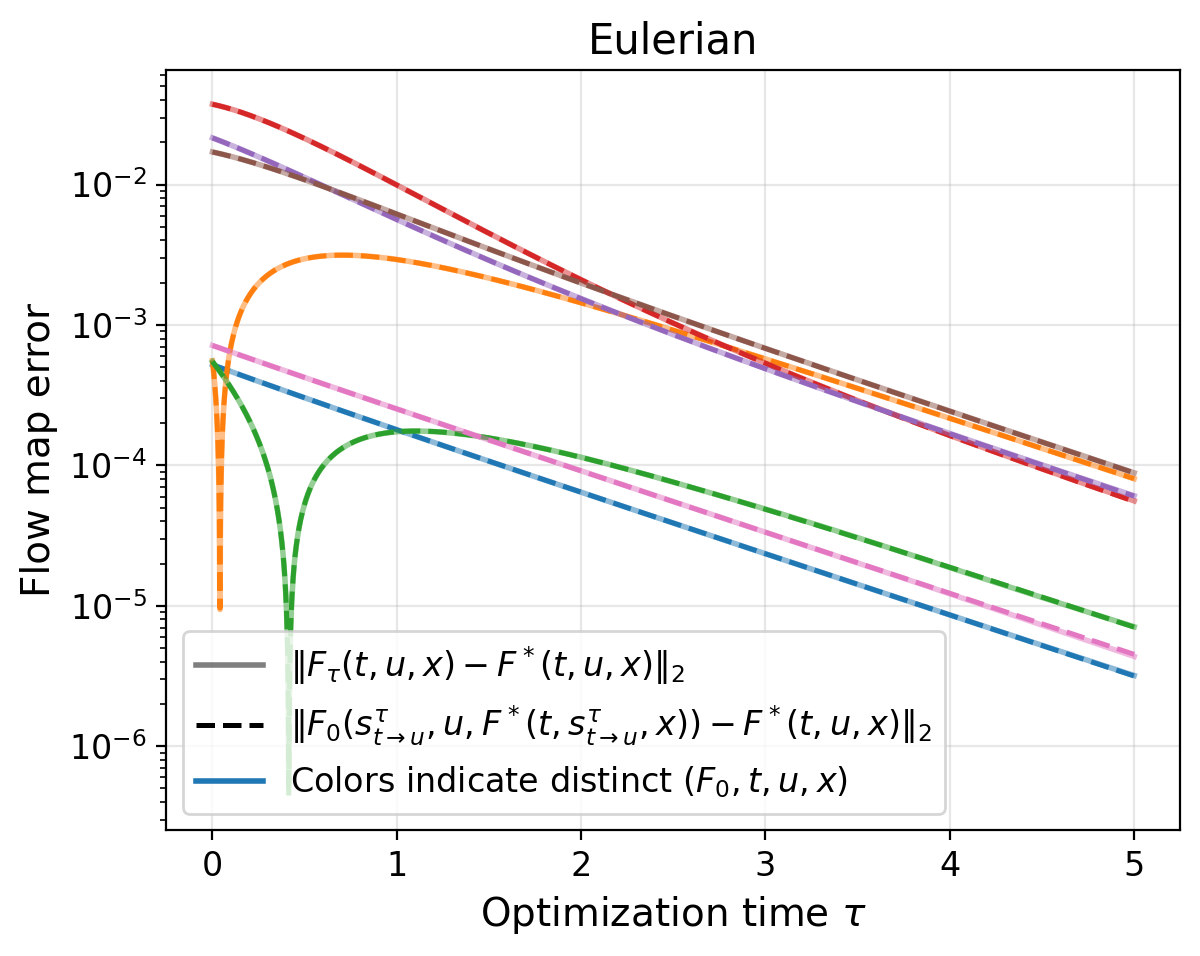}
     \end{subfigure}
     \hfill
     \begin{subfigure}[b]{0.24\textwidth}
         \centering
         \includegraphics[width=\textwidth]{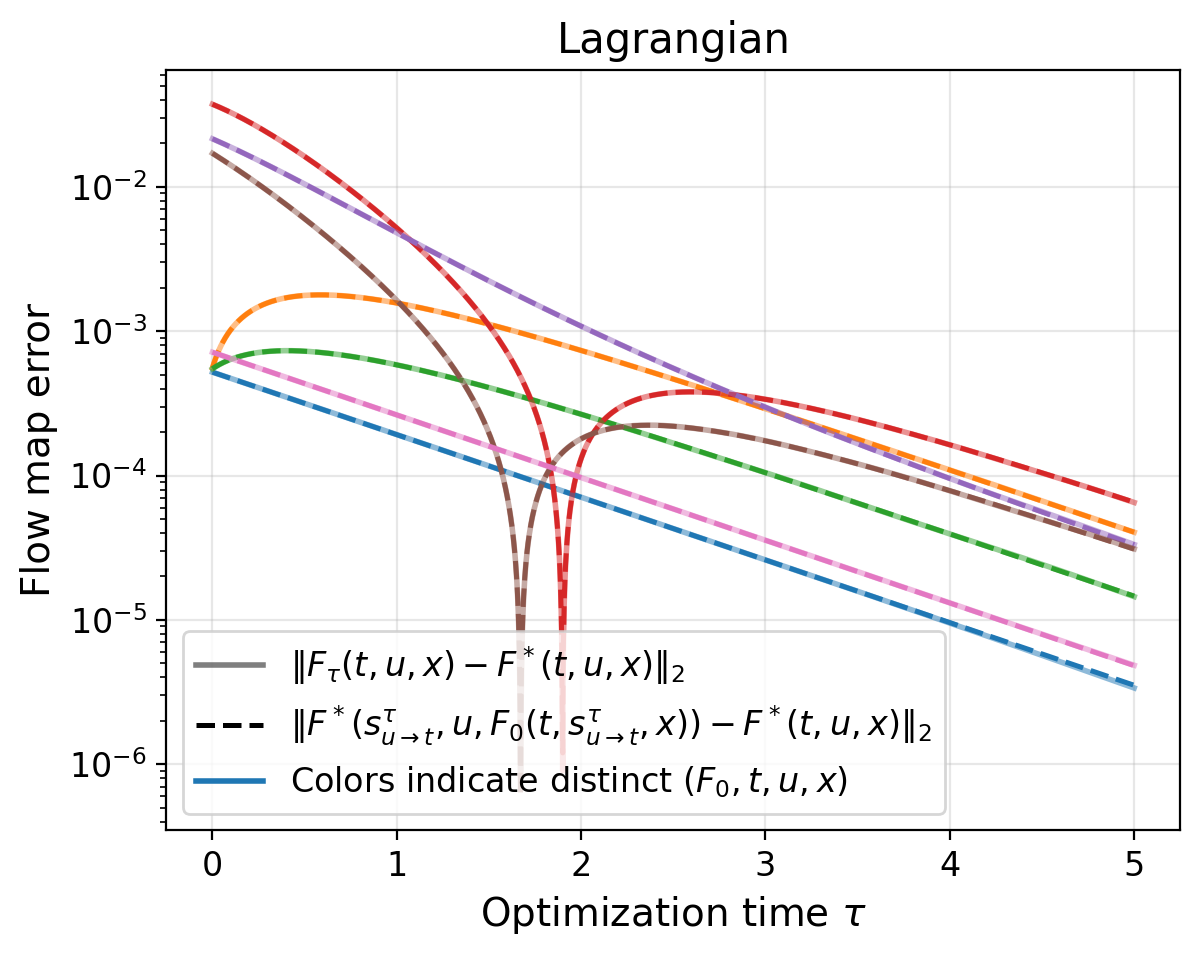}
     \end{subfigure}

     \caption{ \small{ Numerical experiments for the theoretical evolution of $F_\tau$. Left: Norms under simulated semi-gradient flow vs. theoretical value from the composition of $F^*$ and $F_0$, with four spatial dimensions. Right: Theoretical pointwise error to the true flow map (dashed) vs. simulated pointwise error (solid) in one spatial dimension.
     Each color denotes a distinct $(F_0,t,u,x)$ over semi-gradient flow optimization time $\tau$.     
     } 
     }
     \label{fig:convergence-numerical-main}
\end{figure}

\paragraph{Slim-\gls{esd} and \gls{lsd}.}
We train models on \textsc{cifar-10} with slim-\gls{esd}, or slim-\gls{lsd}, and compare memory usage and performance to the original non-stopgrad (\gls{esd}, \gls{lsd}) and stopgrad objectives (\textsc{orig sg}) from \citet{boffi2025build}.
We include more details in \S\ref{appsec:experiment_details}.
The results are shown in Table~\ref{tab:fid_results_best}.  

On \textsc{cifar-10}, \gls{lsd} \textsc{slim} tracks \gls{lsd} \textsc{orig sg} at \gls{fid}-10/50/100.
\gls{esd} \textsc{slim} \emph{improves} on \gls{esd} \textsc{orig sg} at every sampling budget.
The one regime where \textsc{orig sg} retains an edge is single-step sampling in \gls{lsd}, where \textsc{slim} reaches \gls{fid}-1 of 47.6 versus 25.3 for \textsc{orig sg}. 
Our \textsc{slim} variants reduce memory usage by 2$\times$ over \gls{lsd} and \gls{esd}, and 1.5$\times$ over \textsc{orig sg} variants.
We report additional experiments on \textsc{ImageNet}, with similar results to the \textsc{cifar-10} setting, in \S\ref{appsec:imagenet_experiments}.

{\renewcommand\tabularxcolumn[1]{m{#1}}
\begin{table}[!htbp]
  \centering
  \small
  \setlength{\tabcolsep}{4pt}
  \begin{tabularx}{\textwidth}{l|c|c|c|c|c|c|c}
  \toprule
  \multicolumn{1}{c|}{Method}
  & \gls{fid}-1
  & \gls{fid}-10
  & \gls{fid}-50
  & \gls{fid}-100
  & \multicolumn{1}{>{\centering\arraybackslash}X|}{ Memory peak (GiB) ($\downarrow$)}
  & \multicolumn{1}{>{\centering\arraybackslash}X|}{Known stationary point}
  & \multicolumn{1}{>{\centering\arraybackslash}X}{No second-order derivatives} \\
  \midrule
  \gls{esd}
  & $34.3 \pm 0.5$
  & $11.8 \pm 0.3$
  & $6.0 \pm 0.2$
  & $4.9 \pm 0.3$
  & 64.6 & \cmark & \xmark \\
  
  \gls{esd} \textsc{orig sg}
  & $122.0 \pm 24.5$
  & $18.7 \pm 8.6$
  & $6.7 \pm 1.3$
  & $5.8 \pm 1.2$
  & 50.0 & ? & \xmark \\
  
  \gls{esd} \textsc{slim sg}
  & $25.4 \pm 2.8$
  & $4.2 \pm 0.05$
  & $3.1 \pm 0.05$
  & $2.8 \pm 0.03$
  & \textbf{31.2} & \cmark & \cmark \\
  \midrule
  
  \gls{lsd}
  & $291.9 \pm 6.5$ 
  & $10.7 \pm 1.9$
  & $4.2 \pm 0.6$
  & $3.2 \pm 0.4$
  & 60.8 & \cmark & \xmark \\
  
  \gls{lsd} \textsc{orig sg}
  & $25.3 \pm 0.4$
  & $3.3 \pm 0.09$
  & $2.8 \pm 0.02$
  & $2.7 \pm 0.004$
  & 41.0 & ? & \xmark \\
  
  \gls{lsd} \textsc{slim sg}
  & $47.6 \pm 0.5$
  & $3.4 \pm 0.07$
  & $3.0 \pm 0.06$
  & $2.7 \pm 0.05$
  & \textbf{26.9} & \cmark & \cmark \\
  \bottomrule
  \end{tabularx}
  \caption{\small{
    \textbf{\gls{fid} on \textsc{cifar-10}.}
    \gls{fid}-X denotes X sampling steps; mean and std over three random seeds.
    \gls{esd}/\gls{lsd}: Eulerian and Lagrangian self-distillation losses combined with flow matching (no stopgrad).
    \textsc{orig sg}: the stopgrad recommended by \cite{boffi2025build}.
    \textsc{slim}: our proposed stopgrad.
  }}
  \label{tab:fid_results_best}
\end{table}
}

\section{Related work}

\textbf{Flow maps and consistency models.} Flow maps and consistency models train networks to integrate probability-flow ODEs in few steps \citep{song2023consistency, song2023improved, lu2024simplifying, kim2023consistency, salimans2022progressive, boffi2024flowmapmatching, sabour2025align}. Many of these approaches use stopgrads to avoid backpropagating through model derivatives or nested model evaluations.
We discuss the objectives of \cite{boffi2025build} and \cite{geng2025meanflow} in detail.
\citet{goldstein2026flow} proved stationary-point results for SGFlow, a related stopgrad approach, but this approach still differentiates through forward-mode derivatives. In this work, \textbf{we prove uniqueness of the true flow map as a stationary point} for MeanFlow, improved MeanFlow, and our slim \gls{esd}/\gls{lsd} variants, and additionally \textbf{prove convergence of functional semi-gradient flow} for these methods.

\textbf{Stopgrads and semi-gradients.}
Semi-gradients have been studied extensively in \gls{rl} \citep{sutton2018reinforcement, baird1995residual, tsitsiklis1997analysis, bhandari2018finite}, where TD-style updates are not gradients of any fixed loss but nonetheless converge under conditions related to the structure of the Bellman operator and its fixed points \citep{ollivier2018approximate, fellows2023why, brandfonbrener2020geometric, piche2021bridging}.
Non-conservative update fields also arise in differentiable games such as \textsc{gan}s \citep{balduzzi2018mechanics, letcher2019differentiable}.
We derive our functional stopgrad formalism from first principles---perturbing only non-stopgrad occurrences of $f$---and verify that it reduces to PyTorch's \texttt{detach} in finite dimensions. The same construction is used in the context of fine-tuning and sampling in the proofs of adjoint matching and adjoint sampling \citep{domingo-enrich2025adjointmatching, havens2025adjointsampling}.
\citet{ponce2026dual} studies \textit{two-player} settings used in self-supervised methods such as \textsc{byol}; there,  stopgrads prevent model collapse for objectives whose minima would otherwise be collapsed solutions.

\section{Discussion}

We formalized the stopgrad operator in the calculus of variations, and proved a general theorem (Theorem~\ref{thm:principle}) characterizing stationary points of stopgrad functionals.
We analyzed stationary points and convergence for flow map objectives with stopgrad, and proposed slim stopgrad placements for \gls{lsd} and \gls{esd} that avoid differentiating through derivatives of $f$, reducing training memory by $2\times$.
Looking forward, physics-informed neural networks (PINN) losses regress against \gls{pde} residuals and share structural similarities with flow map objectives \citep{hao2024pinnaclepinns}. Our stopgrad regression principle may be used to propose stopgrad placements for \textsc{pinn} losses.

\paragraph{Limitations.}
Our analysis is variational (infinite capacity, continuous optimization), so finite-width networks and stochastic gradient descent may behave differently. Our convergence results for \gls{imf}, slim \gls{esd} and slim \gls{lsd} assume perfect flow-matching pretraining.

\clearpage

\subsubsection*{Code Availability}
We release our code at \url{https://github.com/maxwshen/stopgrad}.

\subsubsection*{Acknowledgments}

The authors would like to thank Amirmojtaba Sabour and Eric Vanden-Eijnden for extended discussion. The authors also thank the authors of \citep{geng2025meanflow} and \citep{boffi2024flowmapmatching, boffi2025build} for open-sourcing their code.
Mark Goldstein thanks the Simons Foundation Flatiron Institute for funding and computational resources.
Max Shen thanks Saeed Saremi for discussions.
Zichu Wang is supported by the MIT Bose Presidential Fellowship.
This work was partly supported by the NIH/NHLBI Award R01HL148248, NSF Award 1922658 NRT-HDR: FUTURE Foundations, Translation, and Responsibility for Data Science, NSF CAREER Award 2145542, ONR N00014-23-1-2634, NIH R01CA296388, NSF 2404476, Optum, the Gates Foundation, and Apple. This work was also supported by the Institute of Information \& Communications Technology Planning \& Evaluation (IITP) with a grant funded by the Ministry of Science and ICT (MSIT) of the Republic of Korea in connection with the Global AI Frontier Lab International Collaborative Research Agreement. (No. RS-2024-00469482 \& RS-2024-00509279).
The authors do not have competing interests.

\subsubsection*{LLM Usage Declaration}
We use LLMs for assistance with mathematical proofs and writing code for experiments. The authors take full responsibility for all claims and results in the manuscript.

\bibliography{references}

\clearpage
\appendix

\section*{Glossary of symbols}
\label{app:glossary}
\addcontentsline{toc}{section}{Glossary of symbols}

\begingroup
\small
\renewcommand{\arraystretch}{1.15}
\noindent\begin{tabularx}{\textwidth}{@{}>{$}l<{$} X @{}}
\toprule
\multicolumn{2}{@{}l}{\textbf{General setting (Theorem~\ref{thm:principle}, Proposition~\ref{prop:floor})}} \\
\midrule
\mbx,\ \mby & inputs and labels. $\mu$ = distribution of $\mbx$. \\
S & open subset of $\bbR^n$ with $\mu(U)>0$ for every nonempty open $U \subseteq S$. \\
\bar S & closure of $S$. $\operatorname{supp}\mu \subseteq \bar S$. \\
H & $L^2(\mu,\bbR^d)$. $\langle \phi,\psi\rangle_\mu := \E_\mu[\phi(\mbx)\cdot\psi(\mbx)]$, $\|\phi\|_\mu^2 := \langle\phi,\phi\rangle_\mu$. \\
\mathcal{M} & model class. $\mathcal{M} \subset C^1 \cap H$. \\
\mathcal{D} & perturbation class. $C^1_c(S) \subseteq \mathcal{D} \subset \mathcal{M}$; for flow maps (\S\ref{app:corollaries}) $C^1_c(\Omega) \subseteq \mathcal{D}$. \\
C^1_c(S) & $C^1$ functions with compact support in $S$. \\
T[f,\mby] & $P[f] + Q[f]\,\mby + R\,\mby$ \eqref{eq:affine-target}. \\
g[f] & $f - T[f,\E[\mby\mid\mbx]] \in H$. $\delta\cL[f;h] = \langle g[f],h\rangle_\mu$. \\
\xi & $\mby - \E[\mby\mid\mbx]$. \\
V[f] & $\tfrac12\E\|(Q[f]+R)\,\xi\|^2$ \eqref{eq:floor-decomp}. \\
\text{additive noise} & Definition~\ref{def:additive}: $Q[f] \equiv Q$ independent of $f$. \\
\midrule
\multicolumn{2}{@{}l}{\textbf{Flow maps (\S\ref{sec:flowmap_stationary}, \S\ref{app:corollaries})}} \\
\midrule
\mbx_t,\ \dot\mbx_t & interpolant and its time derivative. \\
v & $v(t,x) = \E[\dot\mbx_t \mid \mbx_t = x]$. \\
F,\ f & $F(t,u,x) = x + (u-t)f(t,u,x)$ \eqref{eq:f_param}. $F^*, f^*$ = true flow map. \\
\Omega^\circ & $\{0<t<u<1\}\times\bbR^d$. Open. $\mu$ has full support on $\Omega^\circ$. In Theorem~\ref{thm:principle}, $S = \Omega^\circ$. \\
\Omega & $\{0<t\le u<1\}\times\bbR^d$. Relatively open in $\{t\le u\}$. Contains the diagonal $\{t=u\}$. \\
\bar\Omega & $\{0\le t\le u\le 1\}\times\bbR^d$. Closure of $\Omega^\circ$ and of $\Omega$. $\operatorname{supp}\mu \subseteq \bar\Omega = \bar S$. \\
\mu & distribution of $(t,u,\mbx_t)$. May have an atom on $\{t=u\}$. \\
\Omega_\Delta & $(0,1)\times\bbR^d$. \\
\nu & distribution of $(t,\mbx_t)$. Full support on $\Omega_\Delta$. \\
H_\Delta & $L^2(\nu)$, with inner product $\langle\cdot,\cdot\rangle_\nu$. \\
h_\Delta & $h_\Delta(t,x) := h(t,t,x)$ for $h \in \mathcal{D}$. \\
g_{\mathrm{SD}}[f] & semi-gradient of a stopgrad (self-distillation) term. \\
g_{\mathrm{FM}}[f] & $g_{\mathrm{FM}}[f](t,x) = f(t,t,x) - v(t,x)$. \\
E_v[f] & $\partial_t F + (\partial_x F)\,v$ (\S\ref{appsec:form_of_functional_gradient_flows}). \\
L_v[f] & $\partial_u F - v(u,F)$ (\S\ref{appsec:form_of_functional_gradient_flows}). \\
C^\infty_{\mathrm{pol}} & smooth functions all of whose derivatives grow at most polynomially in $x$, uniformly over $(t,u) \in \bar\Omega$ (\S\ref{appsec:convergence}). \\
\mathcal{A} & set of $(t,u)$ in Proposition~\ref{prop:rectifying-templates-unified}. \\
s^\tau_{t\shortrightarrow u} & $u - e^{-\tau}(u-t)$. \\
s^\tau_{u\shortrightarrow t} & $t + e^{-\tau}(u-t)$. \\
\bottomrule
\end{tabularx}
\endgroup

\section{Proof of Theorem~\ref{thm:principle} (Stopgrad Regression Principle)}
\label{appsec:proofs_mf}

We first state the fundamental lemma of the calculus of variations, which the proof invokes to pass from a vanishing inner product (against all test functions) to pointwise vanishing.

\begin{tcolorbox}[boxrule=0pt, frame empty]
\begin{lemma}[Fundamental lemma of the calculus of variations]
\label{lem:fundamental}
Let $X$ be a random variable on $\bbR^n$ with distribution $\mu$
having full support on an open set $S \subseteq \bbR^n$, i.e.\ $\mu(U) > 0$ for every nonempty open $U \subseteq S$.
Let $g : S \to \bbR^d$ be continuous.
If
\begin{align}
  \E\big[g(X) \cdot h(X)\big] = 0
  \quad \text{for all } h \in C^1_c(S, \bbR^d),
\end{align}
then $g = 0$ on $S$.
\end{lemma}
\end{tcolorbox}

\begin{proof}
Suppose $g(x_0) \neq 0$ for some $x_0 \in S$.
By continuity, there exist a neighborhood $U \ni x_0$ and a unit vector $e$ such that $g(x) \cdot e > c > 0$ for all $x \in U$.
Choose $h(x) = \varphi(x)\, e$ where $\varphi \in C^\infty_c(U)$ is a nonnegative bump function with $\varphi(x_0) > 0$.
Then $h \in C^1_c$ and
\begin{align}
  \E[g(X) \cdot h(X)]
  = \E[\varphi(X)\, g(X) \cdot e]
  \geq c\, \E[\varphi(X)\, \mathbf{1}[X \in U]] > 0,
\end{align}
since $\{\varphi > 0\}$ is a nonempty open subset of $S$ and therefore has positive $\mu$-mass by the full-support assumption.
This contradicts $\E[g \cdot h] = 0$.
\end{proof}

\begin{proof}[Proof of Theorem~\ref{thm:principle}]
    Let \(f^\epsilon = f + \epsilon h\) for admissible \(h \in \mathcal{D}\).
    By the stopgrad definition (\S\ref{sec:sg_def}),
    \(T[f,\mby]\) is not perturbed. We compute the first variation as the limit of $(\cL[f^\epsilon] - \cL[f]) / \epsilon$ as $\epsilon \to 0$. For the difference, we have:
    \begin{align*}
      \cL[f^\epsilon] - \cL[f]
      &=
      \frac12\,\E_{\mbx,\mby}\Big[
        \big\|
          f(\mbx) + \epsilon h(\mbx) - T[f,\mby](\mbx)
        \big\|^2
        -
        \big\|
          f(\mbx) - T[f,\mby](\mbx)
        \big\|^2
      \Big].
      \\ &=
      \epsilon\,\E_{\mbx,\mby}\Big[
        \big(
          f(\mbx) - T[f,\mby](\mbx)
        \big)\cdot h(\mbx)
      \Big]
      +
      \frac{\epsilon^2}{2}\,\E_{\mbx}\big[\|h(\mbx)\|^2\big].
    \end{align*}
    By the square-integrability assumption, both expectations are finite:
    the first by Cauchy--Schwarz,
    \begin{align*}
      \E_{\mbx,\mby}\Big[
        \big|
          \big(
            f(\mbx) - T[f,\mby](\mbx)
          \big)\cdot h(\mbx)
        \big|
      \Big]
      \le
      \Big(
        \E_{\mbx,\mby}\big[
          \|f(\mbx) - T[f,\mby](\mbx)\|^2
        \big]
      \Big)^{1/2}
      \Big(
        \E_{\mbx}\big[\|h(\mbx)\|^2\big]
      \Big)^{1/2}
      < \infty,
    \end{align*}
    and the second since $h \in \mathcal{D} \subset H = L^2(\mu)$.
    Therefore we may divide by \(\epsilon\) and let \(\epsilon \to 0\), obtaining the first variation
    \begin{align}
      \delta \cL[f; h]
      &=
      \E_{\mbx,\mby}\Big[
        \big(
          f(\mbx) - T[f,\mby](\mbx)
        \big)\cdot h(\mbx)
      \Big].
    \end{align}
    We now apply the tower property
    \(\E_{\mbx,\mby}[\cdot] = \E_\mbx\big[\E_{\mby \mid \mbx}[\cdot]\big]\).
    Since \(f(\mbx)\) and \(h(\mbx)\) are deterministic given \(\mbx\),
    \begin{align*}
      \delta \cL[f; h]
      &=
      \E_\mbx\Big[
        h(\mbx)\cdot
        \E_{\mby \mid \mbx}\big[
          f(\mbx) - T[f,\mby](\mbx)
        \big]
      \Big] =
      \E_\mbx\Big[
        h(\mbx)\cdot
        \big(
          f(\mbx) - \E_{\mby \mid \mbx}[T[f,\mby](\mbx)]
        \big)
      \Big].
    \end{align*}
    Since \(T\) is affine in \(\mby\),
    \begin{align*}
      \E_{\mby \mid \mbx}[T[f,\mby](\mbx)]
      &=
      P[f](\mbx)
      + Q[f](\mbx)\,\E[\mby \mid \mbx]
      + R(\mbx)\,\E[\mby \mid \mbx] 
      =
      T[f,\, \E[\mby \mid \mbx]](\mbx).
    \end{align*}
    Hence
    \begin{align}
      \delta \cL[f; h]
      &=
      \E_\mbx\Big[
        h(\mbx)\cdot
        \big(
          f(\mbx) - T[f,\, \E[\mby \mid \mbx]](\mbx)
        \big)
      \Big].
    \end{align}
    Define $g(\mbx) := f(\mbx) - T[f,\, \E[\mby \mid \mbx]](\mbx) = \E\big[f(\mbx) - T[f,\mby](\mbx) \,\big|\, \mbx\big]$.
    Since conditional expectation is a contraction on $L^2$,
    $\|g\|_\mu \le \big(\E_{\mbx,\mby}\|f(\mbx) - T[f,\mby](\mbx)\|^2\big)^{1/2} = (2\cL[f])^{1/2} < \infty$,
    so $g \in H$ and the first variation is represented in $H$ as $\delta \cL[f; h] = \langle g, h\rangle_\mu$ for all $h \in \mathcal{D}$.
    Stationarity requires \(\langle g, h\rangle_\mu = 0\) for all \(h \in \mathcal{D}\), in particular for all $h \in C^1_c(S) \subseteq \mathcal{D}$.
    By the continuity hypothesis of the theorem, \(g\) is continuous on $\bar S$, and since \(\mu\) has full support on the open set $S$, the fundamental lemma of the calculus of variations (Lemma~\ref{lem:fundamental}) implies \(g = 0\) on $S$; by continuity, $g = 0$ on $\bar S \supseteq \operatorname{supp}\mu$, i.e., $f^*(\mbx) = T[f^*,\, \E[\mby \mid \mbx]](\mbx)$ for all $\mbx \in \operatorname{supp}\mu$.
    This gives one direction of the if and only if.
    The other direction is immediate: if \eqref{eq:principle_condition} holds, then \(g = 0\) $\mu$-a.e., so
    \(\delta \cL[f; h] = \langle g, h\rangle_\mu = 0\) for all \(h \in \mathcal{D}\).
\end{proof}

We now state the loss decomposition behind Corollary~\ref{cor:uniqueQ}.
\begin{tcolorbox}[boxrule=0pt, frame empty]
\begin{proposition}[Conditional-variance floor]
\label{prop:floor}
Assume the hypotheses of Theorem~\ref{thm:principle} hold for every $f \in \mathcal{M}$, let $\xi := \mby - \E[\mby \mid \mbx]$, so that $\E[\xi \mid \mbx] = 0$, and let $g[f](\mbx) := f(\mbx) - T[f,\,\E[\mby \mid \mbx]](\mbx)$.
Then $g[f] \in H$ and
\begin{align}
    \label{eq:floor-decomp}
    \cL[f] = \tfrac12\,\|g[f]\|_\mu^2 + V[f],
    \qquad
    V[f] := \tfrac12\,\E_{\mbx,\mby}\big\|\big(Q[f](\mbx) + R(\mbx)\big)\,\xi\big\|^2,
\end{align}
where $\|\phi\|_\mu^2 := \E_\mu\|\phi(\mbx)\|^2$.
Consequently:
\begin{enumerate}[leftmargin=*]
    \item[\textbf{(i)}] every stationary point $f^* \in \mathcal{M}$ satisfies $\cL[f^*] = V[f^*]$;
    \item[\textbf{(ii)}] if $Q[f] \equiv Q$ does not depend on $f$ (additive label noise, Definition~\ref{def:additive}), then $V[f] \equiv V$ is a constant and the stationary set is exactly $\{f \in \mathcal{M} :\ \cL[f] = V\}$. In particular, if $\cL$ attains the value $V$ at exactly one $f^* \in \mathcal{M}$, then $f^*$ is the unique stationary point in $\mathcal{M}$, and if $\inf_{f \in \mathcal{M}} \cL[f] > V$ then $\cL$ has no stationary point in $\mathcal{M}$.
\end{enumerate}
\end{proposition}
\end{tcolorbox}

\begin{proof}[Proof of Proposition~\ref{prop:floor}]
With $\xi = \mby - \E[\mby \mid \mbx]$ and the affine form \eqref{eq:affine-target},
\begin{align*}
    f(\mbx) - T[f,\mby](\mbx)
    &= \big(f(\mbx) - T[f,\E[\mby \mid \mbx]](\mbx)\big) - \big(Q[f](\mbx) + R(\mbx)\big)\,\xi \\
    &= g[f](\mbx) - \big(Q[f](\mbx) + R(\mbx)\big)\,\xi .
\end{align*}
Expanding the square in \eqref{eq:principle_loss},
\begin{align*}
    \cL[f]
    = \tfrac12\,\E_\mbx\|g[f](\mbx)\|^2
    + \tfrac12\,\E_{\mbx,\mby}\big\|(Q[f](\mbx) + R(\mbx))\,\xi\big\|^2
    - \E_{\mbx,\mby}\Big[g[f](\mbx) \cdot \big(Q[f](\mbx) + R(\mbx)\big)\,\xi\Big],
\end{align*}
where all three expectations are finite: $g[f] \in H$ by the contraction argument in the proof of Theorem~\ref{thm:principle}, $(Q[f] + R)\xi = g[f] - (f - T[f,\mby])$ is a difference of two square-integrable quantities, and the cross term is then finite by Cauchy--Schwarz.
For the cross term, condition on $\mbx$: $g[f](\mbx)$, $Q[f](\mbx)$ and $R(\mbx)$ are deterministic given $\mbx$, and $\E[\xi \mid \mbx] = 0$, so by the tower property the cross term equals $\E_\mbx\big[g[f](\mbx) \cdot (Q[f](\mbx) + R(\mbx))\,\E[\xi \mid \mbx]\big] = 0$.
This proves \eqref{eq:floor-decomp}.

(i) By Theorem~\ref{thm:principle}, $f^* \in \mathcal{M}$ is stationary iff $g[f^*] = 0$ on $\operatorname{supp}\mu$, which gives $\|g[f^*]\|_\mu = 0$ and hence $\cL[f^*] = V[f^*]$.
Conversely, if $\cL[f] = V[f]$ then $\|g[f]\|_\mu = 0$; as $g[f]$ is continuous on $\bar S$ and $\mu$ has full support on the open set $S$, $g[f] = 0$ on $S$ and hence on $\bar S \supseteq \operatorname{supp}\mu$, so $f$ is stationary.
Thus the stationary set is $\{f \in \mathcal{M} : \cL[f] = V[f]\}$.

(ii) Under additive label noise, $V[f] = \tfrac12\,\E\|(Q(\mbx) + R(\mbx))\xi\|^2 =: V$ does not depend on $f$, so by (i) the stationary set is $\{f \in \mathcal{M} : \cL[f] = V\}$.
Since $\cL[f] \ge V$ for all $f \in \mathcal{M}$, this is the set of minimizers of $\cL$ over $\mathcal{M}$ that attain the value $V$.
If exactly one $f^* \in \mathcal{M}$ attains it, $f^*$ is the unique stationary point in $\mathcal{M}$; if no $f \in \mathcal{M}$ attains it (i.e.\ $\inf_{\mathcal{M}} \cL > V$), there is no stationary point in $\mathcal{M}$.
\end{proof}

\begin{proof}[Proof of Corollary~\ref{cor:uniqueQ}]
Under additive label noise, $V[f] \equiv V$ with $V$ as in the corollary, so \eqref{eq:floor-decomp} gives $\cL[f] = \tfrac12\|g[f]\|_\mu^2 + V \ge V$ for all $f \in \mathcal{M}$.
The remaining claims are Proposition~\ref{prop:floor}(ii).
\end{proof}

\section{Standing Regularity Conditions for Flow Maps}

\begin{remark}[Standing regularity conditions]
\label{rmk:regularity}
All results on flow maps assume the following conditions:
\begin{enumerate}
  \item \textbf{Model and perturbation classes.}
    $f \in \mathcal{M}$: $f$ is $C^1$ in all arguments $(t,u,x)$,
    with $f$ and its first partial derivatives
    $\partial_t f$, $\partial_u f$, $\partial_x f$ continuous up to the boundary of the closed time triangle $\{0 \le t \le u \le 1\} \times \bbR^d$, and $f \in H = L^2(\mu)$.
    Perturbations range over a class $\mathcal{D} \subset \mathcal{M}$ that contains $C^1_c(\Omega)$, the $C^1$ functions with compact support in $\Omega$, and whose diagonal restrictions $h(t,t,\cdot)$ lie in $L^2(\nu)$.
    Because $\Omega$ contains the diagonal, such supports may meet $\{t = u\}$; this matters because the flow matching term depends on $f$ only through $f(t,t,\cdot)$, so test functions supported in the open triangle $\{t < u\}$ alone cannot detect it (Lemma~\ref{lem:split}, Step 3).
  \item \textbf{Velocity regularity.}
    The marginal velocity $v(t,x) = \E[\dot \mbx_t \mid \mbx_t = x]$
    is $C^1$ in $(t,x)$ and Lipschitz in $x$ uniformly in $t$.
  \item \textbf{Integrability.}
    $\E[\|\dot \mbx_t\|^2] < \infty$,
    and all expectations appearing in the functionals are finite.
  \item \textbf{Sampling distributions.}
    Times satisfy $t \le u$.
    The training distribution $\mu$ of $(t,u,\mbx_t)$ has full support on the open set
    $\Omega^\circ := \{(t,u,x) : 0 < t < u < 1\} \times \bbR^d$,
    has finite second moments, and is supported in the closed time triangle,
    $\operatorname{supp}\mu \subseteq \bar\Omega := \{(t,u,x) : 0 \le t \le u \le 1\} \times \bbR^d$;
    it may place mass on the diagonal $\{t = u\}$.
    We write $\Omega := \{(t,u,x) : 0 < t \le u < 1\} \times \bbR^d$, which is relatively open in the half-space $\{t \le u\}$ and contains the diagonal.
    The flow-matching distribution $\nu$ of $(t,\mbx_t)$ has full support on
    $\Omega_\Delta := (0,1) \times \bbR^d$.
  \item \textbf{Admissibility of the target.}
    The true flow map lies in the model class: $f^* \in \mathcal{M}$, where $F^* = x + (u-t)f^*$.
\end{enumerate}
These are standard conditions for flow matching \cite{lipman2024flowmatchingguidecode}.
(1) holds for smooth neural architectures on compact domains under bounded weights;
(2) holds for standard interpolation processes
(e.g., $\mbx_t = (1-t)\mbx_0 + t\mbx_1$ with sub-Gaussian marginals);
(3) is needed for the squared functionals, such as the loss, to be finite.
(4) holds for the usual time samplers (e.g.\ $t,u$ uniform or logit-normal on the triangle, with an atom at $t=u$) because $\mbx_t$ has a positive density whenever $q_0$ does.
(5) holds for $\mathcal{M} = C^1 \cap L^2(\mu)$ under (2) and (4): $F^*$ is $C^1$ on the closed triangle, so $f^* = \int_0^1 v\big(t + r(u-t),\, F^*(t,\, t + r(u-t),\, x)\big)\,dr$ is $C^1$ up to the boundary, including the diagonal; and by Gr\"onwall (as in Lemma~\ref{lem:flowmap-lipschitz}, using $\|v(s,y)\| \le \|v(s,0)\| + L\|y\|$) there are constants $a,b$ depending only on $L$ and $\sup_s \|v(s,0)\|$ with $\|F^*(t,u,x) - x\| \le (u-t)\,(a\|x\| + b)$, so $f^* = (F^* - x)/(u-t)$ has at most linear growth in $x$ uniformly in $(t,u)$ and lies in $L^2(\mu)$ whenever $\mu$ has finite second moments.
If $\mathcal{M}$ is taken to be $C^\infty_{\mathrm{pol}} \cap L^2(\mu)$ (\S\ref{appsec:convergence}), (5) holds when $v \in C^\infty_{\mathrm{pol}}$, by smooth dependence of ODE solutions on initial data together with Gr\"onwall bounds on all derivatives.

These regularity conditions imply that the model and targets are continuous, and losses are finite, satisfying the conditions needed for Theorem~\ref{thm:principle}, and that the full-support hypothesis of Theorem~\ref{thm:principle} holds with $S = \Omega^\circ$ and $\operatorname{supp}\mu \subseteq \bar\Omega = \bar S$.

\end{remark}

\section{Uniqueness of stationary points for flow map objectives}
\label{app:corollaries}

Throughout this section we assume the standing conditions of Remark~\ref{rmk:regularity}, so times satisfy $t \le u$, $\mu$ has full support on $\Omega^\circ = \{0 < t < u < 1\} \times \bbR^d$ with $\operatorname{supp}\mu \subseteq \bar\Omega = \{0 \le t \le u \le 1\} \times \bbR^d$, and $\nu$ has full support on $\Omega_\Delta = (0,1) \times \bbR^d$; as in Remark~\ref{rmk:regularity}, $\Omega = \{0 < t \le u < 1\} \times \bbR^d$ denotes the relatively open part of the triangle that contains the diagonal.
We write $F(t,u,x) = x + (u-t) f(t,u,x)$, and use the derivative identities
\begin{align}
    \label{eq:F-derivs}
    \partial_t F = -f + (u-t)\,\partial_t f,
    \qquad
    \partial_u F = f + (u-t)\,\partial_u f,
    \qquad
    \partial_x F = I + (u-t)\,\partial_x f .
\end{align}
Because $\mathcal{M} \subset C^1$ with derivatives continuous up to the boundary of the closed triangle $\bar\Omega$ (Remark~\ref{rmk:regularity}(1)) and $v$ is $C^1$, every residual below is continuous on $\bar\Omega$; hence a residual that vanishes on $\Omega^\circ$, which is dense in $\bar\Omega$, vanishes on all of $\bar\Omega$.
We use this extension without further comment.
All stationarity statements are with respect to perturbations $h \in \mathcal{D} \supseteq C^1_c(\Omega)$, and ``unique'' means unique in $\mathcal{M}$.

\begin{tcolorbox}[boxrule=0pt, frame empty]
\begin{lemma}[Uniqueness for the transport equation]
    \label{lem:char}
    Suppose $v$ is continuous and Lipschitz in $x$ uniformly in $t \in [0,1]$, so that for every $(t,x)$ the \gls{ode} $\dot X_s = v(s,X_s)$, $X_t = x$, has a unique $C^1$ solution on $[0,1]$, and $F^*(t,s,x) := X_s$ is the true flow map.
    Let $F \in C^1(\bar\Omega;\bbR^d)$ satisfy
    \begin{align*}
        \partial_t F(t,u,x) + \partial_x F(t,u,x)\,v(t,x) = 0
        \quad\text{and}\quad
        F(u,u,x) = x
        \qquad \text{for all } (t,u,x) \in \bar\Omega .
    \end{align*}
    Then $F = F^*$ on $\bar\Omega$.
\end{lemma}
\end{tcolorbox}
\begin{proof}
Fix $(t,u,x) \in \bar\Omega$ and let $X_s$ be the solution of the \gls{ode} with $X_t = x$.
Define $\varphi(s) := F(s,u,X_s)$ for $s \in [t,u]$; note $(s,u,X_s) \in \bar\Omega$ for such $s$.
By the chain rule and the \gls{ode},
\begin{align*}
    \varphi'(s)
    = \partial_1 F(s,u,X_s) + \partial_3 F(s,u,X_s)\,\dot X_s
    = \partial_1 F(s,u,X_s) + \partial_3 F(s,u,X_s)\,v(s,X_s)
    = 0,
\end{align*}
where $\partial_1, \partial_3$ denote the partial derivatives with respect to the first and third arguments.
Hence $\varphi$ is constant on $[t,u]$, and
\begin{align*}
    F(t,u,x) = \varphi(t) = \varphi(u) = F(u,u,X_u) = X_u = F^*(t,u,x).
\end{align*}
\end{proof}

The flow matching term is $\cL_{\mathrm{FM}}[f] := \tfrac12\,\E_{(t,\mbx_t) \sim \nu}\|f(t,t,\mbx_t) - \dot \mbx_t\|^2$, paired in $H_\Delta := L^2(\nu)$ with inner product $\langle \cdot,\cdot \rangle_\nu$; the diagonal restriction $h \mapsto h(t,t,\cdot)$ is applied pointwise to $h \in \mathcal{D} \subset C^1$, not as an operator on $H$.
Stationarity of a sum does not follow from stationarity of each term, since the two first variations could be nonzero and cancel.
The following lemma rules this out.
\begin{tcolorbox}[boxrule=0pt, frame empty]
\begin{lemma}[Splitting / non-cancellation]
    \label{lem:split}
    Let $\cL_{\mathrm{SD}}[f] := \tfrac12\,\E_{(t,u,\mbx_t) \sim \mu}\|f(t,u,\mbx_t) - \sgb{T[f](t,u,\mbx_t)}\|^2$ be of the form \eqref{eq:principle_loss} with $\mu$ of full support on $\Omega^\circ$ and $\operatorname{supp}\mu \subseteq \bar\Omega$, and let $g_{\mathrm{SD}}[f](t,u,x) := f(t,u,x) - T[f](t,u,x)$ be its semi-gradient, assumed continuous on $\bar\Omega$.
    Let $\cL_{\mathrm{FM}}$ be as above, and assume $f \in \mathcal{M}$ with $\cL_{\mathrm{SD}}[f] + \cL_{\mathrm{FM}}[f] < \infty$.
    Then $f$ is a stationary point of $\cL_{\mathrm{SD}} + \cL_{\mathrm{FM}}$ (with perturbations $h \in \mathcal{D} \supseteq C^1_c(\Omega)$) if and only if
    \begin{align*}
        g_{\mathrm{SD}}[f] = 0 \ \text{on } \bar\Omega
        \qquad \text{and} \qquad
        f(t,t,x) = v(t,x) \ \text{on } \Omega_\Delta .
    \end{align*}
\end{lemma}
\end{tcolorbox}

\begin{proof}[Proof of Lemma~\ref{lem:split}]
Let $h \in \mathcal{D}$ and write $h_\Delta(t,x) := h(t,t,x)$ for its diagonal restriction, a pointwise operation on the $C^1$ function $h$; if $h \in C^1_c(\Omega)$ then $h_\Delta \in C^1_c(\Omega_\Delta)$.
The \textsc{fm} term depends on $f$ only through $f(t,t,\cdot)$, so by the tower property (replacing $\dot \mbx_t$ by $v(t,\mbx_t) = \E[\dot \mbx_t \mid \mbx_t]$ exactly as in the proof of Theorem~\ref{thm:principle}) its first variation is $\langle g_{\mathrm{FM}}[f], h_\Delta\rangle_\nu$ with $g_{\mathrm{FM}}[f](t,x) := f(t,t,x) - v(t,x)$; here $g_{\mathrm{FM}}[f] \in H_\Delta = L^2(\nu)$ by the same contraction argument as in the proof of Theorem~\ref{thm:principle}, since $\cL_{\mathrm{FM}}[f] < \infty$, and $h_\Delta \in H_\Delta$ since $\cL_{\mathrm{FM}}[f + \epsilon h] < \infty$ for small $\epsilon$.
Together with Theorem~\ref{thm:principle} for the \textsc{sd} term (whose first variation is represented in $H = L^2(\mu)$),
\begin{align}
    \label{eq:split-variation}
    \delta(\cL_{\mathrm{SD}} + \cL_{\mathrm{FM}})[f;h]
    = \langle g_{\mathrm{SD}}[f], h\rangle_\mu + \langle g_{\mathrm{FM}}[f], h_\Delta\rangle_\nu .
\end{align}
($\Leftarrow$) If $g_{\mathrm{SD}}[f] = 0$ on $\bar\Omega$ and $g_{\mathrm{FM}}[f] = 0$ on $\Omega_\Delta$, then by continuity $g_{\mathrm{FM}}[f] = 0$ on $\bar\Omega_\Delta = [0,1] \times \bbR^d \supseteq \operatorname{supp}\nu$, and $\bar\Omega \supseteq \operatorname{supp}\mu$, so both inner products in \eqref{eq:split-variation} vanish for every $h \in \mathcal{D}$.

($\Rightarrow$) Suppose \eqref{eq:split-variation} is zero for all $h \in \mathcal{D}$; in particular for all $h \in C^1_c(\Omega) \subseteq \mathcal{D}$, which is all that the three steps below use.

\emph{Step 1 (off-diagonal test functions).}
For $h \in C^1_c(\Omega^\circ)$ we have $h_\Delta \equiv 0$, so \eqref{eq:split-variation} reduces to $\langle g_{\mathrm{SD}}[f], h\rangle_\mu = 0$.
Since $g_{\mathrm{SD}}[f]$ is continuous and $\mu$ has full support on the open set $\Omega^\circ$, Lemma~\ref{lem:fundamental} gives $g_{\mathrm{SD}}[f] = 0$ on $\Omega^\circ$.

\emph{Step 2 (extension to the closed triangle).}
$g_{\mathrm{SD}}[f]$ is continuous on $\bar\Omega$ and vanishes on the dense subset $\Omega^\circ$, so $g_{\mathrm{SD}}[f] = 0$ on $\bar\Omega \supseteq \operatorname{supp}\mu$, diagonal included.
Hence $\langle g_{\mathrm{SD}}[f], h\rangle_\mu = 0$ for \emph{every} $h \in \mathcal{D}$, regardless of whether $\mu$ places mass on the diagonal.
The two inner products in \eqref{eq:split-variation} therefore cannot cancel: the first is identically zero.

\emph{Step 3 (near-diagonal test functions).}
It remains to show $g_{\mathrm{FM}}[f] = 0$ on $\Omega_\Delta$.
Fix $\psi \in C^1([0,\infty))$ with $\psi(0) = 1$ and $\psi \equiv 0$ on $[\varepsilon,\infty)$.
For $\phi \in C^1_c(\Omega_\Delta)$ define $h(t,u,x) := \phi(t,x)\,\psi(u-t)$ on $\bar\Omega$.
Choosing $\varepsilon$ smaller than the distance in $t$ from the compact set $\operatorname{supp}\phi \subset (0,1) \times \bbR^d$ to $\{1\} \times \bbR^d$ ensures $h \in C^1_c(\Omega)$ (it is supported where $(t,x) \in \operatorname{supp}\phi$ and $t \le u < t + \varepsilon < 1$), and $h_\Delta(t,x) = \phi(t,x)\psi(0) = \phi(t,x)$.
By Step 2, \eqref{eq:split-variation} reduces to $\langle g_{\mathrm{FM}}[f], \phi\rangle_\nu = 0$ for all $\phi \in C^1_c(\Omega_\Delta)$.
Since $g_{\mathrm{FM}}[f]$ is continuous ($f \in C^1$, $v$ continuous) and $\nu$ has full support on the open set $\Omega_\Delta$, Lemma~\ref{lem:fundamental} gives $g_{\mathrm{FM}}[f] = 0$ on $\Omega_\Delta$, i.e.\ $f(t,t,x) = v(t,x)$.
\end{proof}

\begin{proof}[Proof of Corollary~\ref{cor:meanflow} (MeanFlow)]
\emph{Characterization.}
$\cL_{\mathrm{MF}}$ is of the form \eqref{eq:principle_loss} with $\mbx = (t,u,\mbx_t)$, $\mby = \dot \mbx_t$, $P[f] = (u-t)\partial_t f$, $Q[f] = (u-t)\partial_x f$, $R = I$, so $\E[\mby \mid \mbx] = v(t,\mbx_t)$ and $T[f,\E[\mby \mid \mbx]] = v + (u-t)(\partial_t f + \partial_x f\,v)$, which is continuous on $\bar\Omega$ by Remark~\ref{rmk:regularity}.
Theorem~\ref{thm:principle} (with $S = \Omega^\circ$, $\operatorname{supp}\mu \subseteq \bar\Omega = \bar S$) gives: $f$ is stationary iff $f = v + (u-t)(\partial_t f + \partial_x f\,v)$ on $\operatorname{supp}\mu$, hence on $\Omega^\circ$ and by continuity on $\bar\Omega$.
(On the diagonal the equation reads $f(t,t,\cdot) = v(t,\cdot)$, since the $(u-t)$ prefactor vanishes; when $\mu$ places mass on $\{t = u\}$, as in our experiments, Theorem~\ref{thm:principle} enforces this directly rather than by continuity.)
Using \eqref{eq:F-derivs}, this equation is equivalent to $\partial_t F + (\partial_x F)\,v = 0$ (cf.\ \eqref{eq:mf-equals-minus-Ev}).

\emph{Uniqueness.}
Let $f \in \mathcal{M}$ be stationary.
The transport equation holds on $\bar\Omega$, and the boundary condition $F(u,u,x) = x$ holds by the parameterization \eqref{eq:f_param}.
Lemma~\ref{lem:char} gives $F = F^*$, i.e.\ $f = f^*$ where $F^* = x + (u-t)f^*$.

\emph{Existence.}
$f^* \in \mathcal{M}$ by Remark~\ref{rmk:regularity}(5), and $F^*$ satisfies \eqref{eq:transport_intro} (\S\ref{sec:flow_maps}), so $g_{\mathrm{MF}}[f^*] = 0$ on $\bar\Omega$ and $f^*$ is stationary by Theorem~\ref{thm:principle}.
Hence $f^*$ is the unique stationary point in $\mathcal{M}$.
\end{proof}

\begin{proof}[Proof of Corollary~\ref{cor:imf} (Improved MeanFlow)]
\emph{Characterization.}
$\cL_{\mathrm{iMF}}$ is of the form \eqref{eq:principle_loss} with $\mbx = (t,u,\mbx_t)$, $\mby = \dot \mbx_t$, $P[f] = (u-t)\big(\partial_t f + (\partial_x f)\,f(t,t,\cdot)\big)$, $Q = 0$, $R = I$, so $\E[\mby \mid \mbx] = v(t,\mbx_t)$ and $T[f,\E[\mby \mid \mbx]] = v + (u-t)\big(\partial_t f + (\partial_x f)\,f(t,t,\cdot)\big)$, which is continuous on $\bar\Omega$ by Remark~\ref{rmk:regularity}.
Theorem~\ref{thm:principle} (with $S = \Omega^\circ$, $\operatorname{supp}\mu \subseteq \bar\Omega = \bar S$) gives: $f$ is stationary iff $f = v + (u-t)(\partial_t f + (\partial_x f)\,f(t,t,\cdot))$ on $\operatorname{supp}\mu$, hence on $\Omega^\circ$ and by continuity on $\bar\Omega$.

\emph{Uniqueness.}
Let $f \in \mathcal{M}$ be stationary.
Setting $u = t$ gives $f(t,t,x) = v(t,x)$ for all $(t,x) \in [0,1] \times \bbR^d$.
Substituting this into the equation gives $f = v + (u-t)(\partial_t f + (\partial_x f)\,v)$ on $\bar\Omega$, which by \eqref{eq:F-derivs} is $\partial_t F + (\partial_x F)\,v = 0$, with $F(u,u,x) = x$ by \eqref{eq:f_param}.
Lemma~\ref{lem:char} gives $F = F^*$.

\emph{Existence.}
$f^* \in \mathcal{M}$ by Remark~\ref{rmk:regularity}(5); $f^*(t,t,x) = \partial_u F^*(t,t,x) = v(t,x)$; and $F^*$ satisfies \eqref{eq:transport_intro}.
Hence $g_{\mathrm{iMF}}[f^*] = g_{\mathrm{MF}}[f^*] = 0$ on $\bar\Omega$ and $f^*$ is stationary by Theorem~\ref{thm:principle}.
Hence $f^*$ is the unique stationary point of $\cL_{\mathrm{iMF}}$ in $\mathcal{M}$.
\end{proof}

\begin{proof}[Proof of Corollary~\ref{cor:lsd_slim} (slim \gls{lsd})]
\emph{Characterization.}
$\cL_{\mathrm{LSD\text{-}\slim}}$ is of the form \eqref{eq:principle_loss} with $\mby \equiv 0$ and $T[f](t,u,x) = -(u-t)\,\partial_u f(t,u,x) + f\big(u,u,F(t,u,x)\big)$, continuous on $\bar\Omega$ by Remark~\ref{rmk:regularity}.
Theorem~\ref{thm:principle} gives: $f$ is stationary for the slim-\gls{lsd} term iff $f + (u-t)\partial_u f = f(u,u,F)$ on $\operatorname{supp}\mu$, hence on $\bar\Omega$, i.e.\ by \eqref{eq:F-derivs} $\partial_u F(t,u,x) = f\big(u,u,F(t,u,x)\big)$.

\emph{Combined objective.}
Lemma~\ref{lem:split} applies: $f$ is stationary for $\cL_{\mathrm{LSD\text{-}\slim}} + \cL_{\mathrm{FM}}$ iff $f(u,u,\cdot) = v(u,\cdot)$ on $\Omega_\Delta$ and $\partial_u F(t,u,x) = f(u,u,F(t,u,x))$ on $\Omega$.
Substituting the first into the second and extending by continuity,
\begin{align*}
    \partial_u F(t,u,x) = v\big(u,F(t,u,x)\big) \ \text{ for all } (t,u,x) \in \bar\Omega,
    \qquad F(t,t,x) = x .
\end{align*}

\emph{Uniqueness.}
Let $f \in \mathcal{M}$ be stationary for the combined objective.
For fixed $(t,x)$, $u \mapsto F(t,u,x)$ is a $C^1$ solution on $[t,1]$ of the \gls{ode} $\dot X_u = v(u,X_u)$ with $X_t = x$.
Since $v$ is Lipschitz in $x$ uniformly in $u$, the solution is unique (Picard--Lindel\"of), so $F(t,u,x) = F^*(t,u,x)$.

\emph{Existence.}
$f^* \in \mathcal{M}$ by Remark~\ref{rmk:regularity}(5).
Since $F^*(t,u,x) = x + (u-t)f^*(t,u,x)$, we have $f^*(t,t,x) = \partial_u F^*(t,t,x) = v(t,x)$, so $g_{\mathrm{FM}}[f^*] = 0$; and $\partial_u F^* - f^*(u,u,F^*) = \partial_u F^* - v(u,F^*) = 0$ by \eqref{eq:lagrangian_identity}, so $g_{\mathrm{sLSD}}[f^*] = 0$.
By Lemma~\ref{lem:split}, $f^*$ is stationary.
Hence the true flow map is the unique stationary point of the combined objective in $\mathcal{M}$.

\emph{Global minimizer.}
Both terms have additive label noise (Definition~\ref{def:additive}): the slim term has $\mby \equiv 0$ and floor $0$, and the flow matching term has $T[f,\mby] = \mby$ and floor $V_{\mathrm{FM}} := \tfrac12\,\E\|\dot \mbx_t - v(t,\mbx_t)\|^2$, independent of $f$.
By Proposition~\ref{prop:floor} applied to each term,
$\cL_{\mathrm{LSD\text{-}\slim}}[f] + \cL_{\mathrm{FM}}[f] = \tfrac12\|g_{\mathrm{sLSD}}[f]\|_\mu^2 + \tfrac12\|g_{\mathrm{FM}}[f]\|_\nu^2 + V_{\mathrm{FM}}$.
The combined loss therefore attains the value $V_{\mathrm{FM}}$ exactly when both residuals vanish, which by the above happens only at $f^*$.
So the true flow map is also the unique global minimizer of the combined objective in $\mathcal{M}$, and the combined loss minus $V_{\mathrm{FM}}$ is the stopgrad-free \gls{lsd} objective with the marginal velocity $v$ in place of $\dot \mbx_t$ \citep{boffi2025build}.
This does not replace Lemma~\ref{lem:split}: identifying the minimizers with the stationary points still requires that the two first variations cannot cancel.
\end{proof}

\begin{proof}[Proof of Corollary~\ref{cor:esd_slim} (slim \gls{esd})]
\emph{Characterization.}
$\cL_{\mathrm{ESD\text{-}\slim}}$ is of the form \eqref{eq:principle_loss} with $\mby \equiv 0$ and $T[f](t,u,x) = (u-t)\,\partial_t f(t,u,x) + (\partial_x F)(t,u,x)\,f(t,t,x)$, continuous on $\bar\Omega$ by Remark~\ref{rmk:regularity}.
Theorem~\ref{thm:principle} gives: $f$ is stationary for the slim-\gls{esd} term iff $f = (u-t)\partial_t f + (\partial_x F)\,f(t,t,\cdot)$ on $\bar\Omega$, i.e.\ by \eqref{eq:F-derivs} $\partial_t F + (\partial_x F)\,f(t,t,x) = 0$.

\emph{Combined objective.}
Lemma~\ref{lem:split} applies: $f$ is stationary for $\cL_{\mathrm{ESD\text{-}\slim}} + \cL_{\mathrm{FM}}$ iff $f(t,t,\cdot) = v(t,\cdot)$ on $\Omega_\Delta$ and $\partial_t F + (\partial_x F)f(t,t,\cdot) = 0$ on $\Omega$.
Substituting and extending by continuity, $\partial_t F + (\partial_x F)\,v = 0$ on $\bar\Omega$ with $F(u,u,x) = x$.

\emph{Uniqueness.}
Let $f \in \mathcal{M}$ be stationary for the combined objective.
Lemma~\ref{lem:char} gives $F = F^*$.

\emph{Existence.}
$f^* \in \mathcal{M}$ by Remark~\ref{rmk:regularity}(5).
As in the proof of Corollary~\ref{cor:lsd_slim}, $f^*(t,t,\cdot) = v(t,\cdot)$, and $F^*$ satisfies \eqref{eq:transport_intro}, so both residuals vanish at $f^*$ and $f^*$ is stationary by Lemma~\ref{lem:split}.
Hence the true flow map is the unique stationary point of the combined objective in $\mathcal{M}$.

\emph{Global minimizer.}
Exactly as for slim \gls{lsd}, both terms have additive label noise, so with $g_{\mathrm{sESD}}[f] := f - T[f]$ the slim \gls{esd} residual,
$\cL_{\mathrm{ESD\text{-}\slim}}[f] + \cL_{\mathrm{FM}}[f] = \tfrac12\|g_{\mathrm{sESD}}[f]\|_\mu^2 + \tfrac12\|g_{\mathrm{FM}}[f]\|_\nu^2 + V_{\mathrm{FM}}$ with $V_{\mathrm{FM}}$ independent of $f$.
The combined loss attains $V_{\mathrm{FM}}$ exactly when both residuals vanish, which happens only at $f^*$.
So the true flow map is also the unique global minimizer of the combined objective in $\mathcal{M}$, and the combined loss minus $V_{\mathrm{FM}}$ is the stopgrad-free \gls{esd} objective with $v$ in place of $\dot \mbx_t$ \citep{boffi2025build}.
\end{proof}

\section{Unifying stopgrad objectives under the stopgrad regression principle}
\label{appsec:unifying_stopgrad_objectives}

In this section, we highlight the generality of the stopgrad regression principle by showing that it applies to stopgrad objectives not just for flow map training, but also for reinforcement learning and diffusion sampler training.
Theorem~\ref{thm:principle} works with continuous, smooth functions and distribution $\mu$ on inputs with full support.
As such, we consider settings that match the presentation of Theorem~\ref{thm:principle}, such as continuous states, actions, and rewards in reinforcement learning. 

\paragraph{Measure-valued targets}
In the reinforcement learning examples below, we use the random Dirac measure $\delta_{S_{t+1}}$ on the random next state $S_{t+1}$ in the target $\mby$.
The proof of Theorem~\ref{thm:principle} extends to such measure-valued targets because it only uses linearity of conditional expectations and affine dependence of $T$ on $\mby$.
In particular, the conditional expectation of a random Dirac measure is the transition kernel: $\E[\delta_{S_{t+1}} \mid S_t = s ] = P(\cdot\mid s)$.

\subsection{Temporal difference learning: TD(0)}

We consider a continuous-state Markov reward process on an open set
\(\mathcal S \subset \mathbb R^d\).
Let \(P(ds'\mid s)\) be the transition kernel and let $\bar r(s):=\E[R_{t+1}\mid S_t=s]$ be the conditional mean reward.
We assume:
\begin{enumerate}
    \item \(V \in C^1(\mathcal S)\).
    \item The training state distribution \(\mu\) has full support on \(\mathcal S\), and \(\operatorname{supp}\mu \subseteq \bar{\mathcal S}\).
    \item \(\bar r\) is continuous.
    \item \(P\) is Feller: for every bounded continuous \(\phi\), the map $s \mapsto \int \phi(s')\,P(ds'\mid s)$ is continuous.
    \item The residual in the objective below is square-integrable.
\end{enumerate}

Consider the semi-gradient TD(0) functional \citep{sutton2018reinforcement}
\begin{align}
\cL_{\mathrm{TD}}[V]
:=
\frac12\,\E\Big[
\big(
V(S_t)-\sgb{R_{t+1}+\gamma V(S_{t+1})}
\big)^2
\Big]. \notag
\end{align}
To place this in the form of the stopgrad regression principle, let $\mbx = S_t$, $\mby = (R_{t+1},\delta_{S_{t+1}})$, and define
\begin{align}
T[V,\mby](s)
=
r+\gamma\int_{\mathcal S} V(s')\,\delta_{S_{t+1}}(ds'). \notag
\end{align}
This is affine in \(\mby=(r,\delta_{S_{t+1}})\), with
\begin{align}
P[V](s)=0,
\qquad
Q[V](s)\,(r,\delta_{S_{t+1}})=\gamma\int_{\mathcal S} V(s')\,\delta_{S_{t+1}}(ds'),
\qquad
R(s)\,(r,\delta_{S_{t+1}})=r. \notag
\end{align}
Moreover, $\E[\mby\mid \mbx=s] = \big(\bar r(s),\,P(\cdot\mid s)\big)$, so $T\!\big[V,\E[\mby\mid \mbx=s]\big](s) = \bar r(s)+\gamma\int_{\mathcal S} V(s')\,P(ds'\mid s)$.
By continuity of \(\bar r\), continuity of \(V\), and the Feller property, the map $s\mapsto T\!\big[V,\E[\mby\mid \mbx=s]\big](s)$ is continuous. Thus the assumptions of Theorem~\ref{thm:principle} are satisfied.

\begin{tcolorbox}[boxrule=0pt, frame empty]
\begin{corollary}[Semi-gradient TD(0)]
\label{cor:td0}
Under the assumptions above, \(V^*\) is a stationary point of \(\cL_{\mathrm{TD}}\) if and only if
\begin{align}
V^*(s)
=
\bar r(s)
+
\gamma\int_{\mathcal S} V^*(s')\,P(ds'\mid s)
\qquad \text{for all } s\in\mathcal S. \notag
\end{align}
That is, the stationary points are exactly the solutions of the Bellman evaluation equation.
\end{corollary}
\end{tcolorbox}
Theorem~\ref{thm:principle} characterizes the stationary points as solutions of the Bellman equation; uniqueness is not a consequence of Theorem~\ref{thm:principle} (the label noise is not additive, since $Q[V]$ involves $V$) but is classical: for $\gamma < 1$ the Bellman evaluation operator $V \mapsto \bar r + \gamma \int V\,dP(\cdot \mid s)$ is a $\gamma$-contraction in the sup norm on bounded functions \citep{shapley1953stochastic,blackwell1965discounted,denardo1967contraction}, so it has exactly one fixed point, the true value function.
Uniqueness in this setting is not a contribution of this work.

\subsection{Q-learning}

We next consider a continuous-state, continuous-action Markov decision process with
state space \(\mathcal S \subset \mathbb R^d\) open and action space
\(\mathcal A \subset \mathbb R^m\) compact.
Let \(P(ds'\mid s,a)\) be the transition kernel and let $\bar r(s,a):=\E[R_{t+1}\mid S_t=s,\ A_t=a]$ be the conditional mean reward.
We assume:
\begin{enumerate}
    \item \(q \in C^1(\mathcal S\times \mathcal A)\).
    \item The training distribution of \((S_t,A_t)\) has full support on \(\mathcal S\times \mathcal A\), with support contained in \(\bar{\mathcal S}\times \mathcal A\).
    \item \(\bar r\) is continuous on \(\mathcal S\times \mathcal A\).
    \item \(P\) is Feller on state-action pairs: for every bounded continuous \(\phi\), the map $(s,a)\mapsto \int \phi(s')\,P(ds'\mid s,a)$ is continuous.
    \item The residual in the objective below is square-integrable.
\end{enumerate}

Define
\begin{align}
M_q(s):=\max_{a'\in\mathcal A} q(s,a'). \notag
\end{align}
Because \(\mathcal A\) is compact and \(q\) is continuous, \(M_q\) is continuous.
Now consider the semi-gradient Q-learning functional \citep{sutton2018reinforcement}
\begin{align}
\cL_{\mathrm{Q}}[q]
:=
\frac12\,\E\Big[
\big(
q(S_t,A_t)-\sgb{R_{t+1}+\gamma M_q(S_{t+1})}
\big)^2
\Big]. \notag
\end{align}
To write this in terms of the stopgrad regression principle, let $\mbx=(S_t,A_t)$, $\mby=(R_{t+1},\delta_{S_{t+1}})$, and define
\begin{align}
T[q,\mby](s,a)
=
r+\gamma\int_{\mathcal S} M_q(s')\,\delta_{S_{t+1}}(ds'). \notag
\end{align}
This is affine in \(\mby=(r,\delta_{S_{t+1}})\), with
\begin{align}
P[q](s,a)=0,
\qquad
Q[q](s,a)\,(r,\delta_{S_{t+1}})=\gamma\int_{\mathcal S} M_q(s')\,\delta_{S_{t+1}}(ds'),
\qquad
R(s,a)\,(r,\delta_{S_{t+1}})=r. \notag
\end{align}
Moreover, $\E[\mby\mid \mbx=(s,a)] = \big(\bar r(s,a),\,P(\cdot\mid s,a)\big)$, so $T\!\big[q,\E[\mby\mid \mbx=(s,a)]\big](s,a) = \bar r(s,a)+\gamma\int_{\mathcal S} M_q(s')\,P(ds'\mid s,a)$.
By continuity of \(\bar r\), continuity of \(M_q\), and the Feller property, this map is continuous on \(\mathcal S\times \mathcal A\). Thus Theorem~\ref{thm:principle} applies.

\begin{tcolorbox}[boxrule=0pt, frame empty]
\begin{corollary}[Q-learning]
\label{cor:q_learning}
Under the assumptions above, \(q^*\) is a stationary point of \(\cL_{\mathrm{Q}}\) if and only if
\begin{align}
q^*(s,a)
=
\bar r(s,a)
+
\gamma\int_{\mathcal S}
\max_{a'\in\mathcal A} q^*(s',a')\,
P(ds'\mid s,a)
\qquad \text{for all } (s,a)\in\mathcal S\times \mathcal A. \notag
\end{align}
That is, the stationary points are exactly the solutions of the Bellman optimality equation.
\end{corollary}
\end{tcolorbox}
As for \textsc{td(}{\small 0}\textsc{)}, uniqueness is not a consequence of Theorem~\ref{thm:principle} (the label noise is not additive, since $Q[q]$ involves $q$) but is classical: for $\gamma < 1$ the Bellman optimality operator is a $\gamma$-contraction in the sup norm on bounded functions \citep{shapley1953stochastic,blackwell1965discounted,denardo1967contraction}, so it has exactly one fixed point, the optimal action-value function.

\subsection{Adjoint matching: the inner RAM objective}

Consider the inner regression step used in adjoint matching and adjoint sampling \citep{domingo-enrich2025adjointmatching, havens2025adjointsampling}.
Let \(g:\mathbb R^d\to\mathbb R\) be a terminal objective.
The inner RAM objective regresses a control field \(u(x,t)\) toward the target
\(-\sigma(t)\nabla g(X_1)\), 
on stopgrad'd samples $(X_t, X_1) \sim p_u$, the stochastic process for the learned control $u$.
The stopgrad avoids backpropagating through $p_u$, and is denoted by $\bar u$.
The RAM loss functional can be written as an expectation on $u$ as:
\begin{align}
    \cL_{\mathrm{RAM}}[u]
    :=
    \int_0^1
    \E_{(X_t,X_1)\sim p_{t,1}^{u}}
    \Big[
    \frac12\,
    \big\|
    u( \sgb{X_t},t) + 
    \sgb{\sigma(t)\nabla g(X_1) }
    \big\|^2
    \Big]\,dt. \notag
\end{align}
Equivalently, rewriting the expectation to be on $\bar u$, and absorbing the time integral into the expectation:
\begin{align}
    \cL_{\mathrm{RAM}}[u;\bar u]
    =
    \frac12\,\E_{t, (X_t, X_1) \sim p^{\bar u}_{t,1}}\Big[
    \big\|
    u(X_t,t)-\sgb{-\sigma(t)\nabla g(X_1)}
    \big\|^2
    \Big]. \notag
\end{align}

Set $\mbx=(X_t,t)$, $\mby=\nabla g(X_1)$ and define $T[u,\mby](x,t) = -\sigma(t)\,\mby.$
Then this is affine in \(\mby\), with
\begin{align}
P[u](x,t)=0,
\qquad
Q[u](x,t)\,\mby=0,
\qquad
R(x,t)\,\mby=-\sigma(t)\,\mby. \notag
\end{align}
Moreover, $\E[\mby\mid \mbx=(x,t)] = \E[\nabla g(X_1)\mid X_t=x] = m(x,t)$,
so $T\!\big[u,\E[\mby\mid \mbx=(x,t)]\big](x,t) = -\sigma(t)\,m(x,t)$
which is continuous by assumption. Hence Theorem~\ref{thm:principle} applies directly.

\begin{tcolorbox}[boxrule=0pt, frame empty]
\begin{corollary}[Adjoint matching inner RAM step]
\label{cor:ram}
Under the assumptions above, \(u^*\) is a stationary point of \(\cL_{\mathrm{RAM}}\) if and only if
\begin{align}
u^*(x,t)
=
-\sigma(t)\,\E[\nabla g(X_1)\mid X_t=x]
\qquad \text{for all } (x,t)\in \mathcal X\times(0,1). \notag
\end{align}
\end{corollary}
\end{tcolorbox}

\section{The update direction may not be the gradient of any scalar functional}
\label{appsec:update_is_a_gradient}
We state and prove Theorem~\ref{thm:g-conservative}. Throughout, $\langle \phi,\psi\rangle_\mu := \E_\mu[\phi(\mbx)\cdot \psi(\mbx)]$ is the inner product of $H = L^2(\mu)$, and $\mathcal{M} \subset C^1 \cap H$ is a linear subspace, so that $f + \epsilon h \in \mathcal{M}$ for $f, h \in \mathcal{M}$. For a map $G : \mathcal{M} \to H$, the Gateaux derivative at $f$ in direction $h$ is $D G[f]\,h := \lim_{\epsilon \to 0} (G[f + \epsilon h] - G[f])/\epsilon$, the limit taken in $H$. Directions range over all of $\mathcal{M}$, not only over the perturbation class $\mathcal{D}$ used to define stationary points; this is what admits the polynomial witnesses below.

The following theorem states the symmetry criterion with its hypotheses explicit.
\begin{tcolorbox}[boxrule=0pt, frame empty]
\begin{theorem}[Conservativeness of the semi-gradient $g$]
  \label{thm:g-conservative}
  Assume the hypotheses of Theorem~\ref{thm:principle} hold for every $f \in \mathcal{M}$, with $\mathcal{M}$ a linear subspace of $C^1 \cap H$, and suppose that $T$ is Gateaux differentiable along $\mathcal{M}$: for all $f, h \in \mathcal{M}$ the limit $D T[f]\,h := \lim_{\epsilon \to 0} \big(T[f+\epsilon h, \E[\mby \mid \mbx]] - T[f, \E[\mby \mid \mbx]]\big)/\epsilon$ exists in $H$.
  If the semi-gradient $g[f] = f - T[f, \E[\mby \mid \mbx]]$ is the gradient in $\langle \cdot, \cdot \rangle_\mu$ of a scalar functional $J : \mathcal{M} \to \R$, meaning $\delta J[f; h] = \langle g[f], h \rangle_\mu$ for all $f, h \in \mathcal{M}$, with $(\epsilon_1,\epsilon_2) \mapsto J[f + \epsilon_1 h_1 + \epsilon_2 h_2]$ twice continuously differentiable near $(0,0)$ for all $f, h_1, h_2 \in \mathcal{M}$, then $D T[f]$ is symmetric with respect to $\langle \cdot, \cdot \rangle_\mu$ at every $f \in \mathcal{M}$:
  \begin{align*}
      \langle D T[f]\, h_1,\, h_2 \rangle_\mu = \langle h_1,\, D T[f]\, h_2 \rangle_\mu
      \qquad \text{for all } h_1, h_2 \in \mathcal{M}.
  \end{align*}
\end{theorem}
\end{tcolorbox}
\begin{proof}[Proof of Theorem~\ref{thm:g-conservative}]
By Theorem~\ref{thm:principle}, $g[f] = f - T[f,\, \E[\mby\mid\mbx]]$, so $g : \mathcal{M} \to H$ is Gateaux differentiable along $\mathcal{M}$ with
\begin{align}
  D g[f]\,h \;=\; h - D T[f]\,h .
\end{align}
Lemma~\ref{lem:functional-poincare} below, applied with $V = \mathcal{M}$, gives $\langle D g[f]\,h_1, h_2\rangle_\mu = \langle h_1, D g[f]\,h_2\rangle_\mu$ for all $h_1, h_2 \in \mathcal{M}$. Subtracting the identity $\langle h_1, h_2\rangle_\mu = \langle h_1, h_2\rangle_\mu$ gives $\langle D T[f]\,h_1, h_2\rangle_\mu = \langle h_1, D T[f]\,h_2\rangle_\mu$.
\end{proof}

\begin{tcolorbox}[boxrule=0pt, frame empty]
\begin{lemma}[Symmetry of the second variation]
  \label{lem:functional-poincare}
  Let $H$ be a real inner-product space (not necessarily complete) and $V \subseteq H$ a linear subspace. Let $J : V \to \R$ and $g : V \to H$ satisfy $\delta J[f;\, h] = \langle g[f],\, h\rangle$ for all $f, h \in V$. Fix $f, h_1, h_2 \in V$. Assume that $D g[f]\,h_i := \lim_{\epsilon \to 0}(g[f + \epsilon h_i] - g[f])/\epsilon$ exists in $H$ for $i = 1, 2$, and that the map $\phi(\epsilon_1,\epsilon_2) := J[f + \epsilon_1 h_1 + \epsilon_2 h_2]$ is $C^2$ on a neighborhood of $(0,0)$. Then
  \begin{align*}
    \langle D g[f]\,h_1,\, h_2\rangle = \langle h_1,\, D g[f]\,h_2\rangle .
  \end{align*}
\end{lemma}
\end{tcolorbox}

\begin{proof}[Proof of \cref{lem:functional-poincare}]
Fix $f \in V$ and $h_1, h_2 \in V$, and let $\phi(\epsilon_1,\epsilon_2) := J[f + \epsilon_1 h_1 + \epsilon_2 h_2]$. By the hypothesis on $J$,
\begin{align}
  \partial_{\epsilon_2}\phi(\epsilon_1, 0)
  = \delta J[f + \epsilon_1 h_1;\, h_2]
  = \langle g[f + \epsilon_1 h_1],\, h_2\rangle .
\end{align}
Differentiating in $\epsilon_1$ at $0$: the difference quotient $\big(g[f+\epsilon_1 h_1] - g[f]\big)/\epsilon_1$ converges to $D g[f]\,h_1$ in $H$, so by Cauchy--Schwarz $\partial_{\epsilon_1}\partial_{\epsilon_2}\phi(0,0) = \langle D g[f]\,h_1,\, h_2\rangle$. Exchanging the roles of $h_1$ and $h_2$, $\partial_{\epsilon_2}\partial_{\epsilon_1}\phi(0,0) = \langle D g[f]\,h_2,\, h_1\rangle$. Since $\phi$ is $C^2$ near $(0,0)$, Schwarz's theorem on $\R^2$ gives equality of the mixed partials, and inner-product symmetry gives $\langle D g[f]\,h_1,\, h_2\rangle = \langle h_1,\, D g[f]\,h_2\rangle$.
\end{proof}

\paragraph{Remarks.}
The converse, that symmetry of $D g$ implies the existence of a potential $J[f] = \int_0^1 \langle g[sf], f\rangle\,ds$ on a star-shaped domain, is Vainberg's criterion for potential operators \citep[Theorem~5.1 and Corollary~5.1]{vainberg1964variational}; it needs additional regularity of $g$ along segments and is not used here. Only the forward direction is needed to show that a semi-gradient is \emph{not} a gradient, and it requires neither completeness of $V$ nor boundedness of $D g$. When $H = \R^n$ with the dot product, $D g[f]$ is the Jacobian matrix of $g$ at $f$ and the lemma is equality of mixed partials of a $C^2$ scalar; this is the curl perspective on game dynamics in \citet{balduzzi2018mechanics}.

We restate Corollary~\ref{cor:non-conservative}.
\begin{tcolorbox}[boxrule=0pt, frame empty]
\begin{corollary*}
    The semi-gradients of $\cL_{\mathrm{MF}}$, $\cL_{\mathrm{iMF}}$, $\cL_{\mathrm{LSD\text{-}\slim}}$, and $\cL_{\mathrm{ESD\text{-}\slim}}$ are not gradients of any twice continuously differentiable scalar functional $J : \mathcal{M} \to \R$ in $\langle \cdot, \cdot \rangle_\mu$ (in the sense of Theorem~\ref{thm:g-conservative}).
\end{corollary*}
\end{tcolorbox}

\begin{proof}[Proof of \Cref{cor:non-conservative}]
Suppose $g = f - T[f,\E[\mby \mid \mbx]]$ were the gradient of a scalar functional $J : \mathcal{M} \to \R$ with the regularity stated in Theorem~\ref{thm:g-conservative}. Lemma~\ref{lem:functional-poincare}, applied at a base point $f_0 \in \mathcal{M}$ in two directions $h_1, h_2 \in \mathcal{M}$ at which $DT[f_0]\,h_1$ and $DT[f_0]\,h_2$ exist in $H$, would give $\langle DT[f_0]\,h_1, h_2\rangle_\mu = \langle h_1, DT[f_0]\,h_2\rangle_\mu$. It therefore suffices to exhibit one such $f_0, h_1, h_2$ at which
\begin{equation}
\langle DT[f_0]\, h_1,\, h_2 \rangle_\mu \;\neq\; \langle h_1,\, DT[f_0]\, h_2 \rangle_\mu.
\label{eq:non-sa-witness}
\end{equation}
For the witnesses below, existence of $DT[f_0]\,h_i$ is immediate: they are polynomials in $(t,u)$ with no $x$-dependence, so the terms of $T$ that are nonlinear in $f$ (the products $(\partial_x f)\,f(t,t,\cdot)$ in iMF and slim \gls{esd}, and the composition $f(u,u,F)$ in slim \gls{lsd}) contribute nothing to $T[\epsilon h_i] - T[0]$, which is exactly linear in $\epsilon$ and bounded on the closed triangle.

For each method we use the base point $f_0 \equiv 0$, which corresponds to the identity flow map $F_0(t,u,x) = x$ under the parameterization $F = x + (u-t)f$ and lies in $\mathcal{M}$. At this $f_0$ the spatial-Jacobian terms in $DT[f]$ vanish, $\partial_x F_0 = I$, and $f_0(t,t,\cdot) = 0$, leaving:
\begin{align}
\text{MeanFlow:} \quad DT[f_0]\,h &= (u-t)\big(\partial_t h + (\partial_x h)\,v(t,x)\big), \label{eq:DT-mf}\\
\text{iMF:} \quad DT[f_0]\,h &= (u-t)\,\partial_t h, \label{eq:DT-imf}\\
\text{slim ESD:} \quad DT[f_0]\,h &= (u-t)\,\partial_t h \;+\; h(t,t,x), \label{eq:DT-esd}\\
\text{slim LSD:} \quad DT[f_0]\,h &= -(u-t)\,\partial_u h \;+\; h(u,u,x). \label{eq:DT-lsd}
\end{align}
For witnesses we take $h_1, h_2$ that are low-degree polynomials in $t, u$ with no $x$-dependence, times a fixed unit vector $e \in \bbR^d$ that we suppress in the notation; they are bounded on the closed time triangle, hence lie in $\mathcal{M}$ for either choice of model class ($C^1 \cap H$ or $C^\infty_{\mathrm{pol}} \cap H$), and so do $DT[f_0]\,h_i$. Constant-in-$x$ perturbations annihilate the $\partial_x h$ term in MeanFlow, and the resulting $\mu$-inner products reduce to integrals against the time-marginal $\mu_{(t,u)}$. Throughout, we use the assumption
\begin{equation}
\mu_{(t,u)}\big(\{0 \le t < u \le 1\}\big) \;>\; 0,
\label{eq:positive-mass}
\end{equation}
which holds whenever the diagonal $\{t = u\}$ does not have full mass under $\mu_{(t,u)}$, in particular under Remark~\ref{rmk:regularity}(4).

\paragraph{MeanFlow.} Take $h_1(t,u,x) = t^2$ and $h_2(t,u,x) = t$. Both are constant in $x$, so $\partial_x h_i = 0$ and \eqref{eq:DT-mf} simplifies to $DT[f_0]\,h_1 = 2t(u-t)$ and $DT[f_0]\,h_2 = (u-t)$. Then
\[
\langle DT[f_0]\,h_1,\,h_2\rangle_\mu - \langle h_1,\,DT[f_0]\,h_2\rangle_\mu
\;=\; \mathbb{E}_{\mu_{(t,u)}}\!\big[(u-t)(2t \cdot t - t^2 \cdot 1)\big]
\;=\; \mathbb{E}_{\mu_{(t,u)}}\!\big[(u-t)\,t^2\big].
\]
The integrand is nonnegative on $\{t \le u\}$ and strictly positive on the open set $\{0 < t < u\}$, which has positive $\mu_{(t,u)}$-mass because $\mu$ has full support on $\Omega^\circ$ (Remark~\ref{rmk:regularity}(4)). Hence the expectation is strictly positive. \hfill$\square$

\paragraph{Improved MeanFlow.} With the same $h_1 = t^2$, $h_2 = t$, \eqref{eq:DT-imf} gives $DT[f_0]\,h_1 = 2t(u-t)$ and $DT[f_0]\,h_2 = (u-t)$, exactly as for MeanFlow, so the asymmetry is again $\mathbb{E}_{\mu_{(t,u)}}[(u-t)\,t^2] > 0$. \hfill$\square$

\paragraph{Slim ESD.} Take $h_1(t,u,x) = u$ and $h_2(t,u,x) = 1$. Both are constant in $x$, and $\partial_t h_1 = \partial_t h_2 = 0$, so the transport term in \eqref{eq:DT-esd} vanishes for both. The evaluation-shift term gives $h_1(t,t,x) = t$ and $h_2(t,t,x) = 1$. Hence
\[
\langle DT[f_0]\,h_1,\,h_2\rangle_\mu - \langle h_1,\,DT[f_0]\,h_2\rangle_\mu
\;=\; \mathbb{E}_{\mu_{(t,u)}}\!\big[t \cdot 1 - u \cdot 1\big]
\;=\; \mathbb{E}_{\mu_{(t,u)}}\!\big[t - u\big].
\]
The integrand is nonpositive on $\{t \le u\}$ and strictly negative on $\{t < u\}$. By \eqref{eq:positive-mass}, the expectation is strictly negative. \hfill$\square$

\paragraph{Slim LSD.} Take $h_1(t,u,x) = t$ and $h_2(t,u,x) = 1$. Both are constant in $x$, and $\partial_u h_1 = \partial_u h_2 = 0$, so the transport term in \eqref{eq:DT-lsd} vanishes for both. The evaluation-shift term gives $h_1(u,u,x) = u$ and $h_2(u,u,x) = 1$. Hence
\[
\langle DT[f_0]\,h_1,\,h_2\rangle_\mu - \langle h_1,\,DT[f_0]\,h_2\rangle_\mu
\;=\; \mathbb{E}_{\mu_{(t,u)}}\!\big[u \cdot 1 - t \cdot 1\big]
\;=\; \mathbb{E}_{\mu_{(t,u)}}\!\big[u - t\big].
\]
The integrand is nonnegative on $\{t \le u\}$ and strictly positive on $\{t < u\}$. By \eqref{eq:positive-mass}, the expectation is strictly positive. \hfill$\square$

\paragraph{Conclusion.} For each $T \in \{T_{\mathrm{MF}}, T_{\mathrm{iMF}}, T_{\mathrm{slim\,ESD}}, T_{\mathrm{slim\,LSD}}\}$ we have exhibited $f_0 \in \mathcal{M}$ and $h_1, h_2 \in \mathcal{M}$ at which \eqref{eq:non-sa-witness} holds. By the contrapositive of Theorem~\ref{thm:g-conservative}, the semi-gradients $g_{\mathrm{MF}}, g_{\mathrm{iMF}}, g_{\mathrm{slim\,ESD}}, g_{\mathrm{slim\,LSD}}$ are not gradients of any scalar functional $J: \mathcal{M} \to \mathbb{R}$ in $\langle\cdot,\cdot\rangle_\mu$.

\end{proof}

\section{ MeanFlow without stopgrad loses stationarity at the true flow map }
\label{appsec:proofs_mf_nosg}

\begin{tcolorbox}[boxrule=0pt, frame empty]
\begin{proposition}
    \label{prop:mf_nosg_general}
    Consider the MeanFlow functional \emph{without} stopgrads:
    \begin{align}
      \label{eq:mf_nosg}
      \cL_{\mathrm{nosg}}[f]
      = \E\Big[\big\|
        f - \dot \mbx_t - (u-t)\big(\partial_t f
        + \partial_x f\, \dot \mbx_t\big)
      \big\|^2\Big],
    \end{align}
    where all functions are evaluated at $(t,u,\mbx_t)$, $f, h : [0,1]^2 \times \R^d \to \R^d$,
    and $\partial_x f, \partial_x h \in \R^{d\times d}$ are Jacobians in the spatial argument.
    Define the \emph{posterior covariance}
    $V(t, x) := \Cov(\dot \mbx_t \mid \mbx_t = x) \in \R^{d\times d}$,
    which measures the residual randomness in $\dot \mbx_t$
    after conditioning on $\mbx_t$.
    At the true flow map $F^*$, the first variation in direction $h$ is
    \begin{align}
      \label{eq:mf_nosg_variation}
      \delta \cL_{\mathrm{nosg}}[f^*; h]
      = 2\,\E_{t,u,\mbx_t}\Big[
        (u-t)\,\tr\!\big(
          \partial_x h(t,u,\mbx_t)\, V(t, \mbx_t)\, \partial_x F^*(t,u,\mbx_t)^\top
        \big)
      \Big],
    \end{align}
    where $F^* = x + (u-t)f^*$ and
    $\partial_x F^* = I + (u-t)\partial_x f^*$.
\end{proposition}
\end{tcolorbox}

\begin{proof}[Proof of Proposition~\ref{prop:mf_nosg_general}]
    Let $f^\epsilon = f + \epsilon h$.
    Since there are no stopgrads, this is the usual first variation:
    every occurrence of $f$ is perturbed.
    The quantity inside the norm is
    \begin{align}
      R(\epsilon) := f^\epsilon - \dot \mbx_t
      - (u-t)\big(\partial_t f^\epsilon
      + \partial_x f^\epsilon\, \dot \mbx_t\big) \;\in\; \R^d.
    \end{align}
    Its $\epsilon$-derivative is
    \begin{align}
      R'(\epsilon) = h - (u-t)\big(\partial_t h + \partial_x h\, \dot \mbx_t\big) \;\in\; \R^d.
    \end{align}
    The first variation is the scalar
    $\delta \cL = 2\,\E[\langle R(0),\, R'(0)\rangle]$.

    \medskip
    \noindent\textbf{Evaluating $R(0)$ at $f^*$.}
    Write $\dot \mbx_t = v + \xi$ where
    $v = v(t, \mbx_t) = \E[\dot \mbx_t \mid \mbx_t]$
    and $\xi = \dot \mbx_t - v$ has $\E[\xi \mid \mbx_t] = 0$
    and $\Cov(\xi \mid \mbx_t) = V(t, \mbx_t)$.
    Then
    \begin{align}
      R(0)
      &= f^* - (v + \xi)
        - (u-t)\big(\partial_t f^*
        + \partial_x f^*\,(v + \xi)\big) \\
      &= \underbrace{\big[f^* - v
        - (u-t)(\partial_t f^* + \partial_x f^*\, v)\big]}_{= 0 \text{ (flow map eqn)}}
        - \big(I + (u-t)\partial_x f^*\big)\xi \\
      &= -\,\partial_x F^*\, \xi,
    \end{align}
    recalling that $\partial_x F^* = I + (u-t)\partial_x f^*$.

    \medskip
    \noindent\textbf{Evaluating $R'(0)$.}
    Similarly decompose:
    \begin{align}
      R'(0) = \underbrace{\big[h - (u-t)(\partial_t h + \partial_x h\, v)\big]}_{=:\, \phi}
      - (u-t)\,\partial_x h\, \xi.
    \end{align}
    Here $\phi \in \R^d$ is deterministic given $(t,u,\mbx_t)$.

    \medskip
    \noindent\textbf{Computing $\E[\langle R(0), R'(0)\rangle]$.}
    Expanding the inner product,
    \begin{align}
      \langle R(0),\, R'(0)\rangle
      = -\,\langle \partial_x F^*\, \xi,\; \phi\rangle
      \;+\; (u-t)\,\langle \partial_x F^*\, \xi,\; \partial_x h\, \xi\rangle.
    \end{align}
    For the first term, condition on $(t,u,\mbx_t)$:
    $\partial_x F^*$ and $\phi$ are deterministic, and
    $\E[\xi \mid \mbx_t] = 0$.
    Hence $\E[\langle \partial_x F^*\, \xi,\, \phi\rangle] = 0$ by the tower property.

    For the second term, we use the identity
    $\langle A\xi,\, B\xi\rangle = \xi^\top A^\top B\, \xi = \tr(A^\top B\, \xi\xi^\top)$.
    Conditioning on $(t,u,\mbx_t)$, the matrices $A = \partial_x F^*$ and
    $B = \partial_x h$ are deterministic and
    $\E[\xi\xi^\top \mid \mbx_t] = V(t, \mbx_t)$, so
    \begin{align}
      \E\big[(u-t)\,\langle \partial_x F^*\, \xi,\; \partial_x h\, \xi\rangle\big]
      &= \E\Big[(u-t)\,\tr\!\big((\partial_x F^*)^\top\, \partial_x h\, \E[\xi\xi^\top \mid \mbx_t]\big)\Big] \\
      &= \E\Big[(u-t)\,\tr\!\big(\partial_x h\, V(t,\mbx_t)\, (\partial_x F^*)^\top\big)\Big],
    \end{align}
    using cyclicity of the trace.
    Hence
    $\delta \cL_{\mathrm{nosg}}[f^*; h]
    = 2\,\E\big[(u-t)\,\tr(\partial_x h\, V\, (\partial_x F^*)^\top)\big]$.
\end{proof}

The key observation is that in standard flow matching, the integrand is linear in $\dot \mbx_t$, so replacing the marginal velocity $v$ with $\dot \mbx_t$ only adds a constant (the conditional variance) and preserves stationary points.
In MeanFlow without stopgrads, the term $\partial_x f\, \dot \mbx_t$ makes the integrand nonlinear in $\dot \mbx_t$: substituting $\dot \mbx_t$ for $v$ implicitly pulls an expectation outside a square and drops a conditional-covariance term, so $\cL_{\mathrm{nosg}}$ is a different objective from the one whose minimizer is $F^*$, and its full gradient faithfully optimizes that different objective (this covariance correction was also noted by \citet{boffi2024flowmapmatching} and derived in \citet{kim2026stabilizing}).
We emphasize that this is a property of the simplified objective, not of full gradients versus semi-gradients: full gradients of correctly specified losses are entirely sound.
The stopgrad happens to repair exactly this defect, because it prevents the interaction between $h$ and $\dot \mbx_t$ through $\partial_x f$ and restores the linearity that makes the replacement free; its stationary-point equation is the transport equation, whose solution is the minimizer of the properly specified objective with $v$ in place of $\dot \mbx_t$.

We now specialize to a 1D Gaussian to show the first variation is
explicitly nonzero. In one dimension, $V(t,\mbx_t)$ is a scalar and the trace
in \eqref{eq:mf_nosg_variation} reduces to ordinary multiplication.

\begin{tcolorbox}[boxrule=0pt, frame empty]
\begin{proposition}
    \label{prop:gaussian}
    Let $\mbx_0 \sim \cN(0, 1)$ and $\mbx_1 \sim \cN(\mu, \lambda^2)$
    with $\lambda > 0$.
    Let $\mbx_t = (1-t)\mbx_0 + t \mbx_1$ and $\dot \mbx_t = \mbx_1 - \mbx_0$.
    Then the true flow map $F^*$ is \emph{not} a stationary point
    of $\cL_{\mathrm{nosg}}$.
\end{proposition}
\end{tcolorbox}

The failure is caused by the posterior variance
$V(t) = \lambda^2/\sigma_t^2 > 0$.
When stopgrads are applied (as in the actual MeanFlow method),
the $\partial_x h \cdot \dot \mbx_t$ interaction is severed,
the $\xi^2$ cross-term disappears,
and by Corollary~\ref{cor:meanflow}, the true flow map becomes the unique
stationary point.

\begin{proof}[Proof of Proposition~\ref{prop:gaussian}]
    We recall the closed-form Gaussian quantities.
    The marginal variance is $\sigma_t^2 = (1-t)^2 + t^2 \lambda^2$.
    The velocity is $v(t,x) = \alpha(t)(x - t\mu) + \mu$ where $\alpha(t) = (t\lambda^2 - (1-t))/\sigma_t^2$.
    The flow map is $F^*(t,u,x) = m(t,u)(x - t\mu) + u\mu$ where $m(t,u) = \sigma_u / \sigma_t$.
    The Jacobian is $\partial_x F^*(t,u,x) = m(t,u)$.

    \medskip
    \noindent\textbf{Posterior variance.}
    We compute $V(t) := \Var(\dot \mbx_t \mid \mbx_t)$.
    Since $\dot \mbx_t = \mbx_1 - \mbx_0$ and $\mbx_t = (1-t)\mbx_0 + t\mbx_1$
    with $\mbx_0 \perp \mbx_1$:
    \begin{align}
      \Var(\dot \mbx_t) = 1 + \lambda^2, \qquad
      \Cov(\dot \mbx_t, \mbx_t) = t\lambda^2 - (1-t), \qquad
      \Var(\mbx_t) = \sigma_t^2.
    \end{align}
    By the Gaussian conditional variance formula:
    \begin{align}
      V(t)
      = \Var(\dot \mbx_t) - \frac{\Cov(\dot \mbx_t, \mbx_t)^2}{\Var(\mbx_t)}
      = (1 + \lambda^2) - \frac{(t\lambda^2 - (1-t))^2}{\sigma_t^2}
      = \frac{\lambda^2}{\sigma_t^2}.
    \end{align}
    (The last equality follows from expanding:
    $(1+\lambda^2)\sigma_t^2 - (t\lambda^2 - (1-t))^2 = \lambda^2$.)
    In particular, $V(t) > 0$ for all $t \in [0,1]$ when $\lambda > 0$.

    \medskip
    \noindent\textbf{Choice of admissible perturbation.}
    Admissible perturbations are $C^1_c$ functions, so we cannot use $h(t,u,x)=x$
    directly. Instead, fix a smooth cutoff $\chi \in C^\infty_c(\R)$ with
    $\chi \equiv 1$ on $[-1,1]$, $\chi \equiv 0$ outside $[-2,2]$, and
    $0 \le \chi \le 1$. For $R > 0$, define
    \begin{align}
      h_R(t,u,x) := x\, \chi(x/R).
    \end{align}
    Then $h_R \in C^1_c$ and
    \begin{align}
      \partial_x h_R(x)
      = \chi(x/R) + \tfrac{x}{R}\, \chi'(x/R).
    \end{align}
    Note $0 \le \partial_x h_R(x) \to 1$ pointwise as $R\to\infty$, and
    $|\partial_x h_R(x)| \le 1 + \|s\chi'(s)\|_\infty =: C < \infty$ uniformly in $R$.

    \medskip
    \noindent\textbf{Evaluating the first variation.}
    By Proposition~\ref{prop:mf_nosg_general}, the first variation at $F^*$
    in direction $h_R$ is
    \begin{align}
      \delta \cL_{\mathrm{nosg}}[f^*; h_R]
      = 2\lambda^2\,\E_{t,u,\mbx_t}\!\left[
        \frac{(u-t)\, m(t,u)}{\sigma_t^2}\, \partial_x h_R(\mbx_t)
      \right].
    \end{align}
 The integrand is bounded in absolute value by
$C\,(u-t)\, m(t,u)/\sigma_t^2$, which is bounded on $\{0 \le t \le u \le 1\}$
since $\sigma_t^2 = (1-t)^2 + t^2\lambda^2 \ge \lambda^2/(1+\lambda^2) > 0$.
By dominated convergence,
\begin{align}
  \lim_{R\to\infty} \delta \cL_{\mathrm{nosg}}[f^*; h_R]
  = 2\lambda^2\,\E_{t,u}\!\left[\frac{(u-t)\, m(t,u)}{\sigma_t^2}\right].
\end{align}
Since $\lambda > 0$, $(u-t) \ge 0$ on $\{t \le u\}$ with strict inequality
on the positive-measure set $\{t < u\}$, $m(t,u) = \sigma_u/\sigma_t > 0$,
and $\sigma_t > 0$, the integrand on the right-hand side is nonnegative and
strictly positive on a set of positive measure, so the limit is strictly
positive. Therefore there exists $R$ large enough that
$\delta \cL_{\mathrm{nosg}}[f^*; h_R] > 0$, with $h_R \in C^1_c$ admissible.
The first variation is nonzero on the admissible class, so $F^*$ is
\emph{not} a stationary point of $\cL_{\mathrm{nosg}}$.

\end{proof}

\section{Proofs for convergence of functional semi-gradient flow}
\label{appsec:convergence}

For our convergence arguments, we choose for clarity of exposition to work in a functional analytic setting where $g: \mathcal{M} \to \mathcal{M}$, such that semi-gradient flow on $f \in \mathcal{M}$ stays in $\mathcal{M}$.
This is more narrow than the functional analytic setting introduced in \S\ref{sec:stopgrad_formalism}.
Sufficient conditions for $g: \mathcal{M} \to \mathcal{M}$ are: $\mathcal{M} \subset C^{\infty}_{\text{pol}} \cap H$ where $C^{\infty}_{\text{pol}}$ is the class of smooth functions such that every derivative has at most polynomial growth in $x$ uniformly over $t,u$, and $H$ is $L^2(\mu)$; $v \in C^{\infty}_{\text{pol}}$ and globally Lipschitz in $x$ uniformly in time, and $\mu$ has all polynomial moments.

These assumptions cover a broad class of flow maps and models of interest, such as linear interpolants between smooth positive data densities and a Gaussian base, and neural networks with activations with bounded derivatives. Nevertheless, $g: \mathcal{M} \to \mathcal{M}$ is chosen for clarity of exposition: its assumptions are stricter than absolutely necessary for the convergence arguments, and could be weakened.

\subsection{Form of functional semi-gradient flows}
\label{appsec:form_of_functional_gradient_flows}

We use the parameterization relationship $F(t,u,x)=x+(u-t)f(t,u,x)$.
The residuals from applying Theorem~\ref{thm:principle} to each objective are:
\begin{align}
    g_{\mathrm{MF}}[f](t,u,x)
        &=
        f(t,u,x)
        -
        v(t,x)
        -
        (u-t) \Bigl(
            \partial_t f(t,u,x)
            +(\partial_x f)(t,u,x)\,v(t,x)
        \Bigr)
        \label{eq:g-mf-setup} \\
    g_{\mathrm{sESD}}[f](t,u,x)
    &=
    f(t,u,x)
    -(u-t)\,\partial_t f(t,u,x)
    -(\partial_x F)(t,u,x)\,f(t,t,x)
    \label{eq:g-esd-setup}
    \\
    g_{\mathrm{sLSD}}[f](t,u,x)
    &=
    f(t,u,x)
    +(u-t)\,\partial_u f(t,u,x)
    -f \bigl(u,u,F(t,u,x)\bigr)
    \label{eq:g-lsd-setup}
\end{align}
and recall that the semi-gradient flows are, for each $g$: $\partial_\tau f_\tau = -g[f_\tau]$.
Consider the Eulerian and Lagrangian residuals, which are the residual errors for the standard backward equation and forward equation for a flow map, respectively.
\begin{tcolorbox}[boxrule=0pt, frame empty]
\begin{align}
    E_v[f](t,u,x)
    &:=
    \partial_t F(t,u,x)+(\partial_x F)(t,u,x)\,v(t,x),
    \tag{Eulerian}
    \label{eq:Ev-def-unified}
    \\
    L_v[f](t,u,x)
    &:=
    \partial_u F(t,u,x)-v \bigl(u,F(t,u,x)\bigr).
    \tag{Lagrangian}
    \label{eq:Rv-def-unified}
\end{align}    
\end{tcolorbox}

We will show that
\begin{equation}
    g_{\mathrm{MF}}[f](t,u,x) = -\,E_v[f](t,u,x).
\end{equation}
Furthermore, assume perfect flow-matching pretraining. At a high level, this will let us replace model instantaneous velocities $f(t,t,x)$ with the true marginal velocity $v$.
Then, when post-training with slim \gls{esd} and \gls{lsd}, we will show that:
\begin{align}
    g_{\mathrm{sESD}}[f](t,u,x) &= -\,E_v[f](t,u,x), \\
    g_{\mathrm{sLSD}}[f](t,u,x) &= L_v[f](t,u,x).
\end{align}

Since $\partial_\tau F_\tau = \partial_\tau  \bigl( x + (u-t) f_\tau \bigr) = (u-t)\,\partial_\tau f_\tau$,
and \(\partial_\tau f_\tau=-g[f_\tau]\), we get $\partial_\tau F_\tau = -(u-t)\,g[f_\tau]$.
The learned flow map $F$ evolves over semi-gradient flow optimization time $\tau$ as:
\begin{tcolorbox}[boxrule=0pt, frame empty]
\begin{align}
    \text{Eulerian (MF, s-\textsc{esd}):}\qquad \partial_\tau F_\tau &= (u-t)\,E_v[f_\tau], \label{eq:mf-f-flow-unified} \\
    \text{Lagrangian (s-\textsc{lsd}):}\qquad \partial_\tau F_\tau &= -(u-t)\,L_v[f_\tau]. \label{eq:lsd-f-flow-unified}
\end{align}
\end{tcolorbox}

\paragraph{Derivation.}
To perform this rewriting for slim \gls{esd} and \gls{lsd}, we need to handle the fact that they are distillation-like objectives which use the model's own prediction $f(t,t,x)$ instead of $\dot{\mathbf{x}}_t$.
We will replace these diagonal predictions with $v$ by assuming perfect flow-matching pretraining, such that 
\begin{equation}
    f_0(t,t,x)=v(t,x)
    \qquad\text{for all }(t,x).
    \label{eq:perfect-fm-pretrain-unified}
\end{equation}
By noting a diagonal invariance property of slim \gls{esd} and \gls{lsd} (Lemma \ref{lem:diag-invariance-esd-lsd-unified}), we will obtain that during post-training:
\[
f_\tau(t,t,x)=v(t,x) \qquad\text{for all }(\tau,t,x).
\]
In particular the semi-gradient of the flow matching term, $f_\tau(t,t,\cdot) - v(t,\cdot)$, vanishes identically along the flow, so the semi-gradient flow of the combined objective $\cL_{\mathrm{SD}} + \cL_{\mathrm{FM}}$ coincides with that of the stopgrad term alone; this is why only $g_{\mathrm{sESD}}$ and $g_{\mathrm{sLSD}}$ appear below.
MeanFlow does not require this assumption because it uses $\dot{\mathbf{x}}_t$, so that under Theorem~\ref{thm:principle}, its residual \eqref{eq:g-mf-setup} is already in terms of $v$, as desired.

\begin{tcolorbox}[boxrule=0pt, frame empty]
\begin{lemma}[Diagonal invariance for slim \gls{esd} and slim \gls{lsd}]
    \label{lem:diag-invariance-esd-lsd-unified}
    For every $f \in \mathcal{M}$, the slim residuals \eqref{eq:g-esd-setup} and \eqref{eq:g-lsd-setup} vanish on the diagonal: $g_{\mathrm{sESD}}[f](t,t,x) = g_{\mathrm{sLSD}}[f](t,t,x) = 0$ for all $(t,x)$.
    Consequently, along semi-gradient flows $\partial_\tau f_\tau=-g_{\mathrm{sESD}}[f_\tau]$ or $\partial_\tau f_\tau=-g_{\mathrm{sLSD}}[f_\tau]$,
    the diagonal values do not change: $\partial_\tau f_\tau(t,t,x)=0$.
    Furthermore, Under the perfect-FM-pretraining assumption
    \eqref{eq:perfect-fm-pretrain-unified},
    \begin{equation}
        f_\tau(t,t,x)=v(t,x)
        \qquad\text{for all }(\tau,t,x).
        \label{eq:diag-frozen-v-unified}
    \end{equation}
\end{lemma}
\end{tcolorbox}

\begin{proof}
    For slim \gls{esd}, set \(u=t\) in \eqref{eq:g-esd-setup}. Since the prefactor
    \(u-t\) vanishes and \(F(t,t,x)=x\), we have $(\partial_x F)(t,t,x)=I$,
    because \(F(t,t,x)=x\) identically in \(x\). Therefore $g_{\mathrm{sESD}}[f](t,t,x) = f(t,t,x)-I\,f(t,t,x) = 0$.

    For slim \gls{lsd}, set \(u=t\) in \eqref{eq:g-lsd-setup}. The prefactor
    \(u-t\) vanishes, and since \(F(t,t,x)=x\), $g_{\mathrm{sLSD}}[f](t,t,x)
        = f(t,t,x)-f \bigl(t,t,F(t,t,x)\bigr)
        = f(t,t,x)-f(t,t,x)
        = 0$.
    Therefore, in either post-training flow, $\partial_\tau f_\tau(t,t,x) = -g[f_\tau](t,t,x) = 0$, so the diagonal values are invariant in \(\tau\).
\end{proof}

We now rewrite the post-training residual for each method in terms of the
Eulerian transport residual \(E_v[f]\) or the Lagrangian residual \(L_v[f]\).

\begin{tcolorbox}[boxrule=0pt, frame empty]
\begin{proposition}[Residual forms in \(f\)-coordinates]
    \label{prop:residual-forms-unified}

    For slim \gls{esd} and \gls{lsd}, assume perfect flow-matching pretraining, so that \eqref{eq:diag-frozen-v-unified} holds.
    Then, we have:
    \begin{align}
        g_{\mathrm{MF}}[f](t,u,x) &= -\,E_v[f](t,u,x), \label{eq:mf-equals-minus-Ev} \\
        g_{\mathrm{sESD}}[f](t,u,x) &= -\,E_v[f](t,u,x), \label{eq:esd-equals-minus-Ev} \\
        g_{\mathrm{sLSD}}[f](t,u,x) &= L_v[f](t,u,x). \label{eq:lsd-equals-Rv}
    \end{align}
\end{proposition}
\end{tcolorbox}

\begin{proof}
    We repeatedly use the parametrization $F(t,u,x)=x+(u-t)f(t,u,x)$
    whose derivatives are
    \begin{align}
        \partial_t F(t,u,x)
        &=
        -f(t,u,x)
        +(u-t)\,\partial_t f(t,u,x),
        \label{eq:ft-deriv-unified}
        \tag{$\partial_t F$}
        \\
        \partial_u F(t,u,x)
        &=
        f(t,u,x)
        +(u-t)\,\partial_u f(t,u,x),
        \label{eq:fu-deriv-unified}
        \tag{$\partial_u F$}
        \\
        \partial_x F(t,u,x)
        &=
        I+(u-t)\,\partial_x f(t,u,x).
        \label{eq:fx-deriv-unified}
        \tag{$\partial_x F$}
    \end{align}

    \noindent
    \textbf{(a) MeanFlow.}
    By definition,
    \begin{align}
        E_v[f](t,u,x)
        &=
        \partial_t F(t,u,x)+(\partial_x F)(t,u,x)\,v(t,x). \notag
    \end{align}
    Substitute \eqref{eq:ft-deriv-unified} and \eqref{eq:fx-deriv-unified}:
    \begin{align}
        E_v[f](t,u,x)
        &=
         \Bigl(
            -f(t,u,x) +(u-t)\,\partial_t f(t,u,x)
        \Bigr)
        +
         \Bigl(
            I+(u-t)\,\partial_x f(t,u,x)
        \Bigr)v(t,x)
        \notag\\
        &=
        -f(t,u,x)
        +(u-t)\,\partial_t f(t,u,x)
        +v(t,x)
        +(u-t)(\partial_x f)(t,u,x)\,v(t,x)
        \notag\\
        &=
        - \Bigl(
            f(t,u,x)
            -v(t,x)
            -(u-t) \bigl(
                \partial_t f(t,u,x)
                +(\partial_x f)(t,u,x)\,v(t,x)
            \bigr)
        \Bigr)
        \notag\\
        &=
        -\,g_{\mathrm{MF}}[f](t,u,x). \notag
    \end{align}

    \noindent
    \textbf{(b) Slim \gls{esd}.}
    By diagonal invariance and perfect FM pretraining, $f(u,u,y)=v(u,y)$ for all $(u,y)$ throughout post-training.
    Therefore
    \begin{align}
        g_{\mathrm{sESD}}[f](t,u,x)
        &=
        f(t,u,x)
        -(u-t)\,\partial_t f(t,u,x)
        -(\partial_x F)(t,u,x)\,v(t,x). \notag
    \end{align}
    Using \eqref{eq:fx-deriv-unified},
    \begin{align}
        (\partial_x F)(t,u,x)\,v(t,x)
        &=
         \Bigl(I+(u-t)\,\partial_x f(t,u,x)\Bigr)v(t,x)
        \notag\\
        &=
        v(t,x)+(u-t)(\partial_x f)(t,u,x)\,v(t,x). \notag
    \end{align}
    Hence
    \begin{align}
        g_{\mathrm{sESD}}[f](t,u,x)
        &=
        f(t,u,x)
        -v(t,x)
        -(u-t) \Bigl(
            \partial_t f(t,u,x)
            +(\partial_x f)(t,u,x)\,v(t,x)
        \Bigr)
        \notag\\
        &=
        g_{\mathrm{MF}}[f](t,u,x) =
        -\,E_v[f](t,u,x).
        \notag
    \end{align}

    \noindent
    \textbf{(c) Slim \gls{lsd}.}
    By diagonal invariance and perfect FM pretraining, $f(u,u,y)=v(u,y)$ for all $(u,y)$ throughout post-training.
    Therefore
    \begin{align}
        g_{\mathrm{sLSD}}[f](t,u,x)
        &=
        f(t,u,x)
        +(u-t)\,\partial_u f(t,u,x)
        -v \bigl(u,F(t,u,x)\bigr). \notag
    \end{align}
    Using \eqref{eq:fu-deriv-unified}, we combine the first two terms on the right-hand side:
    \begin{align}
        g_{\mathrm{sLSD}}[f](t,u,x)
        &=
        \partial_u F(t,u,x)-v \bigl(u,F(t,u,x)\bigr) =
        L_v[f](t,u,x). \notag
    \end{align}
\end{proof}

\subsection{Proof for proposition~\ref{prop:rectifying-templates-unified}}
\label{appsec:convergence_proof}
We prove the Eulerian and Lagrangian cases separately, though the arguments are similar.
The key idea is introducing a combined map that composes the learned flow with the true flow map.
The core insight is the combined map obeys a transport equation.
In general, a transport equation says that a function $f(t,x)$ is conserved over time $t$ along trajectories of $x$ following a velocity field $v$.
This transport equation, wherein a certain quantity is conserved over optimization time $\tau$, lets us characterize the semi-gradient flow evolution of the combined map, and then of the learned flow map.

The proof uses backwards evaluations of the true flow map, which exist and are correct under the standing regularity conditions, but does not evaluate $F_\tau$ outside $t \leq u$.

\begin{proof}
    \noindent
    \textbf{(a) Eulerian case.}
    Define a ``double-forward'' combined map
    \[
        a_\tau(t,u,z) := F_\tau \bigl(t,u,F^*(0,t,z)\bigr).    
    \]
    This combined map says: first, starting from $z$, transport it with the true flow map from time $0 \to t$, then take the output and transport it with the learned flow map from $t \to u$.

    Differentiate with respect to \(t\). By the chain rule for the total derivative,
    \begin{align*}
        \partial_t a_\tau(t,u,z)
        &=
        \partial_1 F_\tau \bigl(t,u,F^*(0,t,z)\bigr)
        +
        (\partial_3 F_\tau) \bigl(t,u,F^*(0,t,z)\bigr)\,
        \partial_t F^*(0,t,z).
    \end{align*}
    where we use $\partial_1$ and $\partial_3$ to denote the partial derivative with respect to first, and third, input variables for $F_\tau$.
    
    Since \(F^*(0,t,z)\) is the exact flow of \(v\), we have: $\partial_t F^*(0,t,z)=v \bigl(t,F^*(0,t,z)\bigr)$.
    Using this:
    \begin{align}
        \partial_t a_\tau(t,u,z)
        &=
        \partial_1 F_\tau \bigl(t,u,F^*(0,t,z)\bigr)
        +
        (\partial_3 F_\tau) \bigl(t,u,F^*(0,t,z)\bigr)\,
        v \bigl(t,F^*(0,t,z)\bigr)
        \notag\\
        &=
        E_v[f_\tau] \bigl(t,u,F^*(0,t,z)\bigr).
        \label{eq:dt-a-is-Ev-unified}
    \end{align}
    Next differentiate \(a_\tau\) with respect to \(\tau\). Since \(F^*(0,t,z)\)
    is independent of \(\tau\),
    \begin{align*}
        \partial_\tau a_\tau(t,u,z)
        &=
        \partial_\tau F_\tau \bigl(t,u,F^*(0,t,z)\bigr).
    \end{align*}
    Using $\partial_\tau F_\tau = (u-t)\,E_v[f_\tau]$,
    \begin{align}
        \partial_\tau a_\tau(t,u,z)
        &=
        (u-t)\,E_v[f_\tau] \bigl(t,u,F^*(0,t,z)\bigr)
        \notag\\
        &=
        (u-t)\,\partial_t a_\tau(t,u,z),
        \label{eq:a-linear-pde-unified}
    \end{align}
    where the last step uses \eqref{eq:dt-a-is-Ev-unified}.
    
    Now, we will rewrite the \gls{pde} in the gap variable \(q:=u-t\). Define
    \[
        c_\tau(q,u,z):=a_\tau(u-q,u,z) = a_\tau(t,u,z).
    \]
    Then,
    \[
        \partial_q c_\tau(q,u,z)
        = \partial_q a_\tau(t,u,z) \left( \frac{\partial t}{\partial q} \right)
        = -\partial_t a_\tau(t,u,z).
    \]
    Therefore \eqref{eq:a-linear-pde-unified} becomes
    \begin{align}
        \partial_\tau a_\tau(t,u,z) &= (u-t)\,\partial_t a_\tau(t,u,z) \notag \\
        \partial_\tau c_\tau(q,u,z) -q\,\partial_t a_\tau(t,u,z) &= 0  \notag \\
        \partial_\tau c_\tau(q,u,z) +q\,\partial_q c_\tau(q,u,z) &= 0 \label{eq:c-linear-pde-unified}
    \end{align}

    We recognize \eqref{eq:c-linear-pde-unified} as a transport equation in the
    gap variable \(q=u-t\):
    \[
    \partial_\tau c_\tau(q,u,z)+q\,\partial_q c_\tau(q,u,z)=0.
    \]
    In general, a transport equation $\partial_t F(t,x) + v(t,x) \partial_x F(t,x) = 0$ says that $F(t,x)$ is conserved over time $t$ along trajectories of $x$ following $dx/dt = v$.
    In our case, we have that $c_\tau(q,u,z)$ is conserved over time $\tau$ along trajectories of $q$ following $dq/d\tau = q$. 
    Its forward characteristics satisfy
    \[
    \frac{dq}{d\tau}=q,
    \qquad
    q(\tau)=e^\tau q_0,
    \]
    and \(c_\tau\) is constant along these curves:
    \[
    c_\tau(q(\tau),u,z)=c_0(q_0,u,z).
    \]
    Equivalently, to evaluate the solution at a fixed point \((\tau,q)\), we trace
    the characteristic backward to the initial line \(\tau=0\), which gives
    \[
    q_0=e^{-\tau}q.
    \]
    Hence
    \[
    c_\tau(q,u,z)=c_0(e^{-\tau}q,u,z).
    \]
    Thus, $c_\tau(q,u,z)$ is exponentially compressed toward the value $c_0(q=0, u, z)$ as \(\tau\to\infty\).
    Returning to \((t,u)\),
    \begin{align}
        a_\tau(t,u,z)
        &=
        c_\tau(u-t,u,z)
        \notag\\
        &=
        c_0 \bigl(e^{-\tau}(u-t),u,z\bigr)
        \notag\\
        &=
        a_0 \bigl(u-e^{-\tau}(u-t),u,z\bigr)
        \notag
    \end{align}
    Thus, semi-gradient flow evolves $a$ as:
    \begin{equation}
        a_\tau(t,u,z) = a_0 \bigl(s^\tau_{t{\shortrightarrow}u},u,z\bigr) \label{eq:a-explicit-unified}.
    \end{equation}

    To recover \(F_\tau\), start from the definition of $a_\tau$, then set \(z=F^*(t,0,x)\), equivalently \(x=F^*(0,t,z)\).
    Use the composition property of the true flow map: $F^*(s,t,F^*(0,s,x)) = F^*(0,t,x)$.
    \begin{align}
        F_\tau(t,u, F^*(0,t,z)) &= a_\tau(t,u,z)
        \notag\\
        F_\tau(t,u, F^*(0,t,F^*(t,0,x)))  &= a_\tau(t,u,F^*(t,0,x))
        \notag\\
        F_\tau(t,u,x) &= a_\tau \bigl(t,u,F^*(t,0,x)\bigr)
        \notag\\
        &=
        a_0 \bigl(s^\tau_{t{\shortrightarrow}u},u,F^*(t,0,x)\bigr)
        \notag\\
        &=
        F_0 \Bigl(
            s^\tau_{t{\shortrightarrow}u},
            u,
            F^*(0,s^\tau_{t{\shortrightarrow}u},F^*(t,0,x))
        \Bigr)
        \notag
    \end{align}
    Thus:
    \begin{equation}
        F_\tau(t,u,x) = F_0 \bigl(
            s^\tau_{t{\shortrightarrow}u},
            u,
            F^*(t,s^\tau_{t{\shortrightarrow}u},x)
        \bigr).
    \end{equation}
    This proves \eqref{eq:f-explicit-eulerian-unified}.
    Finally, as \(\tau\to\infty\), \(s^\tau_{t{\shortrightarrow}u}\to u\).
    Therefore
    \[
        F^*(t,s^\tau_{t{\shortrightarrow}u},x)\to F^*(t,u,x).
    \]
    Using continuity of \(F_0\) and the diagonal identity \(F_0(u,u,y)=y\),
    \begin{align}
        F_\tau(t,u,x)
        &=
        F_0 \bigl(
            s^\tau_{t{\shortrightarrow}u},
            u,
            F^*(t,s^\tau_{t{\shortrightarrow}u},x)
        \bigr)
        \to
        F_0 \bigl(u,u,F^*(t,u,x)\bigr)
        \notag =
        F^*(t,u,x).
    \end{align}

    \smallskip
    \noindent
    \textbf{(b) Lagrangian case.}
    Define a ``backtracking'' combined map:
    \[
        b_\tau(t,u,x) := F^*(u,t,F_\tau(t,u,x)).
    \]
    This combined map says: first, transport $x$ using the learned flow map from $t \to u$, then take the output and transport it with the true flow map back from $u \to t$.
    
    We first compute \(\partial_u b_\tau\).
    We use the transport equation obeyed by the true flow map:
    \begin{equation}
        \partial_u F^*(u,t,z) + \partial_z F^*(u,t,z)\,v(u,z) = 0.
        \label{eq:du-backward-flow-unified}
    \end{equation}

    Differentiate \(b_\tau\) with respect to \(u\). By the chain rule for the total derivative, \eqref{eq:du-backward-flow-unified}, and using $\partial_1$ and $\partial_3$ to denote the partial derivative with respect to the first and third input variables:
    \begin{align}
        \partial_u b_\tau(t,u,x)
        &=
         \partial_1 F^* \bigl(u, t, F_\tau(t,u,x) \bigr) 
        +
        \partial_3 F^*(u,t,F_\tau(t,u,x))\,
        \partial_u F_\tau(t,u,x)
        \notag\\
        &=
        - \partial_3 F^*(u,t,F_\tau(t,u,x))\,
        v \bigl(u,F_\tau(t,u,x)\bigr)
        +
        \partial_3 F^* \bigl( u,t,F_\tau(t,u,x) \bigr)\,
        \partial_u F_\tau(t,u,x)
        \notag\\
        &=
        \partial_3 F^*(u,t,F_\tau(t,u,x))\,
         \Bigl(
            \partial_u F_\tau(t,u,x)-v \bigl(u,F_\tau(t,u,x)\bigr)
        \Bigr)
        \notag\\
        &=
        \partial_3 F^*(u,t,F_\tau(t,u,x))\,
        L_v[f_\tau](t,u,x).
        \label{eq:du-b-is-Rv-unified}
    \end{align}

    Differentiate \(b_\tau\) with respect to \(\tau\). Since \(F^*(u,t,\cdot)\) is
    independent of \(\tau\),
    \begin{align}
        \partial_\tau b_\tau(t,u,x)
        &=
        \partial_3 F^*(u,t,F_\tau(t,u,x))\,
        \partial_\tau F_\tau(t,u,x).
    \end{align}
    Using $\partial_\tau F_\tau = -(u-t)\,L_v[f_\tau]$, 
    \begin{align}
        \partial_\tau b_\tau(t,u,x)
        &=
        -(u-t)\,
        \partial_3 F^*(u,t,F_\tau(t,u,x))\,
        L_v[f_\tau](t,u,x)
        \notag\\
        &=
        -(u-t)\,\partial_u b_\tau(t,u,x). \label{eq:b-linear-pde-unified}
    \end{align}
    where the last step uses \eqref{eq:du-b-is-Rv-unified}.

    To rewrite the \gls{pde} in the gap variable \(q=u-t\), define
    \[
        d_\tau(t,q,x):= b_\tau(t,t+q,x) = b_\tau(t,u,x).
    \]
    Then
    \[
        \partial_q d_\tau(t,q,x)
        = \partial_u b_\tau(t,t+q,x) \left( \frac{\partial q}{\partial u} \right)
        = 
        \partial_u b_\tau(t,t+q,x),
    \]
    so \eqref{eq:b-linear-pde-unified} becomes
    \begin{align}
        \partial_\tau b_\tau(t,u,x) &= -(u-t)\,\partial_u b_\tau(t,u,x)
        \notag \\
        \partial_\tau b_\tau(t,u,x) + q\,\partial_u b_\tau(t,u,x) &= 0
        \notag \\
        \partial_\tau d_\tau(t,q,x)+q\,\partial_q d_\tau(t,q,x) &= 0 \label{eq:d-linear-pde-unified}.        
    \end{align}

    We recognize this as a transport equation.
    In general, a transport equation $\partial_t F(t,x) + v(t,x) \partial_x F(t,x) = 0$ says that $F(t,x)$ is conserved over time $t$ along trajectories of $x$ following $dx/dt = v$.
    In our case, we have that $d_\tau(t,q,x)$ is conserved over time $\tau$ along trajectories of $q$ following $dq/d\tau = q$. 

    Solve this exactly as in the Eulerian case:
    \[
        d_\tau(t,q,x)=d_0 \bigl(t,e^{-\tau}q,x\bigr).
    \]
    Therefore
    \begin{align}
        b_\tau(t,u,x)
        &=
        d_\tau(t,u-t,x)
        \notag\\
        &=
        d_0 \bigl(t,e^{-\tau}(u-t),x\bigr)
        \notag\\
        &=
        b_0 \bigl(t,t+e^{-\tau}(u-t),x\bigr)
        \notag
    \end{align}
    Thus, semi-gradient flow evolves $b$ as:
    \begin{equation}
        b_\tau(t,u,x) = b_0 \bigl(t,s^\tau_{u{\shortrightarrow}t},x\bigr) \label{eq:b-explicit-unified}.
    \end{equation}

    To recover \(F_\tau\), apply \(F^*(t,u,\cdot)\) to both sides:
    \begin{align}
        F^*(u, t, F_\tau(t,u,x)) &= b_\tau(t,u,x)
        \notag\\
        F_\tau(t,u,x) &= F^*(t,u,b_\tau(t,u,x))
        \notag\\
        &=
        F^*(t,u,b_0(t,s^\tau_{u{\shortrightarrow}t},x))
        \notag\\
        &=
        F^*(t,u,F^*(s^\tau_{u{\shortrightarrow}t}, t, F_0(t,s^\tau_{u{\shortrightarrow}t},x)))
        \notag
    \end{align}
    Thus:
    \begin{equation}
        F_\tau(t,u,x) = F^*(s^\tau_{u{\shortrightarrow}t}, u, F_0(t,s^\tau_{u{\shortrightarrow}t},x)),
    \end{equation}
    by the semigroup property \(F^*(t,u,\cdot)\circ F^*(s,t,\cdot) = F^*(s,u,\cdot)\).
    This proves \eqref{eq:f-explicit-lagrangian-unified}.

    Finally, as \(\tau\to\infty\), \(s^\tau_{u{\shortrightarrow}t}\to t\). By continuity of
    \(F_0\) and the diagonal identity \(F_0(t,t,x)=x\),
    \begin{align}
        F_\tau(t,u,x)
        &=
        F^*(s^\tau_{u{\shortrightarrow}t}, u, F_0(t,s^\tau_{u{\shortrightarrow}t},x))
        \to
        F^*(t,u,F_0(t,t,x)) = F^*(t,u,x).
    \end{align}
\end{proof}

\subsection{Exponential convergence rates under bounded derivatives.}
\label{appsec:exponential_convergence}

In this section, we prove that convergence occurs exponentially fast when the true velocity field has bounded derivatives.
Before presenting the main proposition for exponential convergence, we first provide a lemma that we will use in the Lagrangian case.

\begin{tcolorbox}[boxrule=0pt, frame empty]
\begin{lemma}
    \label{lem:flowmap-lipschitz}
    Suppose that the true velocity field \(v\) is globally \(L\)-Lipschitz in \(x\), uniformly in time (e.g., by the standing regularity conditions, Remark~\ref{rmk:regularity}):
    \begin{equation}
        \|v(s,x)-v(s,y)\|
        \le
        L\,\|x-y\|
        \qquad
        \text{for all } s\in[0,1],\ x,y\in\bbR^d.
        \label{eq:v-global-Lipschitz-rate-prop}
    \end{equation}
    Then the exact flow map \(F^*(t,u,\cdot)\) is \(e^{L(u-t)}\)-Lipschitz in its
    initial condition.
\end{lemma}
\end{tcolorbox}

\begin{proof}
    Let \(y_1,y_2\in\bbR^d\), and consider two trajectories $Z_i(s):= F^*(t, s,y_i)$ for $i=1,2$.
    By definition of the exact flow, each \(Z_i\) solves $\frac{d}{ds}Z_i(s)=v\bigl(s,Z_i(s)\bigr)$ with $Z_i(t)=y_i$.
    Subtract the two ODEs:
    \[
        \frac{d}{ds}\bigl(Z_1(s)-Z_2(s)\bigr)
        =
        v\bigl(s,Z_1(s)\bigr)-v\bigl(s,Z_2(s)\bigr).
    \]
    Taking norms and using the global Lipschitz bound
    \eqref{eq:v-global-Lipschitz-rate-prop},
    \[
        \frac{d}{ds}\|Z_1(s)-Z_2(s)\|
        \le
        L\,\|Z_1(s)-Z_2(s)\|.
    \]
    Gr\"onwall's inequality then yields
    \[
        \|Z_1(u)-Z_2(u)\|
        \le
        e^{L(u-t)}\|Z_1(t)-Z_2(t)\|
        =
        e^{L(u-t)}\|y_1-y_2\|.
    \]
    Since \(Z_i(u) =F^*(t,u,y_i)\), we conclude that
    \begin{equation}
        \|F^*(t,u,y_1)-F^*(t,u,y_2)\|
        \le
        e^{L(u-t)}\|y_1-y_2\|.
        \label{eq:fstar-Lipschitz-explicit}
    \end{equation}
\end{proof}

\begin{tcolorbox}[boxrule=0pt, frame empty]
\begin{proposition}[Exponential convergence rates]
    \label{prop:exp-rates-unified}
    Assume the hypotheses of
    Proposition~\ref{prop:rectifying-templates-unified}.
    Assume in addition that the true velocity field \(v\) is globally
    \(L\)-Lipschitz in \(x\), uniformly in time, such that the exact flow map \(F^*(t,u,\cdot)\) is \(e^{L(u-t)}\)-Lipschitz in its
    initial condition.

    \begin{enumerate}
        \item[\textbf{(a)}] \textbf{Eulerian case} (MeanFlow and simplified \gls{esd}).
        Assume
        \begin{equation}
            M_a
            :=
            \sup_{0\le t\le u\le 1}\sup_{z\in\bbR^d}
            \bigl\|\partial_t a_0(t,u,z)\bigr\|
            <\infty.
            \label{eq:Ma-def-unified}
        \end{equation}
        Then
        \begin{equation}
            \|F_\tau(t,u,x)-F^*(t,u,x)\|
            \le
            M_a\,e^{-\tau}(u-t).
            \label{eq:eulerian-rate-map-unified}
        \end{equation}

        \item[\textbf{(b)}] \textbf{Lagrangian case} (simplified \gls{lsd}).
        Assume
        \begin{equation}
            M_b
            :=
            \sup_{0\le t\le u\le 1}\sup_{x\in\bbR^d}
            \bigl\|\partial_u b_0(t,u,x)\bigr\|
            <\infty.
            \label{eq:Mb-def-unified}
        \end{equation}
        Then
        \begin{equation}
            \|F_\tau(t,u,x)-F^*(t,u,x)\|
            \le
            e^{L(u-t)}\,M_b\,e^{-\tau}(u-t).
            \label{eq:lagrangian-rate-map-unified}
        \end{equation}
    \end{enumerate}
\end{proposition}
\end{tcolorbox}

\begin{proof}
    We again treat the Eulerian and Lagrangian cases separately.

    \medskip
    \noindent
    \textbf{(a) Eulerian case.}

    By \eqref{eq:a-explicit-unified}, $a_\tau(t,u,z)=a_0\bigl(s^\tau_{t{\shortrightarrow}u},u,z\bigr)$.
    Also, $a_0(u,u,z)
        =
        F_0\bigl(u,u,F^*(0,u,z)\bigr)
        =
        F^*(0,u,z)$,
    because every admissible flow-map parametrization satisfies \(F_0(u,u,y)=y\).

    We now compare \(a_\tau(t,u,z)\) and \(a_0(u,u,z)\). The point is that, for
    fixed \(u\) and \(z\), the map $\sigma \mapsto a_0(\sigma,u,z)$ is a \(C^1\) function of one real variable on the interval \([0,u]\).
    Therefore the fundamental theorem of calculus gives
    \begin{align}
        a_0\bigl(s^\tau_{t{\shortrightarrow}u},u,z\bigr)-a_0(u,u,z)
        &=
        -\int_{s^\tau_{t{\shortrightarrow}u}}^{u}
        \partial_1 a_0(\sigma,u,z) \,d\sigma.
        \label{eq:FTC-eulerian-step}
    \end{align}
    Taking norms and using the uniform bound
    \(\|\partial_t a_0(t,u,z)\|\le M_a\),
    \begin{align}
        \|a_\tau(t,u,z)-F^*(0,u,z)\|
        &=
        \bigl\|
            a_0\bigl(s^\tau_{t{\shortrightarrow}u},u,z\bigr)-a_0(u,u,z)
        \bigr\|
        \notag\\
        &\le
        \int_{s^\tau_{t{\shortrightarrow}u}}^u
        \|\partial_1 a_0(\sigma,u,z)\|\,d\sigma
        \notag\\
        &\le
        \int_{s^\tau_{t{\shortrightarrow}u}}^u M_a\,d\sigma
        \notag\\
        &=
        M_a\,\bigl(u-s^\tau_{t{\shortrightarrow}u}\bigr).
        \label{eq:eulerian-bound-before-gap}
    \end{align}
    Using $s^\tau_{t{\shortrightarrow}u} = u-e^{-\tau}(u-t)$,
    so $u-s^\tau_{t{\shortrightarrow}u} = e^{-\tau}(u-t)$, we get:
    \begin{equation}
        \|a_\tau(t,u,z)-F^*(0,u,z)\|
        \le
        M_a\,e^{-\tau}(u-t).
        \label{eq:a-rate-step-unified}
    \end{equation}
    We now convert this bound back to the original \((t,u,x)\)-variables, by setting $z=F^*(t,0,x)$.
    Since \(a_\tau\) was defined by $a_\tau(t,u,z)=F_\tau\bigl(t,u,F^*(0,t,z)\bigr),$
    and since \(F^*(0,t,F^*(t,0,x))=x\), we obtain $a_\tau\bigl(t,u,F^*(t,0,x)\bigr)=F_\tau(t,u,x)$.
    Similarly, $F^*(0,u,F^*(t,0,x))=F^*(t,u,x)$
    by the semigroup property of the true flow. Therefore,
    \eqref{eq:a-rate-step-unified} becomes
    \[
        \|F_\tau(t,u,x)-F^*(t,u,x)\|
        \le
        M_a\,e^{-\tau}(u-t),
    \]
    which is exactly \eqref{eq:eulerian-rate-map-unified}.

    \medskip
    \noindent
    \textbf{(b) Lagrangian case.}

    By \eqref{eq:b-explicit-unified}, $b_\tau(t,u,x)=b_0\bigl(t,s^\tau_{u{\shortrightarrow}t},x\bigr).$
    Also, $b_0(t,t,x)=F^*(t,t,F_0(t,t,x))=x$.
    As in the Eulerian case, we now make the one-variable estimate explicit.
    For fixed \(t\) and \(x\), the map $\sigma \mapsto b_0(t,\sigma,x)$ is a \(C^1\) function of one real variable on the interval \([t,1]\).
    Therefore the fundamental theorem of calculus gives
    \begin{align}
        b_0\bigl(t,s^\tau_{u{\shortrightarrow}t},x\bigr)-b_0(t,t,x)
        &=
        \int_t^{s^\tau_{u{\shortrightarrow}t}}
        \partial_2 b_0(t,\sigma,x) \,d\sigma.
        \label{eq:FTC-lagrangian-step}
    \end{align}
    Taking norms and using the uniform bound \(\|\partial_u b_0(t,u,x)\|\le M_b\),
    \begin{align}
        \|b_\tau(t,u,x)-x\|
        &=
        \bigl\|
            b_0\bigl(t,s^\tau_{u{\shortrightarrow}t},x\bigr)-b_0(t,t,x)
        \bigr\|
        \notag 
        \le
        \int_t^{s^\tau_{u{\shortrightarrow}t}}
        \|\partial_2 b_0(t,\sigma,x)\|\,d\sigma
        \notag \\
        &\le
        \int_t^{s^\tau_{u{\shortrightarrow}t}} M_b\,d\sigma
        \notag
        =
        M_b\,\bigl(s^\tau_{u{\shortrightarrow}t}-t\bigr).
        \label{eq:lagrangian-bound-before-gap}
    \end{align}
    Using $s^\tau_{u{\shortrightarrow}t} = t+e^{-\tau}(u-t)$, we get:
    \begin{equation}
        \|b_\tau(t,u,x)-x\|
        \le
        M_b\,e^{-\tau}(u-t).
        \label{eq:b-rate-step-unified}
    \end{equation}
    Since $F_\tau(t,u,x)=F^*(t,u,b_\tau(t,u,x))$,
    we may apply the Lipschitz estimate \eqref{eq:fstar-Lipschitz-explicit}
    with initial points \(b_\tau(t,u,x)\) and \(x\):
    \begin{align}
        \|F_\tau(t,u,x)-F^*(t,u,x)\|
        &=
        \bigl\|
            F^*(t,u,b_\tau(t,u,x))-F^*(t,u,x)
        \bigr\|
        \notag\\
        &\le
        e^{L(u-t)}\,\|b_\tau(t,u,x)-x\|
        \notag\\
        &\le
        e^{L(u-t)}\,M_b\,e^{-\tau}(u-t),
    \end{align}
    which proves \eqref{eq:lagrangian-rate-map-unified}.
\end{proof}

\section{Stationary points and convergence of Improved MeanFlow}
\label{appsec:imf_convergence}

Improved MeanFlow (iMF) \citep{geng2025improved} is the objective \eqref{eq:imf_func}.
Its stationary points are characterized in Corollary~\ref{cor:imf}, proved in \S\ref{app:corollaries}: the true flow map is the unique stationary point in $\mathcal{M}$.
Since $Q \equiv 0$, the label noise is additive and Corollary~\ref{cor:uniqueQ} applies: the true flow map is the unique point of $\mathcal{M}$ attaining the conditional-variance floor.

For convergence, the semi-gradient given by Theorem~\ref{thm:principle} is
\begin{align}
    g_{\mathrm{iMF}}[f](t,u,x)
    = f(t,u,x) - v(t,x) - (u-t)\Big(\partial_t f(t,u,x) + (\partial_x f)(t,u,x)\, f(t,t,x)\Big).
    \label{eq:g-imf-setup}
\end{align}
On the diagonal, $g_{\mathrm{iMF}}[f](t,t,x) = f(t,t,x) - v(t,x)$.
Hence along the semi-gradient flow $\partial_\tau f_\tau = -g_{\mathrm{iMF}}[f_\tau]$ the diagonal obeys the linear equation $\partial_\tau\big(f_\tau(t,t,x) - v(t,x)\big) = -\big(f_\tau(t,t,x) - v(t,x)\big)$, so
\begin{align}
    f_\tau(t,t,x) - v(t,x) = e^{-\tau}\big(f_0(t,t,x) - v(t,x)\big).
\end{align}
Under perfect flow-matching pretraining \eqref{eq:perfect-fm-pretrain-unified}, $f_\tau(t,t,\cdot) = v(t,\cdot)$ for all $\tau$, the analogue of Lemma~\ref{lem:diag-invariance-esd-lsd-unified}.
Substituting into \eqref{eq:g-imf-setup},
\begin{align}
    g_{\mathrm{iMF}}[f_\tau] = g_{\mathrm{MF}}[f_\tau] = -\,E_v[f_\tau]
\end{align}
by \eqref{eq:mf-equals-minus-Ev}, so $F_\tau$ obeys the Eulerian flow \eqref{eq:mf-f-flow-unified}.
Proposition~\ref{prop:rectifying-templates-unified}(a) and Proposition~\ref{prop:exp-rates-unified} therefore apply verbatim: $F_\tau$ is the closed-form composition \eqref{eq:f-explicit-eulerian-unified} of $F_0$ and $F^*$, and converges to $F^*$.

\section{Numerical experiments on convergence}
\label{appsec:convergence-numerical-experiments}

We numerically simulate functional semi-gradient flow from an initial $F_0$ over optimization time $\tau$ according to equations \ref{eq:eulerian-gradient-flow}, \ref{eq:lagrangian-gradient-flow}.
We compare the numerically simulated $F_\tau$ to its theoretical value from equations \ref{eq:f-explicit-eulerian-unified} and \ref{eq:f-explicit-lagrangian-unified}.
We model $F$ on a discretized tensor grid for $(t,u,x)$, and compute $\partial_t F, \partial_u F, \partial_x F$ using finite differences. These enable computing $\partial_\tau F_\tau$.
We use fourth-order Runge-Kutta to simulate the semi-gradient flow \gls{pde} in $\tau$.
At each step, we enforce the diagonal identity $F_\tau(t,t,x)=x$ exactly by overwriting elements in the tensor grid where $t=u$.
We perform experiments in 1 to 5 spatial dimensions, for 3 to 7 total dimensions, as grid discretization makes compute cost scale exponentially with the number of dimensions.
Our experiments took seconds to minutes to run on one A100.

\paragraph{Velocity field.}
We use an affine velocity field: $v(x) = Ax +b$, which allows the exact flow map $F^*(t,u,x)$ to be computed efficiently in closed form via a matrix exponential.

\paragraph{Initial flow maps.} We construct $F_0(t,u,x) = F^*(t,u,x) + (u-t) \delta(t,u,x)$ so that $F_0(t,t,x)=x$ automatically. We choose $\delta$ to be smooth trigonometric factors in $t,u,$ and $x$.

\section{Differentiating w.r.t.\ $t$ versus $u$}
\label{appsec:t_vs_u}

In this work, $F(t,u,x)$ maps forward in time from $X_t=x$ to $X_u$ along the flow $dX_s = v(s,X_s)$. In Meanflow \citep{geng2025meanflow},
the map is written in reverse time. The transport \gls{pde} can be obtained by differentiating the forward-time flow map with respect to $t$, or the reverse-time flow map with respect to $u$. In both cases, let $t \leq u$:
\begin{align}
    F_{\text{fwd}}(t,u,x) = x + \int_t^u v(s, X_s)
      = x + \int_t^u v(s, F_{\text{fwd}}(t,s,x))ds
\end{align}
On the other hand, if moving in reverse time:
\begin{align}
    F_{\text{rev}}(t,u,x) = x + \int_u^t v(s, X_s) 
        = x - \int_t^u v(s, F_{\text{rev}}(s,u,x)) ds
\end{align}
Differentiating via Leibniz rule yields:
\begin{align}
    \partial_t F_{\text{fwd}}(t,u,x)
    +
    (\partial_x F_{\text{fwd}})(t,u,x)
    v(t,x) &= 0\\
    \partial_u F_{\text{rev}}(t,u,x)
    +
    (\partial_x F_{\text{rev}})(t,u,x)
    v(u,x) &= 0
\end{align}
In both cases, the PDE is backward transport in the time variable carrying the spatial argument $x$; the other time index is a parameter.

\section{ImageNet Experiments}
\label{appsec:imagenet_experiments}

\begin{wraptable}{r}{0.4\textwidth}
  \centering
  \small
  \setlength{\tabcolsep}{5pt}

  \begin{tabularx}{\linewidth}{l|>{\centering\arraybackslash}X}
  \toprule
  \multicolumn{1}{c|}{Method} & \gls{fid}-1 \\
  \midrule
  \gls{esd} & 287.6 $\pm$ 2.0 \\
  \gls{esd} \textsc{orig sg} & 195.1 $\pm$ 49.2 \\
  \gls{esd} \textsc{slim sg} & 128.7 $\pm$ 1.3 \\
  \midrule
  \gls{lsd} & 390.8 $\pm$ 23.5 \\
  \gls{lsd} \textsc{orig sg} & 152.8 $\pm$ 3.0 \\
  \gls{lsd} \textsc{slim sg} & 159.5 $\pm$ 8.6 \\
  \bottomrule
  \end{tabularx}

  \caption{\small{
    \textbf{\gls{fid} on \textsc{ImageNet}}.
    See Table \ref{tab:fid_results_best} caption.
    Mean and std over three seeds.
  }}
  \label{tab:imagenet_fid_results_best}
\end{wraptable}

We train models on ImageNet with SD-VAE latents.
We pretrain the MeanFlow-B4 model, a 90M diffusion transformer, with flow matching for 50k steps, then train with the full loss, and report one-step \gls{fid} at 200k train steps (40 epochs).
This setting compares objectives in shorter, ablation-style runs, with results in Table \ref{tab:imagenet_fid_results_best}.
After FM-only pretraining, one-step \gls{fid} is 315, showing that all objectives except for \gls{lsd} improve \gls{fid} while training.
We find that \textsc{slim} \gls{lsd} is competitive with \gls{lsd} \textsc{orig}, and \textsc{slim} \gls{esd} remains the best among \gls{esd} variants.

\section{Experimental Methods}
\label{appsec:cifar_experiments}
\label{appsec:experiment_details}

\subsection{CIFAR-10 Experiments}

\paragraph{Data.} We use the CIFAR-10 dataset \citep{krizhevsky2009learning}.

\paragraph{Architecture.}
We use the Song U-Net used in \cite{geng2025meanflow}.
We use image resolution 32, in channels 3, out channels 3, model channels 18, num blocks 4, channel mult noise 2, resample filter [1,3,3,1], channel mult [2,2,2], encoder type standard, and decoder type standard.

\paragraph{Compute.} We ran each experiment (a given method) for 200k gradient steps. To accomplish this, for each experiment we used 8x A100 GPUs. Each method was able to reach 200k steps with 96 hours of training.

\paragraph{Training Details.}

\begin{itemize}

\item \textit{Dataloading.} We use pin-mem = True and 4 workers.

\item \textit{Time sampling.} We use min times 1e-4 and max times 1 - 1e-4 for both training and sampling. We use probability t equals u set to 0.25.

\item \textit{EMA.} We use EMA decay 0.9999. We start the EMA model at the 40k training steps model, and update once per gradient step thereafter.

\item \textit{Batch Size.} We use batch size 64 per GPU and use 8 GPUs, meaning the effective batch size is 64*8=512.

\item \textit{Optimization.} We use the AdamW optimizer with betas (0.9, 0.999). We use base learning rate 2e-4. We use linear warmup from 1e-6 up to the base learning rate for the first 20,000 steps. After warmup, the learning rate is constant at the base learning rate.

\item \textit{Seed.} We use seeds 0, 1, and 2.

\item \textit{Velocity Warmup.} We do not use any velocity warmup steps, similar to \cite{geng2025meanflow} and unlike \cite{boffi2025build}.

\item \textit{Dropout.} We use dropout 0.20

\item \textit{Adaptive Loss Weight.} We use the adaptive loss reweighting from \cite{geng2025meanflow}
with norm-p set to 0.75 and norm-eps set to 1e-3.

\end{itemize}

\paragraph{Evaluation Details.}

We evaluate FID from the non-EMA and EMA models every 25,000 steps. We evaluate this FID for several different sampling step counts per checkpoint, with number of function evaluations in \{1,10,50,100\}. We compute FID with the standard 50,000 samples using batch size 128.
To compute FID, we use \texttt{torchmetrics.image.fid.FrechetInceptionDistance}.

\subsection{ImageNet Experiments}
We use the B4 model from the meanflow-jax github repository which is a 90M diffusion transformer.
We adopt their default settings, which uses patch size 4, learning rate of 0.0001, batch size 256, adam b2 of 0.95, and adaptive L2 weighting norm power of 1.0.

We run flow-matching warmup for 10 epochs, and train for a total of 40 epochs.
We use gradient surgery, prioritizing the diagonal gradient. 
Each run trained on 4x B200s, and training took about 8 hours.

\section{Table of all Flow Map Objectives}

In Table \ref{tab:flowmap_objective_comparison}, we provide a table of all flow map objectives discussed in this work, along with derivative cost (first order or second order), whether stationary points were previously characterized, unknown, or characterized in this work, and whether the objective was introduced in prior work or this work.

\begin{table}[H]
\centering
\scriptsize
\setlength{\tabcolsep}{4pt}
\renewcommand{\arraystretch}{1.3}

\begin{tabularx}{\textwidth}{
    @{}
    >{\RaggedRight\arraybackslash}p{1.0cm}
    >{\RaggedRight\arraybackslash}p{9cm}
    >{\Centering\arraybackslash}p{1.0cm}
    >{\Centering\arraybackslash}p{1.0cm}
    >{\RaggedRight\arraybackslash}p{1.0cm}
}
    \toprule
    \textbf{Method}
    &
    \textbf{Objective}
    &
    \textbf{Derivative cost order}
    &
    \textbf{Stationary point}
    &
    \textbf{Novelty}
    \\
    \midrule
    
    MeanFlow
    &
    \(\displaystyle
    \E\left\|
    f(t,u,x)-\dot x_t
    -(u-t)\sgb{
    \partial_t f(t,u,x)
    +(\partial_xf)(t,u,x)\dot x_t
    }
    \right\|^2
    \)
    &
    First
    &
    This work
    &
    \citep{geng2025meanflow}
    \\
    \midrule
    
    Improved MeanFlow
    &
    \(\displaystyle
    \E\left\|
    f(t,u,x)
    -\dot x_t
    -(u-t)\sgb{
    \partial_t f(t,u,x)
    +(\partial_x f)(t,u,x)f(t,t,x)
    }
    \right\|^2
    \)
    &
    First
    &
    This work
    &
    \citep{geng2025improved}
    \\
    \midrule
    
    ESD
    &
    \(\displaystyle
    \E\left\|
    \partial_tF(t,u,x)
    +(\partial_xF)(t,u,x)f(t,t,x)
    \right\|^2
    +\mathcal L_{\mathrm{FM}}
    \)
    &
    Second
    &
    Known
    &
    \citep{boffi2025build}
    \\
    \midrule
    
    ESD original SG
    &
    \(\displaystyle
    \E\left\|
    \partial_tF(t,u,x)
    +\sgb{
    (\partial_xF)(t,u,x)f(t,t,x)
    }
    \right\|^2
    +\mathcal L_{\mathrm{FM}}
    \)
    &
    Second
    &
    Unknown
    &
    \citep{boffi2025build}
    \\
    \midrule
    
    Slim ESD
    &
    \(\displaystyle
    \begin{aligned}
        &\E\left\|
        f(t,u,x)
        -\sgb{f(t,t,x)
        + (u-t)(
        \partial_t f(t,u,x)
        +(\partial_x f)(t,u,x)f(t,t,x))
        }
        \right\|^2 \\
        &+\mathcal L_{\mathrm{FM}}
    \end{aligned}
    \)
    &
    First
    &
    This work
    &
    This work
    \\
    \midrule
    
    LSD
    &
    \(\displaystyle
    \E\left\|
    \partial_u F(t,u,x)
    -
    f\!\left(u,u,F(t,u,x)\right)
    \right\|^2
    +\mathcal L_{\mathrm{FM}}
    \)
    &
    Second
    &
    Known
    &
    \citep{boffi2025build}
    \\
    \midrule
    
    LSD original SG
    &
    \(\displaystyle
    \E\left\|
    f(t,u,x)
    +(u-t)\partial_u f(t,u,x)
    -\sgb{
    f\!\left(u,u,F(t,u,x)\right)
    }
    \right\|^2
    +\mathcal L_{\mathrm{FM}}
    \)
    &
    Second
    &
    Unknown
    &
    \citep{boffi2025build}
    \\
    \midrule
    
    Slim LSD
    &
    \(\displaystyle
    \E\left\|
    f(t,u,x)
    +(u-t)\sgb{\partial_u f(t,u,x)}
    -\sgb{
    f\!\left(u,u,F(t,u,x)\right)
    }
    \right\|^2
    +\mathcal L_{\mathrm{FM}}
    \)
    &
    First
    &
    This work
    &
    This work
    \\
    \bottomrule
\end{tabularx}

\caption{
Comparison of flow-map objectives, using
\(F(t,u,x)=x+(u-t)f(t,u,x)\) and the interpolation convention
that \(t=1\) is the data endpoint.
}
\label{tab:flowmap_objective_comparison}
\end{table}

\section{Broader Societal Impact}
\label{appsec:broader_impact}

While the contribution in this work is mostly theoretical, the results do provide a cheaper optimization strategy for training flow maps (also known as consistency models). In particular, the models are trainable to about the same quality at half the memory use, at least when benchmarked on the small academic CIFAR-10 dataset and ImageNet datasets. One potential extrapolation here is that others could use these methods to also train their models with fewer resources, which can be good or bad depending on the nature of the actor. For now, we believe the work is so small-scale, that there is no immediately recognizable or unique risk that results from this work.

\end{document}